\documentclass{article}

\usepackage{microtype}
\usepackage{graphicx}
\usepackage{subfigure}
\usepackage{booktabs}
\usepackage[utf8]{inputenc}
\usepackage[T1]{fontenc}
\usepackage{hyperref}
\usepackage{url}
\usepackage{amsfonts}
\usepackage{nicefrac}
\usepackage{microtype}

\usepackage{amsmath,amsfonts,amssymb,bm,mathtools}
\usepackage{bbm}
\usepackage{tikz}
\usepackage{bbold}
\usepackage{graphicx}
\usepackage{pgfplots}
\usepgfplotslibrary{fillbetween}
\usepackage[ruled,algo2e,linesnumbered]{algorithm2e}
\usepackage[sort]{natbib}
\usepackage{comment}
\usepackage[mathscr]{euscript}

\DeclareFontFamily{U}{matha}{\hyphenchar\font45}
\DeclareFontShape{U}{matha}{m}{n}{
  <-6> matha5 <6-7> matha6 <7-8> matha7
  <8-9> matha8 <9-10> matha9
  <10-12> matha10 <12-> matha12
  }{}

\DeclareSymbolFont{matha}{U}{matha}{m}{n}
\DeclareMathSymbol{\Lt}{3}{matha}{"CE}

\newcommand{\B}[1]{\mathbb{#1}}
\newcommand{\C}[1]{\mathcal{#1}}
\newcommand{\EE}{\mathbb{E}}
\newcommand{\bb}[1]{\mathbf{#1}}
\newcommand{\indep}{\protect\mathpalette{\protect\independenT}{\perp}}
\newcommand{\dx}[1]{\mathrm{d}#1}
\newcommand{\sgn}{\mathop{\mathrm{sign}}}
\def\independenT#1#2{\mathrel{\rlap{$#1#2$}\mkern2mu{#1#2}}}

\DeclareMathOperator{\unif}{\textsc{Unif}}
\DeclareMathOperator{\bern}{\textsc{Bern}}
\DeclareMathOperator{\bbeta}{\textsc{Beta}}
\DeclareMathOperator{\norm}{\textsc{N}}

\DeclareMathOperator{\logit}{\text{logit}}
\DeclareMathOperator{\rest}{\upharpoonright}
\DeclareMathOperator{\Span}{\textsc{Span}}
\DeclareMathOperator{\supp}{\textsc{Supp}}

\DeclareMathOperator{\pPr}{Pr}
\DeclareMathOperator{\up}{{Up}}
\DeclareMathOperator{\low}{{Lo}}

\DeclareMathOperator{\img}{\textsc{Img}}
\DeclareMathOperator{\abs}{\textsc{Abs}}

\DeclareMathOperator{\id}{id}

\newcommand*\diff{\mathop{}\!\mathrm{d}}

\renewcommand{\Pr}{\pPr}

\newlength{\myl}
\usepackage{amsthm}

\theoremstyle{definition}
\newtheorem{defn}{Definition}

\theoremstyle{plain}
\newtheorem{prop}{Proposition}
\newtheorem{thm}{Theorem}
\newtheorem{lem}{Lemma}
\newtheorem{cor}{Corollary}

\theoremstyle{remark}
\newtheorem{rmk}{Remark}
\newtheorem{case}{Case}

\usepackage{hyperref}

\usepackage[accepted]{icml2022}

\pgfplotsset{compat=1.17}

\begin{document}

\twocolumn[
\icmltitle{Causal Conceptions of Fairness and their Consequences}
\icmlsetsymbol{equal}{*}
\begin{icmlauthorlist}
  \icmlauthor{Hamed Nilforoshan}{equal,stanford}
  \icmlauthor{Johann Gaebler}{equal,stanford}
  \icmlauthor{Ravi Shroff}{nyu}
  \icmlauthor{Sharad Goel}{hks}
\end{icmlauthorlist}

\icmlaffiliation{stanford}{Stanford University, Stanford, CA}
\icmlaffiliation{nyu}{New York University, New York, NY}
\icmlaffiliation{hks}{Harvard University, Cambridge, MA}

\icmlcorrespondingauthor{Hamed Nilforoshan}{hamedn@cs.stanford.edu}
\icmlcorrespondingauthor{Johann Gaebler}{jgaeb@stanford.edu}
\icmlcorrespondingauthor{Ravi Shroff}{ravi.shroff@nyu.edu}
\icmlcorrespondingauthor{Sharad Goel}{sgoel@hks.harvard.edu}

\icmlkeywords{Machine Learning, ICML}

\vskip 0.3in
]

\printAffiliationsAndNotice{\icmlEqualContribution}

\begin{abstract}
  Recent work highlights the role of causality in designing equitable
  decision-making algorithms. It is not immediately clear, however, how existing
  causal conceptions of fairness relate to one another, or what the consequences
  are of using these definitions as design principles. Here, we first assemble
  and categorize popular causal definitions of algorithmic fairness into two
  broad families: (1) those that constrain the effects of decisions on
  counterfactual disparities; and (2) those that constrain the effects of
  legally protected characteristics, like race and gender, on decisions. We then
  show, analytically and empirically, that both families of definitions
  \emph{almost always}---in a measure theoretic sense---result in strongly
  Pareto dominated decision policies, meaning there is an alternative,
  unconstrained policy favored by every stakeholder with preferences drawn from
  a large, natural class. For example, in the case of college admissions
  decisions, policies constrained to satisfy causal fairness definitions would
  be disfavored by every stakeholder with neutral or positive preferences for
  both academic preparedness and diversity. Indeed, under a prominent definition
  of causal fairness, we prove the resulting policies require admitting all
  students with the same probability, regardless of academic qualifications or
  group membership. Our results highlight formal limitations and potential
  adverse consequences of common mathematical notions of causal fairness.
\end{abstract}

\section{Introduction}

Imagine designing an algorithm to guide decisions for college admissions. To
help ensure algorithms such as this are broadly equitable, a plethora of formal
fairness criteria have been proposed in the machine learning
community~\citep{
  darlington1971another, cleary1968test, zafar2017parity, dwork2012fairness,
  chouldechova2017fair, hardt2016equality, kleinberg2016inherent,
  woodworth2017learning, zafar2017fairness, corbett2017algorithmic,
  chouldechova2020snapshot, berk2021fairness%
}.
For example, under the principle that fair algorithms should have comparable
performance across demographic groups~\citep{hardt2016equality}, one might check
that among applicants who were ultimately academically ``successful'' (e.g., who
eventually earned a college degree, either at the institution in question or
elsewhere), the algorithm would recommend admission for an equal proportion of
candidates across race groups. Alternatively, following the principle that
decisions should be agnostic to legally protected attributes like race and
gender~\citep[cf.][]{corbett2018measure, dwork2012fairness}, one might mandate
that these features not be provided to the algorithm.

Recent scholarship has argued for extending equitable algorithm design
by adopting a causal perspective, leading to myriad additional formal criteria
for fairness~\citep{
  coston2020counterfactual, imai2020principal, imai2020experimental,
  wang2019equal, kusner2017counterfactual, nabi2018fair, wu2019pc,
  mhasawade2021causal, kilbertus2017avoiding, zhang2018fairness,
  zhang2016causal, chiappa2019path, loftus2018causal, galhotra2022causal,
  carey2022causal%
}.
Less attention, however, has been given to understanding the potential
downstream consequences of using these causal definitions of fairness as
algorithmic design principles, leaving an important gap to fill if these
criteria are to responsibly inform policy choices.

Here we synthesize and critically examine the statistical properties and
concomitant consequences of popular causal approaches to fairness. We begin, in
Section~\ref{sec:defs}, by proposing a two-part taxonomy for causal conceptions
of fairness that mirrors the illustrative, non-causal fairness principles
described above. Our first category of definitions encompasses those that
consider the effect of decisions on counterfactual disparities. For example,
recognizing the causal effect of college admission on later success, one might
demand that among applicants who would be academically successful \emph{if
admitted} to a particular college, the algorithm would recommend admission for
an equal proportion of candidates across race groups. The second category of
definitions encompasses those that seek to limit both the direct and indirect
effects of one's group membership on decisions. For example, because one's race
might impact earlier educational opportunities, and hence test scores, one might
require that admissions decisions are robust to the effect of race along such
causal paths.

We show, in Section~\ref{sec:construction}, that when the distribution of causal
effects is known (or can be estimated), one can efficiently compute
utility-maximizing decision policies constrained to satisfy each of the causal
fairness criteria we consider. However, for natural families of utility
functions---for example, those that prefer both higher degree attainment and
more student-body diversity---we prove in Section~\ref{sec:pareto} that causal
fairness constraints \emph{almost always} lead to strongly Pareto dominated
decision policies. To establish this result, we use the theory of
prevalence~\citep{
  christensen1972sets, hunt1992prevalence, anderson2001genericity,
  ott2005prevalence%
},
which extends the notion of full-measure sets to infinite-dimensional vector
spaces. In particular, in our running college admissions example, adhering to
any of the common conceptions of causal fairness would simultaneously result in
a lower number of degrees attained and lower student-body diversity, relative to
what one could achieve by explicitly tailoring admissions policies to achieve
desired outcomes. In fact, under one prominent definition of causal fairness, we
prove that the induced policies require simply admitting all applicants with
equal probability, irrespective of one's academic qualifications or group
membership. These results, we hope, elucidate the structure---and
limitations---of current causal approaches to equitable decision making.

\section{Causal Approaches to Fair Decision Making}
\label{sec:defs}

We describe two broad classes of causal notions of fairness: (1) those that
consider outcomes when \emph{decisions} are counterfactually altered; and (2)
those that consider outcomes when \emph{protected attributes} are
counterfactually altered. We illustrate these definitions in the context of a
running example of college admissions decisions.

\subsection{Problem Setup}
\label{sec:example}

Consider a population of individuals with observed covariates \(X\), drawn i.i.d
from a set \(\mathcal{X} \subseteq \mathbb{R}^n\) with distribution
\(\mathcal{D}_{X}\). Further suppose that \(A \in \mathcal{A}\) describes one or
more discrete protected attributes, such as race or gender, which can be derived
from \(X\) (i.e., \(A = \alpha(X)\) for some measurable function \(\alpha\)).
Each individual is subject to a binary decision \(D \in \{0, 1\}\), determined
by a (randomized) rule \(d(x) \in [0, 1]\), where \(d(x) = \Pr(D = 1 \mid X =
x)\) is the probability of receiving a positive decision.\footnote{%
  That is, \(D = \B 1_{U_D \leq d(X)}\), where \(U_D\) is an independent uniform
  random variable.
}
Given a budget \(b\) with \(0 < b < 1\), we require the decision rule to satisfy
\(\EE[D] \leq b\), limiting the expected proportion of positive decisions.

In our running example, we imagine a population of applicants to a particular
college, where \(d\) denotes an admissions rule and \(D\) indicates a binary
admissions decision. To simplify our exposition, we assume all admitted students
attend the school. In our setting, the covariates \(X\) consist of an
applicant's test score and race \(A \in \{a_0, a_1\}\), where, for notational
convenience, we consider two race groups. The budget \(b\) bounds the expected
proportion of admitted applicants.

Assuming there is no interference between units~\cite{imbens2015causal}, we
write \(Y(1)\) and \(Y(0)\) to denote potential outcomes of interest under each
of the two possible binary decisions, where \(Y = Y(D)\) is the realized
outcome. We assume that \(Y(1)\) and \(Y(0)\) are drawn from a (possibly
infinite) set \(\C Y \subseteq \B R\), where \(|\C Y| > 1\). In our admissions
example, \(Y\) is a binary variable that indicates college graduation (i.e.,
degree attainment), with \(Y(1)\) and \(Y(0)\) describing, respectively, whether
an applicant would attain a college degree if admitted to or if rejected from
the school we consider. Note that \(Y(0)\) is not necessarily zero, as a
rejected applicant may attend---and graduate from---a different university.

Given this setup, our goal is to construct decision policies \(d\) that are
broadly equitable, formalized in part by the causal notions of fairness
described below. We focus on decisions that are made algorithmically, informed
by historical data on applicants and subsequent outcomes.

\subsection{Limiting the Effect of Decisions on Disparities}

A popular class of non-causal fairness definitions requires that error rates
(e.g., false positive and false negative rates) are equal across protected
groups~\citep{hardt2016equality, corbett2018measure}. Causal analogues of these
definitions have recently been proposed~\citep{%
  coston2020counterfactual, imai2020principal, imai2020experimental,
  mishler2021fairness%
},
which require various conditional independence conditions to hold between the
potential outcomes, protected attributes, and decisions.\footnote{%
  In the literature on causal fairness, there is at times ambiguity between
  ``predictions'' \(\hat{Y} \in \{0,1\}\) of \(Y\) and ``decisions'' \(D \in
  \{0,1\}\). Following past work~\citep[e.g.,][]{%
    corbett2017algorithmic, kusner2017counterfactual, wang2019equal%
  },
  here we focus exclusively on decisions, with predictions implicitly impacting
  decisions but not explicitly appearing in our definitions.
}

Below we list three representative examples of this class of fairness
definitions: counterfactual predictive parity~\citep{coston2020counterfactual},
counterfactual equalized odds~\citep{mishler2021fairness,
coston2020counterfactual}, and conditional principal
fairness~\citep{imai2020principal}.\footnote{%
  Our subsequent analytical results extend in a straightforward manner to
  structurally similar variants of these definitions (e.g., requiring \(Y(0)
  \indep A \mid D=1\) or \(D \indep A \mid Y(0)\), variants of counterfactual
  predictive parity and counterfactual equalized odds, respectively).
}

\begin{defn}
\label{defn:predictive-parity}
  \emph{Counterfactual predictive parity} holds when
  \begin{equation}
  \label{eq:counterfactual_predictive_parity}
    Y(1) \indep A \mid D = 0.
  \end{equation}
\end{defn}

In our college admissions example, counterfactual predictive parity means that
among rejected applicants, the proportion who would have attained a college
degree, had they been accepted, is equal across race groups.

\begin{defn}
\label{defn:cf-eo}
  \emph{Counterfactual equalized odds} holds when
  \begin{equation}
  \label{eq:counterfactual_equalized_odds}
    D \indep A \mid Y(1).
  \end{equation}
\end{defn}

In our running example, counterfactual equalized odds is satisfied when two
conditions hold: (1) among applicants who would graduate if admitted (i.e.,
\(Y(1) = 1\)), students are admitted at the same rate across race groups; and
(2) among applicants who would not graduate if admitted (i.e., \(Y(1) = 0\)),
students are again admitted at the same rate across race groups.

\begin{defn}
\label{defn:principal-fairness}
  \emph{Conditional principal fairness} holds when
  \begin{equation}
  \label{eq:conditional_principal_fairness}
    D \indep A \mid Y(0), Y(1), W,
  \end{equation}
  where, for a measurable function \(\omega\) on \(\C X\), \(W = \omega(X)\)
  describes a reduced set of the  covariates \(X\). When \(W\) is constant (or,
  equivalently, when we do not condition on \(W\)), this condition is called
  \emph{principal fairness}.
\end{defn}

In our example, conditional principal fairness means that ``similar''
applicants---where similarity is defined by the potential outcomes and
covariates \(W\)---are admitted at the same rate across race groups.

\subsection{Limiting the Effect of Attributes on Decisions}

An alternative causal framework for understanding fairness considers the effects
of protected attributes on decisions \cite{%
  wang2019equal, kusner2017counterfactual, nabi2018fair, wu2019pc,
  mhasawade2021causal, kilbertus2017avoiding, zhang2018fairness,
  zhang2016causal%
}.
This approach, which can be understood as codifying the legal notion of
disparate treatment~\cite{goel2017combatting, zafar2017fairness}, considers a
decision rule to be fair if, at a high level, decisions for individuals are the
same in ``(a) the actual world and (b) a counterfactual world where the
individual belonged to a different demographic
group''~\cite{kusner2017counterfactual}.\footnote{%
  Conceptualizing a general causal effect of an immutable characteristic such as
  race or gender is rife with challenges, the greatest of which is expressed by
  the mantra, ``no causation without manipulation''
  \cite{holland1986statistics}. In particular, analyzing race as a causal
  treatment requires one to specify what exactly is meant by ``changing an
  individual's race" from, for example, white to
  Black~\citep{gaebler2020causal, hu2020s}. Such difficulties can sometimes be
  addressed by considering a change in the \emph{perception} of race by a
  decision maker~\cite{greiner2011causal}---for instance, by changing the name
  listed on an employment application~\citep{bertrand2004emily}, or by masking
  an individual's
  appearance~\citep{%
    goldin2000orchestrating, grogger2006testing, pierson2020large,
    chohlas2021blind%
  }.
}

In contrast to ``fairness through unawareness''---in which race and other
protected attributes are barred from being an explicit input to a decision
rule~\citep[cf.][]{dwork2012fairness, corbett2018measure}---the causal versions
of this idea consider both the direct and indirect effects of protected
attributes on decisions. For example, even if decisions only directly depend on
test scores, race may indirectly impact decisions through its effects on
educational opportunities, which in turn influence test scores. This idea can be
formalized by requiring that decisions remain the same in expectation even if
one's protected characteristics are counterfactually altered, a condition known
as counterfactual fairness ~\cite{kusner2017counterfactual}.

\begin{defn}
\label{defn:counterfactual-fairness}
  \emph{Counterfactual fairness} holds when
  \begin{equation}
  \label{eq:counterfactual_fairness}
    \EE[D(a') \mid X] = \EE[D \mid X].
  \end{equation}
  where \(D(a')\) denotes the decision
  when one's protected attributes are counterfactually altered to be any \(a'
  \in \mathcal{A}\).
\end{defn}

In our running example, this means that for each group of observationally
identical applicants (i.e., those with the same values of \(X\), meaning
identical race and test score), the proportion of students who are actually
admitted is the same as the proportion who would be admitted if their race were
counterfactually altered.

Counterfactual fairness aims to limit all direct and indirect effects of
protected traits on decisions. In a generalization of this criterion---termed
path-specific fairness~\citep{%
  chiappa2019path, nabi2018fair, zhang2016causal, wu2019pc%
}---one
allows protected traits to influence decisions along certain causal paths but
not others. For example, one may wish to allow the direct consideration of race
by an admissions committee to implement an affirmative action policy, while also
guarding against any indirect influence of race on admissions decisions that may
stem from cultural biases in standardized tests~\citep{williams1983some}.

\begin{figure}
  \begin{center}
    \begin{tikzpicture}[xscale = 1.75, yscale = 2, align = center]
      \node at (0,0) (race) {\(A\)\\{\scriptsize Race}};
      \node at (1,0) (educ) {\(E\)\\{\scriptsize Education}};
      \node at (2,0) (test) {\(T\)\\{\scriptsize Test Score}};
      \node at (3/2, -2/3) (medi) {\(M\)\\{\scriptsize Preparation}};
      \node at (3,0) (deci) {\(D\)\\{\scriptsize Decision}};
      \node at (4, 0) (pass) {\(Y\)\\{\scriptsize Graduation}};

      \draw[->, color = red, line width=0.25mm] (race) to (educ);
      \draw[->, bend left = 30, line width=0.25mm] (race) to (deci);
      \draw[->, line width=0.25mm] (educ) to (medi);
      \draw[->, color = red, line width=0.25mm] (educ) to (test);
      \draw[->, line width=0.25mm] (medi) to (test);
      \draw[->, bend right = 20, line width=0.25mm] (medi) to (pass);
      \draw[->, color = red, line width=0.25mm] (test) to (deci);
      \draw[->, line width=0.25mm] (deci) to (pass);

    \end{tikzpicture}
  \end{center}
  \caption{%
    A causal DAG illustrating a hypothetical process for college admissions.
    Under path-specific fairness, one may require, for example, that race does
    not affect decisions along the path highlighted in red.
  }
\label{fig:dag}
\end{figure}
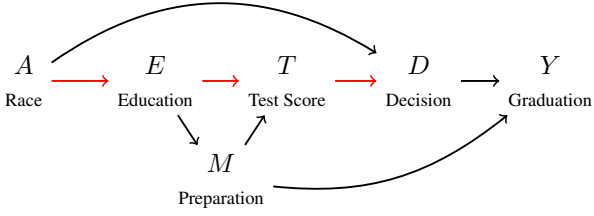

The formal definition of path-specific fairness requires specifying a causal DAG
describing relationships between attributes (both observed covariates and latent
variables), decisions, and outcomes. In our running example of college
admissions, we imagine that each individual's observed covariates are the result
of the process illustrated by the causal DAG in Figure~\ref{fig:dag}. In this
graph, an applicant's race \(A\) influences the educational opportunities \(E\)
available to them prior to college; and educational opportunities in turn
influence an applicant's level of college preparation, \(M\), as well as their
score on a standardized admissions test, \(T\), such as the SAT. We assume the
admissions committee only observes an applicant's race and test score so that
\(X = (A, T)\), and makes their decision \(D\) based on these attributes.
Finally, whether or not an admitted student subsequently graduates (from any
college), \(Y\), is a function of both their preparation and whether they were
admitted.\footnote{%
  In practice, the racial composition of an admitted class may itself influence
  degree attainment, if, for example, diversity provides a net benefit to
  students~\citep{page2007making}. Here, for simplicity, we avoid consideration
  of such peer effects.
}

To define path-specific fairness, we start by defining, for the decision \(D\),
path-specific counterfactuals, a general concept in causal
DAGs~\citep[cf.][]{pearl2001direct}. Suppose \(\mathcal{G} = (\mathcal{V},
\mathcal{U}, \mathcal{F})\) is a causal model with nodes \(\mathcal{V}\),
exogenous variables \(\mathcal{U}\), and structural equations \(\mathcal{F}\)
that define the value at each node \(V_j\) as a function of its parents
\(\wp(V_j)\) and its associated exogenous variable \(U_j\). (See, for example,
\citet{pearl2009causal} for further details on causal DAGs.) Let \(V_1, \dots,
V_m\) be a topological ordering of the nodes, meaning that \(\wp(V_j) \subseteq
\{V_1, \dots, V_{j-1}\}\) (i.e., the parents of each node appear in the ordering
before the node itself). Let \(\Pi\) denote a collection of paths from node
\(A\) to \(D\). Now, for two possible values \(a\) and \(a'\) for the variable
\(A\), the path-specific counterfactuals \(D_{\Pi,a,a'}\) for the decision \(D\)
are generated by traversing the list of nodes in topological order, propagating
counterfactual values obtained by setting \(A = a'\) along paths in \(\Pi\), and
otherwise propagating values obtained by setting \(A = a\). (In Algorithm 1 in
the Appendix, we formally define path-specific counterfactuals for an arbitrary
node---or collection of nodes---in the DAG.)

To see this idea in action, we work out an illustrative example, computing
path-specific counterfactuals for the decision \(D\) along the single path \(\Pi
= \{A \rightarrow E \rightarrow T \rightarrow D\}\) linking race to the
admissions committee's decision through test score, highlighted in red in
Figure~\ref{fig:dag}. In the system of equations below, the first column
corresponds to draws \(V^*\) for each node \(V\) in the DAG, where we set \(A\)
to \(a\), and then propagate that value as usual. The second column corresponds
to draws \(\overline{V}^*\) of path-specific counterfactuals, where we set \(A\)
to \(a'\), and then propagate the counterfactuals only along the path \(A
\rightarrow E \rightarrow T \rightarrow D\). In particular, the value for the
test score \(\overline{T}^*\) is computed using the value of \(\overline{E}^*\)
(since the edge \(E \rightarrow T\) is on the specified path) and the value of
\(M^*\) (since the edge \(M \rightarrow T\) is not on the  path). As a result of
this process, we obtain a draw \(\overline{D}^*\) from the distribution of
\(D_{\Pi,a,a'}\).
\begin{align*}
  A^{*} &= a,                 & \overline{A}^* &= a', \\
  E^{*} &= f_E(A^{*}),        & \overline{E}^* &= f_E(\overline{A}^*), \\
  M^{*} &= f_M(E^{*}),        & \overline{M}^* &= f_M(E^{*}), \\
  T^{*} &= f_T(E^{*}, M^{*}), & \overline{T}^* &= f_T(\overline{E}^*, M^{*}), \\
  D^{*} &= f_D(A^{*}, T^{*}), & \overline{D}^* &= f_D(A^{*}, \overline{T}^*).
\end{align*}

Path-specific fairness formalizes the intuition that the influence of a
sensitive attribute on a downstream decision may, in some circumstances, be
considered legitimate (i.e., it may be acceptable for the attribute to affect
decisions along certain paths in the DAG). For instance, an admissions committee
may believe that the effect of race \(A\) on admissions decisions \(D\) which
passes through college preparation \(M\) is legitimate, whereas the effect of
race along the path \(A \rightarrow E \rightarrow T \rightarrow D\), which may
reflect access to test prep or cultural biases of the tests, rather than actual
academic preparedness, is illegitimate. In that case, the admissions committee
may seek to ensure that the proportion of applicants they admit from a certain
race group remains unchanged if one were to counterfactually alter the race of
those individuals along the path \(\Pi = \{A \rightarrow E \rightarrow T
\rightarrow D\}\).

\begin{defn}
\label{defn:ps}
  Let \(\Pi\) be a collection of paths, and, for a measurable function \(w\) on
  \(\C X\), let \(W = \omega(X)\) describe a reduced set of the covariates
  \(X\). \emph{Path-specific fairness}, also called \emph{\(\Pi\)-fairness},
  holds when, for any \(a' \in \mathcal{A}\),
  \begin{equation}
  \label{eq:path_specific_fairness}
    \EE[D_{\Pi, A, a'} \mid W] = \EE[D \mid W].
  \end{equation}
\end{defn}

In the definition above, rather than a particular counterfactual level \(a\),
the baseline level of the path-specific effect is \(A\), i.e., an individual's
actual (non-counterfactually altered) group membership (e.g., their actual
race). We have implicitly assumed that the decision variable \(D\) is a
descendant of the covariates \(X\). In particular, without loss of generality,
we assume \(D\) is defined by the structural equation \(f_D(x, u_D) =
\mathbb{1}_{u_D \leq d(x)}\), where the exogenous variable \(u_D \sim
\unif(0,1)\), so that \(\Pr(D = 1 \mid X = x) = d(x)\). If \(\Pi\) is the set of
all paths from \(A\) to \(D\), then \(D_{\Pi,A,a'} = D(a')\), in which case, for
\(W = X\), path-specific fairness is the same as counterfactual fairness.

\section{Constructing Causally Fair Policies}
\label{sec:construction}

The definitions of causal fairness above constrain the set of decision policies
one might adopt, but, in general, they do not yield a unique policy. For
instance, a policy in which applicants are admitted randomly and independently
with probability \(b\)---where \(b\) is the specified budget---satisfies
counterfactual equalized odds (Def.~\ref{defn:cf-eo}), conditional principal
fairness (Def.~\ref{defn:principal-fairness}), counterfactual fairness
(Def.~\ref{defn:counterfactual-fairness}), and path-specific fairness
(Def.~\ref{defn:ps}).\footnote{%
  A policy satisfying counterfactual predictive parity
  (Def.~\ref{defn:predictive-parity}) is not guaranteed to exist. For example,
  if \(b = 0\)---in which case \(D = 0\) a.s.---and \(\EE[Y(1) \mid A = a_1]
  \neq \EE[Y(1) \mid A = a_2]\), then
  Eq.~\eqref{eq:counterfactual_predictive_parity} cannot hold. Similar
  counterexamples can be constructed for \(b \ll 1\).
}

However, such a randomized policy may be sub-optimal in the eyes of
decision-makers aiming to maximize outcomes such as class diversity or degree
attainment. Past work has described multiple approaches to selecting a single
policy from among those satisfying any given fairness definition, including
maximizing concordance of the decision with the outcome
variable~\citep{nabi2018fair, chiappa2019path} or with an existing
policy~\citep{wang2019equal} (e.g., in terms of binary accuracy or
KL-divergence).

Here, as we are primarily interested in the downstream consequences of various
causal fairness definitions, we consider causally fair policies that maximize
utility~\citep{%
  liu2018delayed, kasy2021fairness, corbett2017algorithmic,cai2020fair,L2BF%
}.

Suppose \(u(x)\) denotes the utility of assigning a positive decision to
individuals with observed covariate values \(x\), relative to assigning them
negative decisions. In our running example, we set
\begin{equation}
\label{eq:util}
  u(x) = \EE[Y(1) \mid X = x] + \lambda \cdot \mathbb{1}_{\alpha(x) = a_1},
\end{equation}
where \(\EE[Y(1) \mid X = x]\) denotes the likelihood the applicant would
graduate if admitted, \(\mathbb{1}_{\alpha(x) = a_1}\) indicates whether the
applicant identifies as belonging to race group \(a_1\) (e.g., \(a_1\) may
denote a group historically underrepresented in higher education), and \(\lambda
\geq 0\) is an arbitrary constant that balances preferences for both student
graduation and racial diversity.

We seek decision policies that maximize expected utility, subject to satisfying
a given definition of causal fairness, as well as the budget constraint.
Specifically, letting \(\mathcal{C}\) denote the family of all decision policies
that satisfy one of the causal fairness definitions listed above, a
utility-maximizing policy \(d^*\) is given by
\begin{align}
  \label{eq:opt}
  \begin{split}
    d^* \in \arg\max_{d \in \mathcal{C}} & \quad \EE[d(X) \cdot u(X)] \\
    \text{s.t.}   &\quad \EE[d(X)] \leq b.
  \end{split}
\end{align}

Constructing optimal policies poses both statistical and computational
challenges. One must, in general, estimate the joint distribution of covariates
and potential outcomes---and, even more dauntingly, causal effects along
designated paths for path-specific definitions of fairness. In some settings, it
may be possible to obtain these estimates from observational analyses of
historical data or randomized controlled trials, though both approaches
typically involve substantial hurdles in practice.

We prove that if one has this statistical information, it is possible to
efficiently compute causally fair utility-maximizing policies by solving either
a single linear program or a series of linear programs (Appendix,
Theorem~\ref{thm:lp}). In the case of counterfactual equalized odds, conditional
principal fairness, counterfactual fairness, and path-specific fairness, we show
that the definitions can be translated to linear constraints. For counterfactual
predictive parity, the defining independence condition yields a quadratic
constraint, which we show can be expressed as a linear constraint by further
conditioning on one of the decision variables, and the optimization problem in
turn can be solved through a series of linear programs.

\section{The Structure of Causally Fair Policies}
\label{sec:pareto}

Above, for each definition of causal fairness, we sketched how to construct
utility-maximizing policies that satisfy the corresponding constraints. Now we
explore the structural properties of causally fair policies. We show---both
empirically and analytically, under relatively mild distributional
assumptions---that policies constrained to be causally fair are disfavored by
every individual in a natural class of decision makers with varying preferences
for diversity. To formalize these results, we start by introducing some notation
and then defining the concept of (strong) Pareto dominance.

\subsection{Pareto Dominance and Consistent Utilities}

For a real-valued utility function \(u\) and decision policy \(d\), we write
\(u(d) = \EE[d(X) \cdot u(X)]\) to denote the utility of \(d\) under \(u\).

\begin{defn}
  For a budget \(b\), we say a decision policy \(d\) is \emph{feasible} if
  \(\EE[d(X)] \leq b\).
\end{defn}

Given a collection of utility functions encoding the preferences of different
individuals, we say a decision policy \(d\) is Pareto dominated if there exists
a feasible alternative \(d'\) such that none of the decision makers prefers
\(d\) over \(d'\), and at least one decision maker strictly prefers \(d'\) over
\(d\), a property formalized in Definition~\ref{defn:pareto}.

\begin{defn}
\label{defn:pareto}
  Suppose \(\mathcal{U}\) is a collection of utility functions. A decision
  policy \(d\) is \emph{Pareto dominated} if there exists a feasible alternative
  \(d'\) such that \(u(d') \geq u(d)\) for all \(u \in \mathcal{U}\), and there
  exists \(u' \in \mathcal{U}\) such that \(u'(d') > u'(d)\). A policy \(d\) is
  \emph{strongly Pareto dominated} if there exists a feasible alternative \(d'\)
  such that \(u(d') > u(d)\) for all \(u \in \mathcal{U}\). A policy \(d\) is
  \emph{Pareto efficient} if it is feasible and not Pareto dominated, and the
  \emph{Pareto frontier} is the set of Pareto efficient policies.
\end{defn}

To develop intuition about the structure of causally fair decision policies, we
continue working through our illustrative example of college admissions. We
consider a collection of decision makers with utilities \(\mathcal{U}\) of the
form in Eq.~\eqref{eq:util}, for \(\lambda \geq 0\). In this example, decision
makers differ in their preferences for diversity (as determined by \(\lambda\)),
but otherwise have similar preferences. We call such a collection of utilities
\emph{consistent modulo \(\alpha\)}.

\begin{defn}
\label{defn:eight}
  We say that a set of utilities \(\mathcal{U}\) is \emph{consistent modulo
  \(\alpha\)} if, for any \(u, u' \in \mathcal{U}\):
  \begin{enumerate}
    \item For any \(x\), \(\sgn(u(x)) = \sgn(u'(x))\);
    \item For any \(x_1\) and \(x_2\) such that \(\alpha(x_1) = \alpha(x_2)\),
      \(u(x_1) > u(x_2)\) if and only if \(u'(x_1) > u'(x_2)\).
  \end{enumerate}
\end{defn}

For consistent utilities, the Pareto frontier takes a particularly simple form,
represented by (a subset of) group-specific threshold policies.

\begin{prop}
\label{prop:threshold}
  Suppose \(\mathcal{U}\) is a set of utilities that is consistent modulo
  \(\alpha\). Then any Pareto efficient decision policy \(d\) is a multiple
  threshold policy. That is, for any \(u \in \mathcal{U}\), there exist
  group-specific constants \(t_{a} \geq 0\) such that, a.s.:
  \begin{equation}
    d(x)=
    \begin{cases}
      1 & u(x) > t_{\alpha(x)}, \\
      0 & u(x) < t_{\alpha(x)}. \\
    \end{cases}
  \end{equation}
\end{prop}

The proof of Proposition~\ref{prop:threshold} is in the Appendix.\footnote{%
  \label{fn:thresholds} In the statement of the proposition, we do not specify
  what happens at the thresholds \(u(x) = t_{\alpha(x)}\) themselves, as one can
  typically ignore the exact manner in which decisions are made at the
  threshold. Specifically, given a threshold policy \(d\), we can construct a
  standardized threshold policy \(d'\) that is constant within group at the
  threshold (i.e., \(d'(x) = c_{\alpha(x)}\) when \(u(x) = t_{\alpha(x)}\)), and
  for which: (1) \(\EE[d'(X)|A] = \EE[d(X)|A]\); and (2) \(u(d') = u(d)\). In
  our running example, this means we can standardize threshold policies so that
  applicants at the threshold are admitted with the same group-specific
  probability.
}

\subsection{An Empirical Example}

With these preliminaries in place, we now empirically explore the structure of
causally fair decision policies in the context of our stylized example of
college admissions, given by the causal DAG in Figure~\ref{fig:dag}. In the
hypothetical pool of 100,000 applicants we consider, applicants in the target
race group \(a_1\) have, on average, fewer educational opportunities than those
applicants in group \(a_0\), which leads to lower average academic preparedness,
as well as lower average test scores. See Section~\ref{appendix:example} in the
Appendix for additional details, including the specific structural equations we
use.

\begin{figure}[t]
  \centering
  \includegraphics[width=\linewidth,clip]{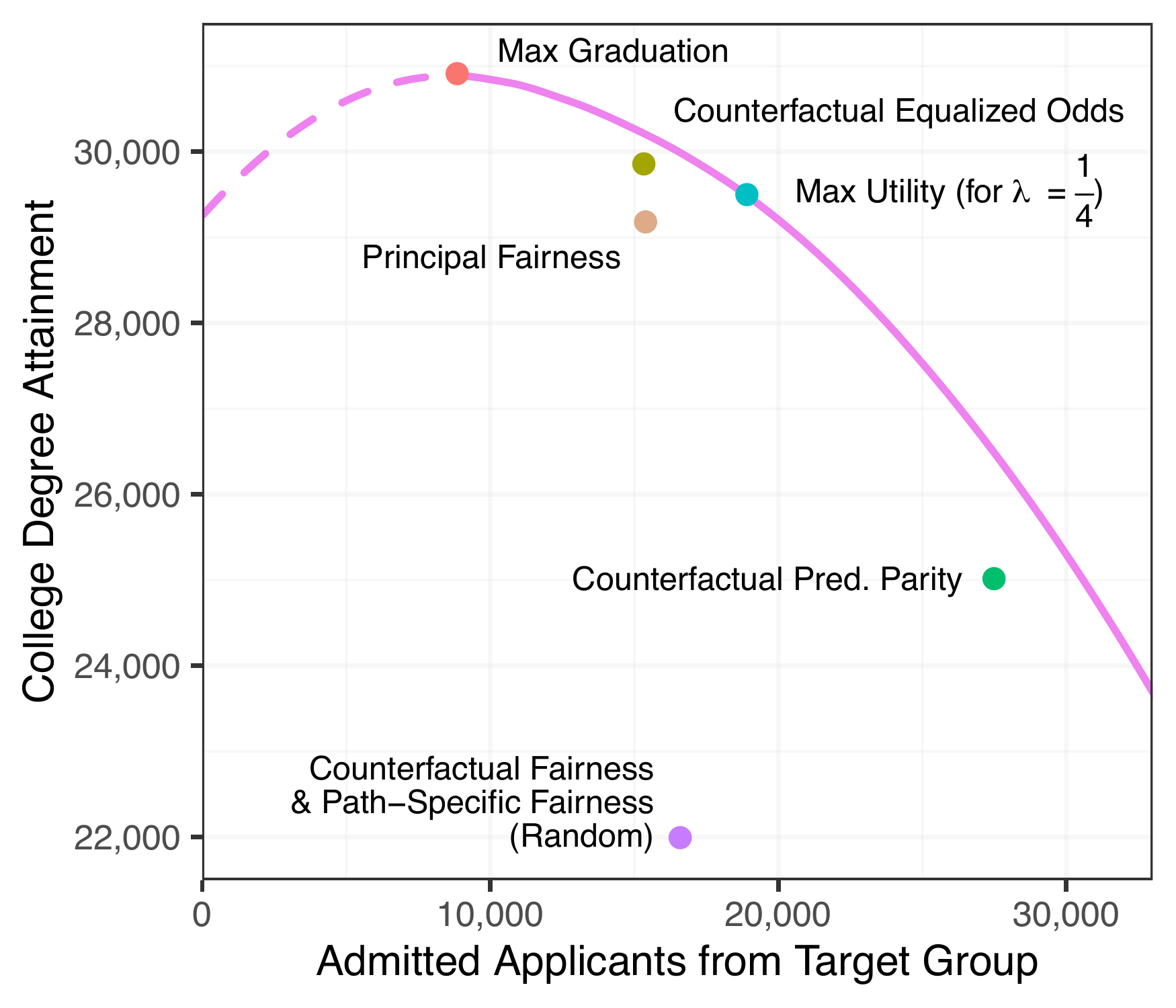}
  \vspace{-7mm}
  \caption{%
    Utility-maximizing policies for various definitions of causal fairness in an
    illustrative example of college admissions, with the Pareto frontier
    depicted by the solid purple curve. For path-specific fairness, we set
    \(\Pi\) equal to the single path \(A \rightarrow E \rightarrow T \rightarrow
    D\), and set \(W = X\). For each causal fairness definition, the  depicted
    constrained policies are strongly Pareto dominated, meaning there is an
    alternative feasible policy that simultaneously achieves greater
    student-body diversity and higher college degree attainment. Our analytical
    results show, more generally, that under mild distributional assumptions,
    every policy constrained to satisfy these causal fairness definitions is
    strongly Pareto dominated.
  }
\label{fig:frontier_01_pred_par}
\end{figure}

For the utility function in Eq.~\eqref{eq:util} with \(\lambda = \tfrac 1 4\),
we apply Theorem~\ref{thm:lp} to compute utility-maximizing policies for each of
the above causal definitions of fairness. We plot the results in
Figure~\ref{fig:frontier_01_pred_par}, where, for each policy, the horizontal
axis shows the expected number of admitted applicants from the target race
group, and the vertical axis shows the expected number of college graduates.
Additionally, for the family of utilities \(\C U\) given by Eq.~\eqref{eq:util}
for \(\lambda \geq 0\), we depict the Pareto frontier by the solid purple curve,
computed via Proposition~\ref{prop:threshold}.\footnote{%
  For all the cases we consider, the optimal policies admit the maximum
  proportion of students allowed under the budget \(b\) (i.e., \(\Pr(D = 1) =
  b\)). To compute the Pareto frontier in Figure~\ref{fig:frontier_01_pred_par},
  it is sufficient---by Proposition~\ref{prop:threshold} and
  Footnote~\ref{fn:thresholds}---to sweep over (standardized) group-specific
  threshold policies relative to the utility \(u_0(x) = \EE[Y(1) | X = x]\).
}
For reference, the dashed purple line corresponds to max-utility policies
constrained to satisfy the level of diversity indicated on the \(x\)-axis,
though these policies are not on the Pareto frontier, as they result in fewer
college graduates and lower diversity than the policy that maximizes graduation
alone (indicated by the ``max graduation'' point in Figure
\ref{fig:frontier_01_pred_par}).

For each fairness definition, the depicted policies are strongly Pareto
dominated, meaning that there is an alternative feasible policy favored by all
decision makers with preferences in \(\mathcal{U}\). In particular, for each
definition of causal fairness, there is an alternative feasible policy in which
one simultaneously achieves  more student-body diversity and more college
graduates. In some instances, the efficiency gap is quite stark.
Utility-maximizing policies constrained to satisfy either counterfactual
fairness or path-specific fairness require one to admit each applicant
independently with fixed probability \(b\) (where \(b\) is the budget),
regardless of academic preparedness or group membership.\footnote{%
  For path-specific fairness, we set \(\Pi\) equal to the single path \(A
  \rightarrow E \rightarrow T \rightarrow D\), and set \(W = X\) in this
  example.
}
These results show that constraining decision-making algorithms to satisfy
popular definitions of causal fairness can have unintended consequences, and may
even harm the very groups they were ostensibly designed to protect.

\subsection{The Statistical Structure of Causally Fair Policies}

The patterns illustrated in Figure~\ref{fig:frontier_01_pred_par} and discussed
above are not idiosyncracies of our particular example, but rather hold quite
generally. Indeed, Theorem~\ref{thm:dist} shows that for \emph{almost every}
joint distribution of \(X\), \(Y(0)\), and \(Y(1)\) such that \(u(X)\) has a
density, any decision policy satisfying counterfactual equalized odds or
conditional principal fairness is Pareto dominated. Similarly, for almost every
joint distribution of \(X\) and \(X_{\Pi, A, a}\), we show that policies
satisfying path-specific fairness (including counterfactual fairness) are Pareto
dominated. (NB: The analogous statement for counterfactual predictive parity is
not true, which we address in Proposition~\ref{prop:pred_parity}.)

The notion of \emph{almost every} distribution that we use here was formalized
by \citet{christensen1972sets}, \citet{hunt1992prevalence},
\citet{anderson2001genericity}, and others \citep[cf.][for a
review]{ott2005prevalence}. Suppose, for a moment, that combinations of
covariates and outcomes take values in a finite set of size \(m\). Then the
space of joint distributions on covariates and outcomes can be represented by
the unit \((m-1)\)-simplex: \(\Delta^{m-1} = \{p \in \mathbb{R}^{m} \mid p_i
\geq 0 \ \text{and} \ \sum_{i=1}^m p_i = 1\}\). Since \(\Delta^{m-1}\) is a
subset of an \((m-1)\)-dimensional hyperplane in \(\mathbb{R}^m\), it inherits
the usual Lebesgue measure on \(\mathbb{R}^{m-1}\). In this finite-dimensional
setting, \emph{almost every} distribution means a subset of distributions that
has full Lebesgue measure on the simplex. Given a property that holds for almost
every distribution in this sense, that property holds almost surely under any
probability distribution on the space of distributions that is described by a
density on the simplex. We use a generalization of this basic idea that extends
to infinite-dimensional spaces, allowing us to consider distributions with
arbitrary support. (See the Appendix for further details.)

To prove this result, we make relatively mild restrictions on the set of
distributions and utilities we consider to exclude degenerate cases, as
formalized by Definition~\ref{defn:fine} below.

\begin{defn}
\label{defn:fine}
  Let \(\C G\) be a collection of functions from \(\C Z\) to \(\B R^d\) for some
  set \(\C Z\). We say that a distribution of \(Z\) on \(\C Z\) is \emph{\(\C
  G\)-fine} if \(g(Z)\) has a density for all \(g \in \C G\).
\end{defn}

In particular, \(\C U\)-fineness ensures that the distribution of \(u(X)\) has a
density. In the absence of \(\C U\)-fineness, corner cases can arise in which an
especially large number of policies may be Pareto efficient, in particular when
\(u(X)\) has large atoms and \(X\) can be used to predict the potential outcomes
\(Y(0)\) and \(Y(1)\) even after conditioning on \(u(X)\). See
Prop.~\ref{prop:counterexample} for details. Our example of college admissions,
where \(\C U\) is defined by Eq.~\eqref{eq:util}, is \(\C U\)-fine.

\begin{thm}
\label{thm:dist}
  Suppose \(\C U\) is a set of utilities consistent modulo \(\alpha\). Further
  suppose that  for all \(a \in \C A\) there exist a \(\C U\)-fine distribution
  of \(X\) and a utility \(u \in \C U\) such that \(\Pr(u(X) > 0, A = a) > 0\),
  where \(A = \alpha(X)\). Then,
  \begin{itemize}
    \item For almost every \(\C U\)-fine distribution of \(X\) and \(Y(1)\), any
      decision policy satisfying counterfactual equalized odds is strongly
      Pareto dominated.
    \item If \(|\img(\omega)| < \infty\) and there exists a \(\C U\)-fine
      distribution of \(X\) such that \(\Pr(A = a, W = w) > 0\) for all \(a \in
      \C A\) and \(w \in \img(\omega)\), where \(W = \omega(X)\), then, for
      almost every \(\C U\)-fine joint distribution of \(X\), \(Y(0)\), and
      \(Y(1)\), any decision policy satisfying conditional principal fairness is
      strongly Pareto dominated.
    \item If \(|\img(\omega)| < \infty\) and  there exists a \(\C U\)-fine
      distribution of \(X\) such that \(\Pr(A = a, W = w_i) > 0\) for all \(a
      \in \C A\) and some distinct \(w_0, w_1 \in \img(\omega)\), then, for
      almost every \(\C U^{\C A}\)-fine joint distributions of \(A\) and the
      counterfactuals \(X_{\Pi, A, a'}\), any decision policy satisfying
      path-specific fairness is strongly Pareto dominated.\footnote{%
        Here, \(u^\C A : (x_a)_{a \in \C A} \mapsto (u(x_a))_{a \in \C A}\) and
        \(\C U^{\C A}\) is the set of \(u^{\C A}\) for \(u \in \C U\). In other
        words, the requirement is that the joint distribution of the \(u(X_{\Pi,
        A, a})\) has a density.
      }
  \end{itemize}
\end{thm}

The proof of Theorem~\ref{thm:dist} is given in the Appendix. At a high-level,
the proof proceed in three steps, which we outline below using the example of
counterfactual equalized odds. First, we show that for almost every fixed \(\C
U\)-fine joint distribution \(\mu\) of \(X\) and \(Y(1)\) there is at most one
policy \(d^*(x)\) satisfying counterfactual equalized odds that is not strongly
Pareto dominated. To see why, note that for any specific \(y_0\), since
counterfactual equalized odds requires that \(D \indep A \mid Y(1) = y_0\),
setting the threshold for one group determines the thresholds for all the
others; the budget constraint then can be used to fix the threshold for the
original group. Second, we construct a ``slice'' around \(\mu\) such that for
any distribution \(\nu\) in the slice, \(d^*(x)\) is still the only policy that
can potentially lie on the Pareto frontier while satisfying counterfactual
equalized odds. We create the slice by strategically perturbing \(\mu\) only
where \(Y(1) = y_1\), for some \(y_1 \neq y_0\). This perturbation moves mass
from one side of the thresholds of \(d^*(x)\) to the other, consequently
breaking the balance requirement \(D \indep A \mid Y(1) = y_1\) for almost every
\(\nu\) in the slice. This phenomenon is similar to the problem of
infra-marginality~\citep{simoiu2017problem, ayres2002outcome}, which likewise
afflicts non-causal notions of fairness~\citep{corbett2017algorithmic,
corbett2018measure}. Finally, we appeal to the notion of prevalence to stitch
the slices together, showing that for almost every distribution, any policy
satisfying counterfactual equalized odds is strongly Pareto dominated. Analogous
versions of this general argument apply to the cases of conditional principal
fairness and path-specific fairness.\footnote{%
  This argument does not depend in an essential way on the
  definitions being causal. In Corollary~\ref{cor:eo}, we show an analogous
  result for the non-counterfactual version of equalized odds.
}

In some common settings, path-specific fairness with \(W = X\) constrains
decisions so severely that the only allowable policies are constant (i.e.,
\(d(x_1) = d(x_2)\) for all \(x_1, x_2 \in \C X\)). For instance, in our running
example, path-specific fairness requires admitting all applicants with the same
probability, irrespective of academic preparation or group membership. Thus, all
applicants are admitted with probability \(b\), where \(b\) is the budget, under
the optimal policy constrained to satisfy path-specific fairness.

To build intuition for this result, we sketch the argument for a finite
covariate space \(\C X\). Given a policy \(d\) that satisfies path-specific
fairness, select \(x^* \in \arg \max_{x \in \C X} d(x)\).

By the definition of path-specific fairness, for any \(a \in \C A\),
\begin{equation}
  \begin{aligned}
  \label{eq:ps-sketch}
    d(x^*) & = \EE[D_{\Pi, A, a} \mid X = x^*] \\
      & = \ \sum_{\mathclap{x \in \alpha^{-1}(a)}} \ d(x) \cdot \Pr(X_{\Pi, A,
      a} = x \mid X = x^*).
  \end{aligned}
\end{equation}
That is, the probability of an individual with covariates \(x^*\) receiving a
positive decision must be the average probability of the individuals with
covariates \(x\) in group \(a\) receiving a positive decision, weighted by the
probability that an individual with covariates \(x^*\) in the real world
\emph{would} have covariates \(x\) counterfactually.

Next, we suppose that there exists an \(a' \in \C A\) such that \(\Pr(X_{\Pi,
A, a'} = x \mid X = x^*) > 0\) for all \(x \in \alpha^{-1}(a')\). In this case,
because \(d(x) \leq d(x^*)\) for all \(x \in \C X\), Eq.~\eqref{eq:ps-sketch}
shows that in fact \(d(x) = d(x^*)\) for all \(x \in \alpha^{-1}(a')\).

Now, let \(x'\) be arbitrary. Again, by the definition of path-specific
fairness, we have that
\begin{align*}
  d(x') & = \EE[D_{\Pi, A, a'} \mid X = x'] \\
    & = \ \sum_{\mathclap{x \in \alpha^{-1}(a')}} \ d(x) \cdot \Pr(X_{\Pi, A,
    a'} = x \mid X = x') \\
    &=  \ \sum_{\mathclap{x \in \alpha^{-1}(a')}} \ d(x^*) \cdot \Pr(X_{\Pi, A,
    a'} = x \mid X = x^*), \\
    &= d(x^*),
\end{align*}
where we use in the third equality the fact \(d(x) = d(x^*)\) for all \(x \in
\alpha^{-1}(a')\), and in the final equality the fact that \(X_{\Pi, A, a'}\) is
supported on \(\alpha^{-1}(a')\).

Theorem~\ref{thm:path_specific} formalizes and extends this argument to more
general settings, where \(\Pr(X_{\Pi, A, a'} = x \mid X = x^*)\) is not
necessarily positive for all \(x \in \alpha^{-1}(a')\). The proof of
Theorem~\ref{thm:path_specific} is in the Appendix, along with extensions to
continuous covariate spaces and a more complete characterization of \(\Pi\)-fair
policies for finite \(\C X\).

\begin{thm}
\label{thm:path_specific}
  Suppose \(\C X\) is finite and  \(\Pr(X = x) > 0\) for all \(x \in \C X\).
  Suppose \(Z = \zeta(X)\) is a random variable such that:
  \begin{enumerate}
    \item \(Z = Z_{\Pi, A, a'}\) for all \(a' \in \C A\),
    \item \(\Pr(X_{\Pi, A, a'} = x' \mid X = x) > 0\) for all \(a' \in \C A\)
      such that \(\alpha(x) \neq a'\) and \(x, x' \in \C X\) such that
      \(\zeta(x) = \zeta(x')\).
  \end{enumerate}
  Then, for any \(\Pi\)-fair policy \(d\), with  \(W = X\), there exists a
  function \(f\) such that \(d(X) = f(Z)\), i.e., \(d\) is constant across
  individuals having the same value of \(Z\).
\end{thm}

The first condition of Theorem~\ref{thm:path_specific} holds for any reduced set
of covariates \(Z\) that is not causally affected by changes in \(A\) (e.g.,
\(Z\) is not a descendent of \(A\)). The second condition requires that among
individuals with covariates \(x\), a positive fraction have covariates  \(x'\)
in a counterfactual world in which they belonged to another group \(a'\).
Because \(\zeta(x)\) is the same in the real and counterfactual worlds---since
\(Z\) is unaffected by \(A\), by the first condition---we only consider \(x'\)
such that \(\zeta(x') = \zeta(x)\) in the second condition.

In our running example, the only non-race covariate is test score, which is
downstream of race. Further, among students with a given test score, a positive
fraction achieve any other test score in the counterfactual world in which their
race is altered. As such, the empty set of reduced covariates---formally encoded
by setting \(\zeta\) to a constant function---satisfies the conditions of
Theorem~\ref{thm:path_specific}. The theorem then implies that under any
\(\Pi\)-fair policy, every applicant is admitted with equal probability.

Even when decisions are not perfectly uniform lotteries, as in our admissions
example, Theorem~\ref{thm:path_specific} suggests that enforcing
\(\Pi\)-fairness can lead to unexpected outcomes. For instance, suppose we
modify our admissions example to additionally include age as a covariate that is
causally unconnected to race---as some past work has done. In that case,
\(\Pi\)-fair policies would admit students based on their age alone,
irrespective of test score or race. Although in some cases such restrictive
policies might be desirable, this strong structural constraint implied by
\(\Pi\)-fairness appears to be a largely unintended consequence of the
mathematical formalism.

The conditions of Theorem~\ref{thm:path_specific} are relatively mild, but do
not hold in every setting. Suppose that in our admissions example it were the
case that \(T_{\Pi, A, a_0} = T_{\Pi, A, a_1} + c\) for some constant
\(c\)---that is, suppose the effect of intervening on race is a constant change
to an applicant's test score. Then the second condition of
Theorem~\ref{thm:path_specific} would no longer hold for a constant \(\zeta\).
Indeed, any multiple-threshold policy in which \(t_{a_0} = t_{a_1} + c\) would
be \(\Pi\)-fair. In practice, though, such deterministic counterfactuals would
seem to be the exception rather than the rule. For example, it seems reasonable
to expect that test scores would depend on race in complex ways that induce
considerable heterogeneity.

Lastly, we note that \(W \neq X\) in some variants of path-specific fairness
\citep[e.g.,][]{zhang2018fairness, nabi2018fair}, in which case
Theorem~\ref{thm:path_specific} does not apply. Although, in that case, policies
are typically still Pareto dominated in accordance with Theorem~\ref{thm:dist}.

We conclude our analysis by investigating counterfactual predictive parity, the
least demanding of the causal notions of fairness we have considered, requiring
only that \(Y(1) \indep A \mid D = 0\). As such, it is in general possible to
have a policy on the Pareto frontier that satisfies this condition. However, in
Proposition~\ref{prop:pred_parity}, we show that this cannot happen in some
common cases, including our example of college admissions.

In that setting, when the target group has lower average graduation rates---a
pattern that often motivates efforts to actively increase diversity---decision
policies constrained to satisfy counterfactual predictive parity are Pareto
dominated. The proof of the proposition is in the Appendix.

\begin{prop}
\label{prop:pred_parity}
  Suppose \(\C A = \{a_0, a_1\}\), and consider the family \(\C U\) of utility
  functions of the form
  \begin{equation*}
    u(x) =  r(x) + \lambda \cdot \B 1_{\alpha(x) = a_1},
  \end{equation*}
  indexed by \(\lambda \geq 0\), where \(r(x) = \EE[Y(1) \mid X = x]\). Suppose
  the conditional distributions of \(r(X)\) given \(A\) are beta distributed,
  i.e.,
  \begin{equation*}
    r(X) \mid A = a \sim \bbeta(\mu_a, v),
  \end{equation*}
  with \(v > 2\) and \(\mu_{a_0} > \mu_{a_1} > 1/v\).\footnote{%
    Here we parameterize the beta distribution in terms of its mean \(\mu\) and
    sample size \(v\). In terms of the common, alternative \(\alpha\)-\(\beta\)
    parameterization, \(\mu = \alpha / (\alpha + \beta)\) and \(v = \alpha +
    \beta\).
  }
  Then any policy satisfying counterfactual predictive parity is strongly Pareto
  dominated.
\end{prop}

\section{Discussion}

We have worked to collect, synthesize, and investigate several causal
conceptions of fairness that recently have appeared in the machine learning
literature. These definitions formalize intuitively desirable properties---for
example, minimizing the direct and indirect effects of race on decisions. But,
as we have shown both analytically and with a synthetic example, they can,
perhaps surprisingly, lead to policies with unintended downstream outcomes. In
contrast to prior impossibility results~\citep{kleinberg2016inherent,
chouldechova2017fair}, in which different formal notions of fairness are shown
to be in conflict with each other, we demonstrate trade-offs between formal
notions of fairness and resulting social welfare. For instance, in our running
example of college admissions, enforcing various causal fairness definitions can
lead to a student body that is both less academically prepared and less diverse
than what one could achieve under natural alternative policies, potentially
harming the very groups these definitions were ostensibly designed to protect.
Our results thus highlight a gap between the goals and potential consequences of
popular causal approaches to fairness.

What, then, is the role of causal reasoning in designing equitable algorithms?
Under a consequentialist perspective to algorithm design~\citep{L2BF,
cai2020fair, liang2021algorithmic}, one aims to construct policies with the most
desirable expected outcomes, a task that inherently demands causal reasoning.
Formally, this approach corresponds to solving the unconstrained optimization
problem in Eq.~\eqref{eq:opt}, where preferences for diversity may be directly
encoded in the utility function itself, rather than by constraining the class of
policies, mitigating potentially problematic consequences. While conceptually
appealing, this consequentialist approach still faces considerable practical
challenges, including estimating the expected effects of decisions, and
eliciting preferences over outcomes.

Our analysis illustrates some of the limitations of mathematical formalizations
of fairness, reinforcing the need to explicitly consider the consequences of
actions, particularly when decisions are automated and carried out at scale.
Looking forward, we hope our work clarifies the ways in which causal reasoning
can aid the equitable design of algorithms.

\section*{Acknowledgements}

We thank Guillaume Basse, Jennifer Hill, and Ravi Sojitra for helpful
conversations. H.N was supported by a Stanford Knight-Hennessy Scholarship and
the National Science Foundation Graduate Research Fellowship under Grant No.
DGE-1656518. J.G was supported by a Stanford Knight-Hennessy Scholarship. R.S.
was supported by the NSF Program on Fairness in AI in Collaboration with Amazon
under the award ``FAI: End-to-End Fairness for Algorithm-in-the-Loop Decision
Making in the Public Sector,'' no. IIS-2040898. Any opinion, findings, and
conclusions or recommendations expressed in this material are those of the
author(s) and do not necessarily reflect the views of the National Science
Foundation or Amazon. Code to reproduce our results is available at
\url{https://github.com/stanford-policylab/causal-fairness}.

\bibliography{bibli}
\bibliographystyle{icml2022}

\appendix

\setcounter{thm}{0}
\setcounter{defn}{0}
\renewcommand{\thethm}{\Alph{section}.\arabic{thm}}
\renewcommand{\thelem}{\Alph{section}.\arabic{lem}}
\renewcommand{\thedefn}{\Alph{section}.\arabic{defn}}
\renewcommand{\theprop}{\Alph{section}.\arabic{prop}}
\renewcommand{\thecor}{\Alph{section}.\arabic{cor}}

\section{Path-specific Counterfactuals}
\label{sec:path-specific-counterfactuals}

Constructing policies which satisfy path-specific fairness requires computing
path-specific counterfactual values of features. In
Algorithm~\ref{alg:counterfactuals}, we describe the formal construction of
path-specific counterfactuals \(Z_{\Pi,a,a'}\), for an arbitrary variable \(Z\)
(or collection of variables) in the DAG. To generate a sample \(Z_{\Pi,a,a'}^*\)
from the distribution of \(Z_{\Pi,a,a'}\), we first sample values \(U^*_j\) for
the exogenous variables. Then, in the first loop, we traverse the DAG in
topological order, setting \(A\) to \(a\) and iteratively computing values
\(V_j^*\) of the other nodes based on the structural equations in the usual
fashion. In the second loop, we set \(A\) to \(a'\), and then iteratively
compute values \(\overline{V_j}^*\) for each node. \(\overline{V_j}^*\) is
computed using the structural equation at that node, with value
\(\overline{V_{\ell}}^*\) for each of its parents that are connected to it along
a path in \(\Pi\), and the value \(V^*_{\ell}\) for all its other parents.
Finally, we set \(Z_{\Pi,a,a'}^*\) to \(\overline{Z}^*\).

\begin{algorithm2e}[ht]
\KwData{\(\mathcal{G}\) (topologically ordered), \(\Pi\), \(a\), and \(a'\)}
\KwResult{A sample \(Z_{\Pi, a, a'}^*\) from \(Z_{\Pi, a, a'}\)}
  \vspace{2mm}
  Sample values \(\{U_{j}^{*}\}\) for the exogenous variables
  \vspace{2mm}
  \tcc{Compute counterfactuals by setting \(A\) to \(a\)}
    \For{\(j = 1, \ldots, m\)}{
      \eIf{\(V_j = A\)}{
        \(V_j^* \gets a\)
      }{
        \(\wp(V_{j})^{*} \gets \{V_{\ell}^* \mid V_{\ell} \in \wp(V_j)\}\) \\
        \(V_{j}^{*} \gets f_{V_{j}}(\wp(V_{j})^{*}, U_{j}^{*})\)\\
      }
    }
  \vspace{2mm}
  \tcc{%
    Compute counterfactuals by setting \(A\) to \(a'\) and propagating values
    along paths in \(\Pi\)
  }
  \For{\(j = 1, \ldots, m\)}{
    \eIf{\(V_j = A\)}{
      \(\overline{V}_j^* \gets a'\)
    }{
      \For{\(V_k \in \wp(V_j)\)}{
        \eIf{edge \((V_k, V_j)\) lies on a path in \(\Pi\)}{
          \(V_{k}^{\dagger} \gets \overline{V}_{k}^*\)
        }{
          \(V_{k}^{\dagger} \gets V_{k}^*\)
        }
      }
    \(\wp(V_{j})^{\dagger} \gets \{V_{\ell}^{\dagger} \mid V_{\ell} \in
    \wp(V_j)\}\) \\ \(\overline{V}_j^* \gets f_{V_j}(\wp(V_j)^{\dagger},
    U_j^*)\)
    }
  }
  \vspace{2mm}
  \(Z_{\Pi, a, a'}^* \gets \overline{Z}^*\)
  \caption{Path-specific counterfactuals}
\label{alg:counterfactuals}
\end{algorithm2e}

\section{Constructing Causally Fair Policies}

In order to construct causally fair policies, we prove that the optimization
problem in Eq.~\eqref{eq:opt} can be efficiently solved as a single linear
program---in the case of counterfactual equalized odds, conditional principal
fairness, counterfactual fairness, and path-specific fairness---or as a series
of linear programs in the case of counterfactual predictive parity.

\label{app:proofs}

\begin{thm}
\label{thm:lp}
  Consider the optimization problem given in Eq.~\eqref{eq:opt}.
  \begin{enumerate}
    \item If \(\mathcal{C}\) is the class of policies that satisfies
      counterfactual equalized odds or conditional principal fairness, and the
      distribution of \((X, Y(0), Y(1))\) is known and supported on a finite set
      of size \(n\), then a utility-maximizing policy constrained to lie in
      \(\mathcal{C}\) can be constructed via a linear program with \(O(n)\)
      variables and constraints.
    \item If \(\mathcal{C}\) is the class of policies that satisfies
      path-specific fairness (including counterfactual fairness), and the
      distribution of \((X, D_{\Pi, A, a})\) is known and supported on a finite
      set of size \(n\), then a utility-maximizing policy constrained to lie in
      \(\mathcal{C}\) can be constructed via a linear program with \(O(n)\)
      variables and constraints.
    \item Suppose \(\mathcal{C}\) is the class of policies that satisfies
      counterfactual predictive parity, that the distribution of \((X, Y(1))\)
      is known and supported on a finite set of size \(n\), and that the
      optimization problem in Eq.~\eqref{eq:opt} has a feasible solution.
      Further suppose \(Y(1)\) is supported on \(k\) points, and let
      \(\Delta^{k-1} = \{p \in \mathbb{R}^{k} \mid p_i \geq 0 \ \text{and} \
      \sum_{i=1}^k p_i = 1\}\) be the unit \((k-1)\)-simplex. Then one can
      construct a set of linear programs \(\mathcal{L} = \{L(v)\}_{v \in
      \Delta^k}\), with each having \(O(n)\) variables and constraints, such
      that the solution to one of the LPs in \(\mathcal{L}\) is a
      utility-maximizing policy constrained to lie in \(\mathcal{C}\).
  \end{enumerate}
\end{thm}

\begin{proof}
  Let \(\C X = \{x_1, \ldots, x_m\}\); then, we seek decision variables \(d_i\),
  \(i = 1, \ldots, m\), corresponding to the probability of making a positive
  decision for individuals with covariate value \(x_i\). Therefore, we require
  that \(0 \leq d_i \leq 1\).

  Letting \(p_i = \Pr(X = x_i)\) denote the mass of \(X\) at \(x_i\), note that
  the objective function \(\EE[d(X) \cdot u(X)]\) equals \(\sum_{i=1}^m d_i
  \cdot u(x_i) \cdot p_i\) and the budget constraint \(\sum_{i = 1}^m d_i \cdot
  p_i \leq b\) are both linear in the decision variables.

  We now show that each of the causal fairness definitions can be enforced via
  linear constraints. We do so in three parts as listed in theorem.

  \paragraph{Theorem~\ref{thm:lp} Part 1}

  First, we consider counterfactual equalized odds. A decision policy satisfies
  counterfactual equalized odds when \(D \indep A \mid Y(1).\) Since \(D\) is
  binary, this condition is equivalent to the expression \(\EE[d(X) \mid A = a,
  Y(1) = y] = \EE[d(X) \mid Y(1) = y]\) for all \(a \in \C A\) and \(y \in \C
  Y\) such that \(\Pr(Y(1) = y) > 0\). Expanding this expression and replacing
  \(d(x_j)\) by the corresponding decision variable \(d_j\), we obtain that
  \begin{multline*}
    \sum_{i = 1}^m d_i \cdot \Pr(X = x_i \mid A = a, Y(1) = y) \\
      = \sum_{i = 1}^m d_i \cdot \Pr(X = x_i \mid Y(1) = y)
  \end{multline*}
  for each \(a \in \C A\) and each of the finitely many values \(y \in \C Y\)
  such that \(\Pr(Y(1) = y) > 0\). These constraints are linear in the \(d_i\)
  by inspection.

  Next, we consider conditional principal fairness. A decision policy satisfies
  conditional principal fairness when \(D \indep A \mid Y(0), Y(1), W\), where
  \(W = \omega(X)\) denotes a reduced set of the covariates \(X\). Again, since
  \(D\) is binary, this condition is equivalent to the expression \( \EE[d(X)
  \mid A = a, Y(0) = y_0, Y(1) = y_1, W = w] = \EE[d(X) \mid Y(0) = y_0, Y(1) =
  y_1, W = w]\) for all \(y_0\), \(y_1\), and \(w\) satisfying \(\Pr(Y(0) = y_0,
  Y(1) = y_1, W = w) > 0\). As above, expanding this expression and replacing
  \(d(x_j)\) by the corresponding decision variable \(d_j\) yields linear
  constraints of the form
  \begin{multline*}
    \sum_{i = 1}^m d_i \cdot \Pr(X = x_i \mid A = a, S = s) \\
      = \sum_{j = 1}^m d_i \cdot \Pr(X = x_i \mid S = s)
  \end{multline*}
  for each \(a \in \C A\) and each of the finitely many values of \(S = (Y(0),
  Y(1), W)\) such that  \(s = (y_0, y_1, w) \in \C Y \times \C Y \times \C W\)
  satisfies \(\Pr(Y(0) = y_0, Y(1) = y_1, W = w) > 0\). Again, these constraints
  are linear by inspection.

  \paragraph{Theorem~\ref{thm:lp} Part 2}

  Suppose a decision policy satisfies path-specific fairness for a given
  collection of paths \(\Pi\) and a (possibly) reduced set of covariates \(W =
  \omega(X),\) meaning that for every \(a' \in \C A\), \(\EE[D_{\Pi, A, a'}
  \mid W ] = \EE[D \mid W]\).

  Recall from the definition of path-specific counterfactuals that \(D_{\Pi, A,
  a'} = f_{D}(X_{\Pi, A, a'}, U_{D}) = \mathbb{1}_{U_{D} \leq d(X_{\Pi, A,
  a'})}\), where \(U_D \indep \{X_{\Pi, A, a}, X\}\). Since \(W = \omega(X)\),
  \(U_D \indep \{X_{\Pi, A, a}, W\}\), it follows that
  \begin{align*}
    \EE[D_{\Pi, A, a'} &\mid W = w] \\
      &\hspace{-0.25cm}= \sum_{i = 1}^m \EE[D_{\Pi, A, a'} \mid X_{\Pi, A, a} =
        x_i, W = w] \\
      &\hspace{0.75cm} \cdot \Pr(X_{\Pi, A, a} = x_i \mid W = w) \\
      &\hspace{-0.25cm}= \sum_{i = 1}^m \EE[\B 1_{U_D \leq d(X_{\Pi, A, a'})}
        \mid X_{\Pi, A, a} = x_i, W = w] \\
      &\hspace{0.75cm} \cdot \Pr(X_{\Pi, A, a'} = x_i \mid W = w) \\
      &\hspace{-0.25cm}= \sum_{i = 1}^m d(X_{\Pi, A, a'}) \cdot \Pr(X_{\Pi, A,
        a'} = x_i \mid W = w) \\
      &\hspace{-0.25cm}= \sum_{i = 1}^m d_i \cdot \Pr(X_{\Pi, A, a'} = x_i \mid
        W = w).
  \end{align*}
  An analogous calculation yields that \(\EE[D \mid W = w] = \sum_{i = 1}^m d_i
  \cdot \Pr(X = x_i \mid W = w)\). Equating these expressions gives
  \begin{multline*}
    \sum_{i = 1}^m d_i \cdot \Pr(X = x_i \mid W = w) \\
      = \sum_{i = 1}^m d_i \cdot \Pr(X_{\Pi, A, a'} = x_i \mid W = w)
  \end{multline*}
  for each \(a' \in \C A\) and each of the finitely many \(w \in \C W\) such
  that \(\Pr(W = w) > 0\). Again, each of these constraints is linear by
  inspection.

  \paragraph{Theorem~\ref{thm:lp} Part 3}

  A decision policy satisfies counterfactual predictive parity if \(Y(1) \indep
  A \mid D = 0,\) or equivalently, \(\Pr(Y(1) = y \mid A = a, D = 0) = \Pr(Y(1)
  \mid D = 0)\) for all \(a \in \C A\). We may rewrite this expression to
  obtain:
  \begin{align*}
    \dfrac{\Pr(Y(1) = y, A = a, D = 0)}{\Pr(A = a, D = 0 )} = C_{y},
  \end{align*}
  where \(C_{y} = \Pr(Y(1) = y \mid D = 0)\).

  Expanding the numerator on the left-hand side of the above equation yields
  \begin{multline*}
    \Pr(Y(1) = y, A = a, D = 0) \\
      = \sum_{i=1}^m [1 - d_i] \cdot \Pr(Y(1) = y, A = a, X = x_i)
  \end{multline*}

  Similarly, expanding the denominator yields
  \begin{multline*}
    \Pr(Y(1) = y, D = 0) \\
      = \sum_{i=1}^m [1 - d_i] \cdot \Pr(Y(1) = y, X = x_i).
  \end{multline*}
  for each of the finitely many \(y \in \C Y\). Therefore, counterfactual
  predictive parity corresponds to
  \begin{align}
  \label{eq:cpp_constraint}
    \begin{split}
    \sum_{i=1}^m [1 &- d_i] \cdot \Pr(Y(1) = y, A = a, X = x_i) \\
      &= C_y \cdot \sum_{i=1}^m [1 - d_i] \cdot \Pr(Y(1) = y, X = x_i),
    \end{split}
  \end{align}
  for each \(a \in \C A\) and \(y \in \C Y\). Again, these constraints are
  linear in the \(d_i\) by inspection.

  Consider the family of linear programs \(\C L = \{L(v)\}_{v \in \Delta^k}\)
  where the linear program \(L(v)\) has the same objective function
  \(\sum_{i=1}^m d_i \cdot u(x_i) \cdot p_i\) and budget constraint
  \(\sum_{i=1}^m d_i \cdot p_i \leq b\) as before, together with additional
  constraints for each \(a \in \C A\) as in Eq.~\eqref{eq:cpp_constraint}, where
  \(C_{y_i} = v_i\) for \(i = 1, \ldots, k\).

  By assumption, there exists a feasible solution to the optimization problem in
  Eq.~\eqref{eq:opt}, so the solution to at least one program in \(\C L\) is a
  utility-maximizing policy that satisfies counterfactual predictive parity.
\end{proof}

\section{A Stylized Example of College Admissions}
\label{appendix:example}

In the example we consider in Section~\ref{sec:example}, the exogenous variables
\( \mathscr U = \{u_A,u_D, u_E, u_M, u_T, u_Y\}\) in the DAG are independently
distributed as follows:
\begin{align*}
  U_A, U_D, U_Y & \sim \unif(0, 1),\\
  U_E, U_M, U_T & \sim \norm(0, 1).
\end{align*}
For fixed constants \(\mu_A\), \(\beta_{E,0}\), \(\beta_{E,A}\),
\(\beta_{M,0}\), \(\beta_{M,E}\), \(\beta_{T,0}\), \(\beta_{T,E}\),
\(\beta_{T,M}\), \(\beta_{T,B}\), \(\beta_{T,u}\), \(\beta_{Y,0}\),
\(\beta_{Y,D}\), we define the endogenous variables \( \mathscr V = \{A, E, M,
T, D, Y\}\) in the DAG by the following structural equations:
\begin{align*}
  f_A(u_A)
    &=
    \begin{cases}
      a_1 & \text{if} \ u_A \leq \mu_A \\
      a_0 & \text{otherwise}
    \end{cases}, \\
  f_E(a, u_E)
    &= \beta_{E,0} + \beta_{E,A} \cdot \B 1(a = a_1) + u_E, \\
  f_M(e, u_M)
    &= \beta_{M,0} + \beta_{M,E} \cdot e + u_M, \\
  f_T(e, m, u_T)
    &= \beta_{T,0} + \beta_{T,E} \cdot e \\
    & \hspace{5mm} + \beta_{T,M} \cdot m + \beta_{T,B} \cdot e \cdot m +
    \beta_{T,u} \cdot u_T, \\
  f_D(x, u_D)
    &= \B 1(u_D \leq d(x)), \\
  f_Y(m, u_Y, \delta)
    &= \B 1(u_Y \leq \logit^{-1}(\beta_{Y,0} + m + \beta_{Y, D} \cdot \delta)),
\end{align*}
where \(\logit^{-1}(x) = (1 + \exp(-x))^{-1}\) and \(d(x)\) is the decision
policy. In our example, we use constants \(\mu_A = \tfrac 1 3\), \(\beta_{E,0} =
1\), \(\beta_{E,A} = -1\), \(\beta_{M,0} = 0\), \(\beta_{M,E} = 1\),
\(\beta_{T,0} = 50\), \(\beta_{T,E} = 4\), \(\beta_{T,M} = 4\), \(\beta_{T,u} =
7\), \(\beta_{T,B} = 1\), \(\beta_{Y,0} = - \tfrac 1 2\), \(\beta_{Y,D}= \tfrac
1 2\). We also assume a budget \(b=\frac{1}{2}\).

\section{Proof of Proposition~\ref{prop:threshold}}

We begin by more formally defining (multiple) threshold policies. We assume,
without loss of generality, that \(\Pr(A = a) > 0\) for all \(a \in \C A\)
throughout.

\begin{defn}
  Let \(u(x)\) be a utility function. We say that a policy \(d(x)\) is a
  \emph{threshold policy} with respect to \(u\) if there exists some \(t\) such
  that
  \begin{equation*}
    d(x) =
      \begin{cases}
        1   & u(x) > t, \\
        0   & u(x) < t,
      \end{cases}
  \end{equation*}
  and \(d(x) \in [0, 1]\) is arbitrary if \(u(x) = t\). We say that \(d(x)\) is
  a \emph{multiple threshold policy} with respect to \(u\) if there exist
  group-specific constants \(t_a\) for \(a \in \C A\) such that
  \begin{equation*}
    d(x) =
      \begin{cases}
        1   & u(x) > t_{\alpha(x)}, \\
        0   & u(x) < t_{\alpha(x)},
      \end{cases}
  \end{equation*}
  and \(d(x) \in [0, 1]\) is arbitrary if \(u(x) = t_{\alpha(x)}\).
\end{defn}

\begin{rmk}
\label{rmk:unique_threshold}
  In general, it is possible for different thresholds to produce threshold
  policies that are almost surely equal. For instance, if \(u(X) \sim
  \bern(\tfrac 1 2)\), then the policies \(\B 1_{u(X) > p}\) are almost surely
  equal for all \(p \in [0, 1)\). Nevertheless, we speak in general of
  \emph{the} threshold associated with the threshold policy \(d(X)\) unless
  there is ambiguity.
\end{rmk}

We first observe that if \(\C U\) is consistent modulo \(\alpha\), then whether
a decision policy \(d(x)\) is a multiple threshold policy does not depend on our
choice of \(u \in \C U\).

\begin{lem}
\label{lem:canonical}
  Let \(\C U\) be a collection of utilities consistent modulo \(\alpha\), and
  suppose \(d : \C X \to [0,1]\) is a decision rule. If \(d(x)\) is a multiple
  threshold rule with respect to a utility \(u^* \in \C U\), then \(d(x)\) is a
  multiple threshold rule with respect to every \(u \in \C U\). In particular,
  if \(d(x)\) can be represented by non-negative thresholds over \(u^*\), it can
  be represented by non-negative thresholds over any \(u \in \C U\).
\end{lem}

\begin{proof}
  Suppose \(d(x)\) is represented by thresholds \(\{t_a^*\}_{a \in \C A}\) with
  respect to \(u^*\). We construct the thresholds \(\{t_a\}_{a \in \C A}\)
  explicitly.

  Fix \(a \in \C A\) and suppose that there exists \(x^* \in \alpha^{-1}(a)\)
  such that \(u^*(x^*) = t_a^*\). Then set \(t_a = u(x^*)\). Now, if \(u(x) >
  t_a = u(x^*)\) then,  by consistency modulo \(\alpha\), \(u^*(x) > u^*(x^*) =
  t_a^*\). Similarly if \(u(x) < t_a\) then \(u^*(x) < t_a^*\). We also note
  that by consistency modulo \(\alpha\), \(\sgn(t_a) = \sgn(u(x^*)) =
  \sgn(u^*(x^*)) = \sgn(t_a^*)\).

  If there is no \(x^* \in \alpha^{-1}(a)\) such that \(u^*(x^*) = t_a^*\), then
  let
  \begin{equation*}
    t_a = \inf_{x \in S_a} u(x)
  \end{equation*}
  where \(S_a = \{x \in \alpha^{-1}(a) \mid u^*(x) > t_a^* \}\). Note that since
  \(\sgn(u(x)) = \sgn(u^*(x))\) for all \(x\) by consistency modulo \(\alpha\),
  if \(t_a^* \geq 0\), it follows that \(t_a \geq 0\) as well.

  We need to show in this case also that if \(u(x) > t_a\) then \(u^*(x) >
  t_a^*\), and if \(u(x) < t_a\) then \(u^*(x) < t_a^*\). To do so, let \(x \in
  \alpha^{-1}(a)\) be arbitrary, and suppose \(u(x) > t_a\). Then, by
  definition, there exists \(x' \in \alpha^{-1}(a)\) such that \(u(x) > u(x') >
  t_a\) and \(u^*(x') > t_a^*\), whence \(u^*(x) > u^*(x') > t_a^*\) by
  consistency modulo \(\alpha\). On the other hand, if \(u(x) < t_a\), it
  follows by the definition of \(t_a\) that \(u^*(x) \leq t_a^*\); since
  \(u^*(x) \neq t_a^*\) by hypothesis, it follows that \(u^*(x) < t_a^*\).

  Therefore, it follows in both cases that for
  \(x \in \alpha^{-1}(a)\), if \(u(x) > t_a\) then \(u^*(x) > t_a^*\), and if
  \(u(x) < t_a\) then \(u^*(x) < t_a^*\). Therefore
  \begin{equation*}
    d(x)=
    \begin{cases}
      1 & \text{if} \ u(x) > t_{\alpha(x)}, \\
      0 & \text{if} \ u(x) < t_{\alpha(x)}, \\
    \end{cases}
  \end{equation*}
  i.e., \(d(x)\) is a multiple threshold policy with respect to \(u\). Moreover,
  as noted above, if \(t_a^* \geq 0\) for all \(a \in \C A\), then \(t_a \geq
  0\) for all \(a \in \C A\).
\end{proof}

We now prove the following strengthening of Prop.~\ref{prop:threshold}.

\begin{lem}
\label{lem:threshold}
  Let \(\C U\) be a collection of utilities consistent modulo \(\alpha\). Let
  \(d(x)\) be a feasible decision policy that is not a.s.\ a multiple threshold
  policy with non-negative thresholds with respect to \(\C U\), then \(d(x)\) is
  strongly Pareto dominated.
\end{lem}

\begin{proof}
  We prove the claim in two parts. First, we show that any policy that is not a
  multiple threshold policy is strongly Pareto dominated. Then, we show that any
  multiple threshold policy that cannot be represented with non-negative
  thresholds is strongly Pareto dominated.

  If \(d(x)\) is not a multiple threshold policy, then there exists a \(u \in \C
  U\) and \(a^* \in \C A\) such that \(d(x)\) is not a threshold policy when
  restricted to \(\alpha^{-1}(a^*)\) with respect to \(u\).

  We will construct an alternative policy \(d'(x)\) that attains strictly
  greater utility on \(\alpha^{-1}(a^*)\) and is identical elsewhere. Thus,
  without loss of generality, we assume there is a single group, i.e.,
  \(\alpha(x) = a^*\). The proof proceeds heuristically by moving some of the
  mass below a threshold to above a threshold to create a feasible policy with
  improved utility.

  Let \(b = \EE[d(X)]\). Define
  \begin{align*}
    m^{\low}(t)
      &= \EE[d(X) \cdot \B 1_{u(X) < t}], \\
    m^{\up}(t)
      &= \EE[(1 - d(X)) \cdot \B 1_{u(X) > t}].
  \end{align*}
  We show that there exists \(t^*\) such that \(m^{\up}(t^*) > 0\) and
  \(m^{\low}(t^*) > 0\). For, if not, consider
  \begin{equation*}
    \tilde t = \inf \{t \in \B R : m^{\up}(t) = 0\}.
  \end{equation*}
  Note that \(d(X) \cdot \B 1_{u(X) > \tilde t} = \B 1_{u(X) > \tilde t}\) a.s.
  If \(\tilde{t} = -\infty\), then by definition \(d(X) = 1\) a.s., which is a
  threshold policy, violating our assumption on \(d(X)\). If \(\tilde t >
  -\infty\), then for any \(t' < \tilde t\), we have, by definition that
  \(m^{\up}(t') > 0\), and so by hypothesis \(m^{\low}(t') = 0\). Therefore
  \(d(X) \cdot \B 1_{u(X) < \tilde t} = 0\) a.s., and so, again, \(d(X)\) is a
  threshold policy, contrary to hypothesis.

  Now, with \(t^*\) as above, for notational simplicity, let \(m^{\up} =
  m^{\up}(t^*)\) and \(m^{\low} = m^{\low}(t^*)\) and consider the alternative
  policy
  \begin{equation*}
    d'(x) =
      \begin{cases}
        (1 - m^{\up}) \cdot d(x) & u(x) < t^*, \\
        d(x) & u(x) = t^*, \\
        1 - (1 - m^{\low}) \cdot (1 - d(x))
        & u(x) > t^*.
      \end{cases}
  \end{equation*}
  Then it follows by construction that
  \begin{align*}
     \EE[d'(X)]
      &= (1 - m^{\up}) \cdot m^{\low} + \EE[d(X) \cdot \B 1_{u(X) = t^*}] \\
      &\hspace{0.75cm}+ \Pr(u(X) > t^*) - (1 - m^{\low}) \cdot m^{\up} \\
      &= m^{\low} + \EE[d(X) \cdot \B 1_{u(X) = t^*}] \\
      &\hspace{0.75cm}+ \Pr(u(X) > t^*) - m^{\up} \\
      &= \EE[d(X) \cdot \B 1_{u(X) < t^*}] + \EE[d(X) \cdot \B 1_{u(X) = t^*}]
      \\
      &\hspace{0.75cm}+ \EE[\B 1_{u(X) > t^*}] - \EE[(1 - d(X)) \cdot \B 1_{u(X) >
      t^*}] \\
      &= \EE[d(X)] \\
      &= b,
  \end{align*}
  so \(d'(x)\) is feasible. However,
  \begin{align*}
    d'(x) - d(x)
      &= m^{\low} \cdot (1 - d(x)) \cdot \B 1_{u(x) > t^*} \\
      &\hspace{1cm}- m^{\up} \cdot d(x) \cdot \B 1_{u(x) < t^*},
  \end{align*}
  and so
  \begin{align*}
    \EE[(d'(X) &- d(X)) \cdot u(X)] \\
      &= m^{\low} \cdot \EE[(1 - d(X)) \cdot \B 1_{u(X) > t^*} \cdot u(X)] \\
      &\hspace{1cm} - m^{\up} \cdot \EE[d(X) \cdot \B 1_{u(X) < t^*} \cdot u(X)]
      \\
      &> m^{\low} \cdot t^* \cdot \EE[(1 - d(X)) \cdot \B 1_{u(X) > t^*}] \\
      &\hspace{1cm} - m^{\up} \cdot t^* \cdot \EE[d(X) \cdot \B 1_{u(X) < t^*}]
      \\
      &= t^* \cdot m^{\low} \cdot m^{\up} - t^* \cdot m^{\up} \cdot m^{\low} \\
      &= 0.
  \end{align*}
  Therefore
  \begin{equation*}
    \EE[d(X) \cdot u(X)] < \EE[d'(X) \cdot u(X)].
  \end{equation*}

  It remains to show that \(u'(d') > u'(d)\) for arbitrary \(u' \in \C U\). Let
  \begin{equation*}
    t' = \inf \{u'(x) : d'(x) > d(x)\}.
  \end{equation*}
  Note that by construction for any \(x, x' \in \C X\), if \(d'(x) > d(x)\) and
  \(d'(x') < d(x')\), then \(u(x) > t^* > u(x')\). It follows by consistency
  modulo \(\alpha\) that \(u'(x) \geq t' \geq u'(x')\), and, moreover, that at
  least one of the inequalities is strict. Without loss of generality, assume
  \(u'(x) > t' \geq u'(x')\). Then, we have that \(u(x) > t^*\) if and only if
  \(u'(x) > t'\). Therefore, it follows that
  \begin{equation*}
    \EE[(d'(X) - d(X)) \cdot \B 1_{u'(X) > t'}] = m^{\up} > 0.
  \end{equation*}
  Since \(\EE[d'(X) - d(X)] = 0\), we see that
  \begin{align*}
    \EE[(d'(X) &- d(X)) \cdot u'(X)] \\
      &= \EE[(d'(X) - d(X)) \cdot \B 1_{u'(X) > t'} \cdot u'(X)] \\
      &\hspace{0.75cm} +\EE[(d'(X) - d(X)) \cdot \B 1_{u'(X) \leq t'} \cdot u'(X)]
      \\
      &> t' \cdot \EE[(d'(X) - d(X)) \cdot \B 1_{u'(X) > t'}] \\
      &\hspace{0.75cm}+ t' \cdot \EE[(d'(X) - d(X)) \cdot \B 1_{u'(X) \leq t'}] \\
      &= t' \cdot \EE[d'(X) - d(X)] \\
      &= 0,
  \end{align*}
  where in the inequality we have used the fact that if \(d'(x) > d(x)\),
  \(u'(x) > t' \), and if \(d'(x) < d(x)\), \(u'(x) \leq t'\). Therefore
  \begin{equation*}
    \EE[d(X) \cdot u'(X)] < \EE[d'(X) \cdot u'(X)],
  \end{equation*}
  i.e., \(d'(x)\) strongly Pareto dominates \(d(x)\).

  Now, we prove the second claim, namely, that a multiple threshold policy
  \(\tau(x)\) that cannot be represented with non-negative thresholds is
  strongly Pareto dominated. For, if \(\tau(x)\) is such a policy, then, by
  Lemma~\ref{lem:canonical}, for any \(u \in \C U\), \(\EE[\tau(X) \cdot \B
  1_{u(X) < 0}] > 0\). It follows immediately that \(\tau'(x) = \tau(x) \cdot \B
  1_{u(x) > 0}\) satisfies \(u(\tau') > u(\tau)\). By consistency modulo
  \(\alpha\), the definition of \(\tau'(x)\) does not depend on our choice of
  \(u\), and so \(u(\tau') > u(\tau)\) for every \(u \in \C U\), i.e.,
  \(\tau'(x)\) strongly Pareto dominates \(\tau(x)\).
\end{proof}

The following results, which draw on Lemma~\ref{lem:threshold}, are useful in
the proof of Theorem~\ref{thm:dist}.

\begin{defn}
  We say that a decision policy \(d(x)\) is \emph{budget-exhausting} if
  \begin{equation*}
    \begin{split}
      \min(b, \Pr(u(X) > 0)) & \leq \EE[d(X)]\\
      &\leq \min(b, \Pr(u(X) \geq 0)).
    \end{split}
  \end{equation*}
\end{defn}

\begin{rmk}
  We note that if \(\C U\) is consistent modulo \(\alpha\), then whether or not
  a decision policy \(d(x)\) is budget-exhausting does not depend on the choice
  of \(u \in \C U\). Further, if \(\Pr(u(X) = 0) = 0\)---e.g., if the
  distribution of \(X\) is \(\C U\)-fine---then the decision policy is
  budget-exhausting if and only if \(\EE[d(X)] = \min(b, \Pr(u(X) > 0))\).
\end{rmk}

\begin{cor}
\label{cor:exhaust}
  Let \(\C U\) be a collection of utilities consistent modulo \(\alpha\). If
  \(\tau(x)\) is a feasible policy that is not a budget-exhausting multiple
  threshold policy with non-negative thresholds, then \(\tau(x)\) is strongly
  Pareto dominated.
\end{cor}

\begin{proof}
  Suppose \(\tau(x)\) is not strongly Pareto dominated. By
  Lemma~\ref{lem:threshold}, it is a multiple threshold policy with non-negative
  thresholds.

  Now, suppose toward a contradiction that \(\tau(x)\) is not budget-exhausting.
  Then, either \(\EE[\tau(X)] > \min(b, \Pr(u(X) \geq 0))\) or \(\EE[\tau(X)] <
  \min(b, \Pr(u(X) > 0))\).

  In the first case, since \(\tau(x)\) is feasible, it follows that
  \(\EE[\tau(X)] > \Pr(u(X) \geq 0)\). It follows that \(\tau(x) \cdot \B
  1_{u(x) < 0}\) is not almost surely zero. Therefore
  \begin{equation*}
    \EE[\tau(X)] < \EE[\tau(X) \cdot \B 1_{u(X) > 0}],
  \end{equation*}
  and, by consistency modulo \(\alpha\), this holds for any \(u \in \C U\).
  Therefore \(\tau(x)\) is strongly Pareto dominated, contrary to hypothesis.

  In the second case, consider
  \begin{equation*}
    d(x) = \theta \cdot \B 1_{u(x) > 0} + (1 - \theta) \cdot \tau(x).
  \end{equation*}
  Since \(\EE[\tau(X)] < \min(b, \Pr(u(X) > 0))\) and
  \begin{equation*}
    \EE[d(X)] = \theta \cdot \Pr(u(X) > 0) + (1 - \theta) \cdot \EE[\tau(X)],
  \end{equation*}
  there exists some \(\theta > 0\) such that \(d(x)\) is feasible.

  For that \(\theta\), a similar calculation shows immediately that \(u(d) >
  u(\tau)\), and, by consistency modulo \(\alpha\), \(u'(d) > u'(\tau)\) for all
  \(u' \in \C U\). Therefore, again, \(d(x)\) strongly Pareto dominates
  \(\tau(x)\), contrary to hypothesis.
\end{proof}

\begin{lem}
\label{lem:quantile}
  Given a utility \(u\), there exists a mapping \(T\) from \([0, 1]^{\C A}\) to
  \([-\infty, \infty]^{\C A}\) taking sets of quantiles \(\{q_a\}_{a \in \C A}\)
  to thresholds \(\{t_a\}_{a \in \C A}\) such that:
  \begin{enumerate}
    \item \(T\) is monotonically non-increasing in each coordinate;
    \item For each set of quantiles, there is a multiple threshold policy \(\tau
      : \C X \to [0, 1]\) with thresholds \(T(\{q_a\})\) with respect to \(u\)
      such that \(\EE[\tau(X) \mid A = a] = q_a\).
  \end{enumerate}
\end{lem}

\begin{proof}
  Simply choose
  \begin{equation}
  \label{eq:monotonic}
    t_a = \inf \{s \in \B R : \Pr(u(X) > s) < q_a \}.
  \end{equation}
  Then define
  \begin{equation*}
    p_a =
      \begin{cases}
        \frac {q_a - \Pr(u(X) > t_a \mid A = a)} {\Pr(u(X) = t_a \mid A = a)}
          & \hspace{-4.13467pt} \Pr(u(X) = t_a, A = a) > 0 \\
        0
          & \hspace{-4.13467pt} \Pr(u(X) = t_a, A = a) = 0.
      \end{cases}
  \end{equation*}
  Note that \(\Pr(u(X) \geq t_a \mid A = a) \geq q_a\), since, by definition,
  \(\Pr(u(X) > t_a - \epsilon \mid A = a) \geq q_a\) for all \(\epsilon > 0\).
  Therefore,
  \begin{equation*}
    \Pr(u(X) > t_a \mid A = a) + \Pr(u(X) = t_a \mid A = a) \geq q_a,
  \end{equation*}
  and so \(p_a \leq 1\). Further, since \(\Pr(u(X) > t_a \mid A = a) \leq q_a\),
  we have that \(p_a \geq 0\).

  Finally, let
  \begin{equation*}
    d(x) =
      \begin{cases}
        1 & u(x) > t_{\alpha(x)}, \\
        p_a & u(x) = t_{\alpha(x)}, \\
        0 & u(x) < t_{\alpha(x)},
      \end{cases}
  \end{equation*}
  and it follows immediately that \(\EE[d(X) \mid A = a] = q_a\). That \(t_a\)
  is a monotonically non-increasing function of \(q_a\) follows immediately from
  Eq.~\eqref{eq:monotonic}.
\end{proof}

We can further refine Cor.~\ref{cor:exhaust} and Lemma~\ref{lem:quantile} as
follows:

\begin{lem}
\label{lem:maximize}
  Let \(u\) be a utility. Then a feasible policy is utility maximizing if and
  only if it is a budget-exhausting threshold policy. Moreover, there exists at
  least one utility maximizing policy.
\end{lem}

\begin{proof}
  Let \(\bar \alpha\) be a constant map, i.e., \(\bar \alpha : \C X \to \bar {\C
  A}\), where \(|\bar {\C A}| = 1\). Then \(\C U = \{u\}\) is consistent modulo
  \(\bar \alpha\), and so by Cor.~\ref{cor:exhaust}, any Pareto efficient policy
  is a budget exhausting multiple threshold policy relative to \(\C U\). Since
  \(\C U\) contains a single element, a policy is Pareto efficient if and only
  if it is utility maximizing. Since \(\bar \alpha\) is constant, a policy is a
  multiple threshold policy relative to \(\bar \alpha\) if and only if it is a
  threshold policy. Therefore, a policy is utility maximizing if and only if it
  is a budget exhausting threshold policy. By Lemma~\ref{lem:quantile}, such a
  policy exists, and so the maximum is attained.
\end{proof}

\section{Prevalence and the Proof of Theorem~\ref{thm:dist}}
\label{app:dist}

The notion of a probabilistically ``small'' set---such as the event in which an
idealized dart hits the exact center of a target---is, in finite-dimensional
real vector spaces, typically encoded by the idea of a Lebesgue null set.

Here we prove that the set of distributions such that there exists a policy
satisfying either counterfactual equalized odds, conditional principal fairness,
or counterfactual fairness that is not strongly Pareto dominated is ``small'' in
an analogous sense. The proof turns on the following intuition. Each of the
fairness definitions imposes a number of constraints. By
Lemma~\ref{lem:threshold}, any policy that is not strongly Pareto dominated is a
multiple threshold policy. By adjusting the group-specific thresholds of such a
policy, one can potentially satisfy one constraint per group. If there are more
constraints than groups, then one has no additional degrees of freedom that can
be used to ensure that the remaining constraints are satisfied. If, by chance,
those constraints \emph{are} satisfied with the same threshold policy, they are
not satisfied robustly---even a minor distribution shift, such as increasing the
amount of mass above the threshold by any amount on the relevant subpopulation,
will break them. Therefore, over a ``typical'' distribution, at most \(|\C A|\)
of the constraints can simultaneously be satisfied by a Pareto efficient policy,
meaning that typically no Pareto efficient policy fully satisfies all of the
conditions of the fairness definitions.

Formalizing this intuition, however, requires considerable care. In
Section~\ref{sec:shyness}, we give a brief introduction to a popular
generalization of null sets to infinite-dimensional vector spaces, drawing
heavily on a review article by \citet{ott2005prevalence}. In
Section~\ref{sec:roadmap} we provide a  roadmap of the proof itself. In
Section~\ref{sec:shyness-prelims}, we establish the main hypotheses necessary to
apply the notion of prevalence to a convex set---in our case, the set of \(\C
U\)-fine distributions. In Section~\ref{sec:shyness-prelims2}, we establish a
number of technical lemmata used in the proof of Theorem~\ref{thm:dist}, and
provide a proof of the theorem itself in Section~\ref{sec:shyness-proof}. In
Section~\ref{sec:counterexample}, we show why the hypothesis of \(\C
U\)-fineness is important and how conspiracies between atoms in the distribution
of \(u(X)\) can lead to ``robust'' counterexamples.

\subsection{Shyness and Prevalence}
\label{sec:shyness}

Lebesgue measure \(\lambda_n\) on \(\B R^n\) has a number of desirable
properties:
  \begin{itemize}
    \item \textbf{Local finiteness:} For any point \(v \in \B R^n\), there
      exists an open set \(U\) containing \(x\) such that \(\lambda_n[U] <
      \infty\);
    \item \textbf{Strict positivity:} For any open set \(U\), if \(\lambda_n[U]
      = 0\), then \(U = \emptyset\);
    \item \textbf{Translation invariance:} For any \(v \in \B R^n\) and
      measurable set \(E\), \(\lambda_n[E + v] = \lambda_n[E]\).
  \end{itemize}
No measure on an infinite-dimensional, separable Banach space, such as \(L^1(\B
R)\), can satisfy these three properties \cite{ott2005prevalence}. However,
while there is no generalization of Lebesgue measure to infinite dimensions,
there is a generalization of Lebesgue null sets---called \emph{shy} sets---to
the infinite-dimensional context that preserves many of their desirable
properties.

\begin{defn}[\citet{hunt1992prevalence}]
  Let \(V\) be a completely metrizable topological vector space. We say that a
  Borel set \(E \subseteq V\) is \emph{shy} if there exists a Borel measure
  \(\mu\) on \(V\) such that:
  \begin{enumerate}
    \item There exists compact \(C \subseteq V\) such that \(0 < \mu[C] <
      \infty\),
    \item For all \(v \in V\), \(\mu[E + v] = 0\).
  \end{enumerate}
  An arbitrary set \(F \subseteq V\) is shy if there exists a shy Borel set \(E
  \subseteq V\) containing \(F\).

  We say that a set is \emph{prevalent} if its complement is shy.
\end{defn}

Prevalence generalizes the concept of Lebesgue ``full measure'' or ``co-null''
sets (i.e., sets whose complements have null Lebesgue measure) in the following
sense:

\begin{prop}[\citet{hunt1992prevalence}]
\label{prop:shy_axioms}
  Let \(V\) be a completely metrizable topological vector space. Then:
  \begin{itemize}
    \item Any prevalent set is dense in \(V\);
    \item If \(G \subseteq L\) and \(G\) is prevalent, then \(L\) is prevalent;
    \item A countable intersection of prevalent sets is prevalent;
    \item Every translate of a prevalent set is prevalent;
    \item If \(V = \B R^n\), then \(G \subseteq \B R^n\) is prevalent if and
      only if \(\lambda_n[\B R^n \setminus G] = 0\).
  \end{itemize}
\end{prop}

As is conventional for sets of full measure in finite-dimensional spaces, if
some property holds for every \(v \in E\), where \(E\) is prevalent, then we say
that the property holds for \emph{almost every \(v \in V\)} or that it holds
\emph{generically in \(V\)}.

Prevalence can also be generalized from vector spaces to convex subsets of
vector spaces, although additional care must be taken to ensure that a relative
version of Prop.~\ref{prop:shy_axioms} holds.

\begin{defn}[\citet{anderson2001genericity}]
\label{defn:shy_rel}
  Let \(V\) be a topological vector space and let \(C \subseteq V\) be a convex
  subset completely metrizable in the subspace topology induced by \(V\). We say
  that a universally measurable set \(E \subseteq C\) is \emph{shy in \(C\) at
  \(c \in C\)} if for each \(1 \geq \delta > 0\), and each neighborhood \(U\) of
  \(0\) in \(V\), there is a regular Borel measure \(\mu\) with compact support
  such that
  \begin{equation*}
    \supp(\mu) \subseteq \left(\delta(C - c) + c \right) \cap (U + c),
  \end{equation*}
  and \(\mu[E + v] = 0\) for every \(v \in V\).

  We say that \(E\) is \emph{shy in \(C\)} or \emph{shy relative to \(C\)} if
  \(E\) is shy in \(C\) at \(c\) for every \(c \in C\). An arbitrary set \(F
  \subseteq V\) is shy in \(C\) if there exists a universally measurable shy set
  \(E \subseteq C\) containing \(F\).

  A set \(G\) is \emph{prevalent} in \(C\) if \(C \setminus G\) is shy in \(C\).
\end{defn}

\begin{prop}[\citet{anderson2001genericity}]
  If \(E\) is shy at some point \(c \in C\), then \(E\) is shy at every point in
  \(C\) and hence is shy in \(C\).
\end{prop}

Sets that are shy in \(C\) enjoy similar properties to sets that are shy in
\(V\).

\begin{prop}[\citet{anderson2001genericity}]
\label{prop:shy_axioms_rel}
  Let \(V\) be a topological vector space and let \(C \subseteq V\) be a convex
  subset completely metrizable in the subspace topology induced by \(V\). Then:
  \begin{itemize}
    \item Any prevalent set in \(C\) is dense in \(C\);
    \item If \(G \subseteq L\) and \(G\) is prevalent in \(C\), then \(L\) is
      prevalent in \(C\);
    \item A countable intersection of sets prevalent in \(C\) is prevalent in
      \(C\)
    \item If \(G\) is prevalent in \(C\) then \(G + v\) is prevalent in \(C +
      v\) for all \(v \in V\).
    \item If \(V = \B R^n\) and \(C \subseteq V\) is a convex subset with
      non-empty interior, then \(G \subseteq C\) is prevalent in \(C\) if and
      only if \(\lambda_n[C \setminus G] = 0\).
  \end{itemize}
\end{prop}

Sets that are shy in \(C\) can often be identified by inspecting their
intersections with a finite-dimensional subspace \(W\) of \(V\), a strategy we
use to prove Theorem~\ref{thm:dist}.

\begin{defn}[\citet{anderson2001genericity}]
  A universally measurable set \(E \subseteq C\), where \(C\) is convex and
  completely metrizable, is said to be \emph{\(k\)-shy in \(C\)} if there exists
  a \(k\)-dimensional subspace \(W \subseteq V\) such that
  \begin{enumerate}
    \item A translate of the set \(C\) has positive Lebesgue measure in \(W\),
      i.e., \(\lambda_W[C + v_0] > 0\) for some \(v_0 \in V\);
    \item Every translate of the set \(E\) is a Lebesgue null set in \(W\),
      i.e., \(\lambda_W[E + v] = 0\) for all \(v \in V\).
  \end{enumerate}
  Here \(\lambda_W\) denotes \(k\)-dimensional Lebesgue measure supported on
  \(W\).\footnote{%
    Note that Lebesgue measure on \(W\) is only defined up to a choice of basis;
    however, since \(\lambda[T(A)] = |\det(T)| \cdot \lambda[A]\) for any linear
    automorphism \(T\) and Lebesgue measure \(\lambda\), whether a set has null
    measure does not depend on the choice of basis.
  }
  We refer to such a \(W\) as a \emph{\(k\)-dimensional probe} witnessing the
  \(k\)-shyness of \(E\), and to an element \(w \in W\) as a
  \emph{perturbation}.
\end{defn}

The following intuition motivates the use of probes to detect shy sets. By
analogy with Fubini's theorem, one can imagine trying to determine whether a
subset of a finite-dimensional vector space is large or small by looking at its
cross sections parallel to some subspace \(W \subseteq V\). If a set \(E
\subseteq V\) is small in each cross section---i.e., if \(\lambda_W[E + v] = 0\)
for all \(v \in V\)---then \(E\) itself is small in \(V\), i.e., \(E\) has
\(\lambda_V\)-measure zero.

\begin{prop}[\citet{anderson2001genericity}]
\label{prop:k_shy}
  Every \(k\)-shy set in \(C\) is shy in \(C\).
\end{prop}

\subsection{Outline}
\label{sec:roadmap}

To aid the reader in following the application of the theory in
Section~\ref{sec:shyness} to the proof of Theorem~\ref{thm:dist}, we provide the
following outline of the argument.

In \textbf{Section~\ref{sec:shyness-prelims}} we establish the context to which
we apply the notion of relative shyness. In particular, we introduce the vector
space \(\B K\) consisting of the \emph{totally bounded Borel measures} on the
state space \(\C K\)---where \(\C K\) is \(\C X \times \C Y\), \(\C X \times \C
Y \times \C Y\), or \(\C A \times \C X^{\C A}\), depending on which notion of
fairness is under consideration. We further isolate the subspace \(\bb K
\subseteq \B K\) of \(\C U\)-fine totally bounded Borel measures. Within this
space, we are interested in the convex set \(\bb Q \subseteq \bb K\), the set of
\emph{\(\C U\)-fine joint probability distributions} of, respectively, \(X\) and
\(Y(1)\); \(X\), \(Y(0)\), \(Y(1)\); or \(A\) and the \(X_{\Pi, A, a}\). Within
\(\bb Q\), we identify \(\bb E \subseteq \bb Q\), the set of \(\C U\)-fine
distributions on \(\C K\) \emph{over which there exists a policy satisfying the
relevant fairness definition that is not strongly Pareto dominated}. The claim
of Theorem~\ref{thm:dist} is that \(\bb E\) is shy relative to \(\bb Q\).

To ensure that relative shyness generalizes Lebesgue null measure in the
expected way---i.e., that Prop.~\ref{prop:shy_axioms_rel}
holds---Definition~\ref{defn:shy_rel} has three technical requirements: (1) that
the ambient vector space \(V\) be a topological vector space; (2) that the
convex set \(C\) be completely metrizable; and (3) that the shy set \(E\) be
universally measurable. In \textbf{Lemma~\ref{lem:banach}}, we observe that \(\B
K\) is a complete topological vector space under the total variation norm, and
so is a Banach space. We extend this in \textbf{Cor.~\ref{cor:banach}}, showing
that \(\bb K\) is also a Banach space. We use this fact in
\textbf{Lemma~\ref{lem:convex}} to show that \(\bb Q\) is a completely
metrizable subset of \(\bb K\), as well as convex. Lastly, in
\textbf{Lemma~\ref{lem:e_closed}}, we show that the set \(\bb E\) is closed, and
therefore universally measurable.

In \textbf{Section~\ref{sec:shyness-prelims2}}, we develop the machinery needed
to construct a probe \(\bb W\) for the proof of Theorem~\ref{thm:dist} and prove
several lemmata simplifying the eventual proof of the theorem. To build the
probe, it is necessary to construct measures \(\mu_{\max,a}\) with maximal
support on the utility scale. This ensures that if any two threshold policies
produce different decisions on \emph{any} \(\mu \in \bb K\), they will produce
different decisions on typical perturbations. The construction of the
\(\mu_{\max,a}\), is carried out in \textbf{Lemma~\ref{lem:mtu}} and
\textbf{Cor.~\ref{cor:maximal}}. Next, we introduce the basic style of argument
used to show that a subset of \(\bb Q\) is shy in \textbf{Lemma~\ref{lem:probe}}
and \textbf{Lemma~\ref{lem:condition}}, in particular, by showing that the set
of \(\mu \in \bb Q\) that give positive probability to an event \(E\) is either
prevalent or empty. We use then use a technical lemma,
\textbf{Lemma~\ref{lem:uncountable_sum}}, to show, in effect, that a generic
element of \(\bb Q\) has support on the utility scale wherever a given fixed
distribution \(\mu \in \bb Q\) does. \textbf{In Defn.~\ref{defn:overlap}}, we
introduce the concept of overlapping and splitting utilities, and show in
\textbf{Lemma~\ref{lem:overlap}} that this property is generic in \(\bb Q\)
unless there exists a \(\omega\)-stratum that contains no positive-utility
observables \(x\). Lastly, in \textbf{Lemma~\ref{lem:simple}}, we provide a mild
simplification of the characterization of finitely shy sets that makes the the
proof of Theorem~\ref{thm:dist} more straightforward.

Finally, in \textbf{Section~\ref{sec:shyness-proof}}, we give the proof of
Theorem~\ref{thm:dist}. We divide the proof into three parts. In the first part,
we restrict our attention to the case of counterfactual equalized odds, and show
in detail how to combine the lemmata of the previous section to construct the
(at most) \(2 \cdot |\C A|\)-dimensional probe \(\bb W\). In the second part we
consider two distinct cases. The argument in both cases is conceptually
parallel. First, we argue that the balance conditions of counterfactual
equalized odds encoded by Eq.~\eqref{eq:counterfactual_equalized_odds} must be
broken by a typical perturbation in \(\bb W\). In particular, we argue that for
a given base distribution \(\mu\), there can be at most one budget-exhausting
multiple threshold policy that can---although need not necessarily---satisfy
counterfactual equalized odds. We show that the form of this policy cannot be
altered by an appropriate perturbation in \(\bb W\), but that the conditional
probability of a positive decision will, in general, be altered in such a way
that Eq.~\eqref{eq:counterfactual_equalized_odds} can only hold for a
\(\lambda_{\bb W}\)-null set of perturbations. In the final section, we lay out
modiciations that can be made to the proof given for counterfactual equalized
odds in the first two parts that adapt the argument to the cases of conditional
principal fairness and path-specific fairness. In particular, we show how to
construct the probe \(\bb W\) in such a way that the additional conditioning on
the reduced covariates \(W = \omega(X)\) in
Eqs.~\eqref{eq:conditional_principal_fairness}~and~\eqref{eq:path_specific_fairness}
does not affect the argument.

\subsection{Convexity, Complete Metrizability, and Universal Measurability}
\label{sec:shyness-prelims}

In this section, we establish the background requirements of
Prop.~\ref{prop:k_shy} for the setting of Theorem~\ref{thm:dist}. In particular,
we exhibit the \(\C U\)-fine distributions as a convex subset of a topological
vector space, the set of totally bounded \(\C U\)-fine Borel measures. We show
that the \(\C U\)-fine probability distributions form a completely metrizable
subset in the topology it inherits from the space of totally bounded measures.
Lastly, we show that the set of regular distributions under which there exists a
Pareto efficient policy satisfying one of the three fairness criteria is closed,
and therefore universally measurable.

\subsubsection{Background and notation}

We begin by establishing some notational conventions. We let \(\C K\) denote the
underlying state space over which the distributions in Theorem~\ref{thm:dist}
range. Specifically, \(\C K = \C X \times \C Y\) in the case of counterfactual
equalized odds; \(\C K = \C X \times \C Y \times \C Y\) in the case of
conditional principal fairness; and \(\C K = \C A \times \C X^{\C A}\) in the
case of path-specific fairness. We note that since \(\C X \subseteq \B R^k\) for
some \(k\) and \(Y \subseteq \B R\), \(\C K\) may equivalently be considered a
subset of \(\B R^n\) for some \(n \in \B N\), with the subspace topology (and
Borel sets) inherited from \(\B R^n\).\footnote{%
  In the case of path-specific fairness, we can equivalently think of \(\C A\)
  as a set of integers indexing the groups.
}

We recall the definition of totally bounded measures.

\begin{defn}
  Let \(\C M\) be a \(\sigma\)-algebra on \(V\), and let \(\mu\) be a countably
  additive \((V, \C M)\)-measure. Then, we define
  \begin{equation}
    |\mu|[E] = \sup \sum_{i = 1}^\infty |\mu[E_i]|
  \end{equation}
  where the supremum is taken over all countable partitions \(\{E_i\}_{i \in \B
  N}\), i.e., collections such that \(\bigcup_{i=1}^\infty E_i = E\) and \(E_i
  \cap E_j = \emptyset\) for \(j \neq i\). We call \(|\mu|\) the \emph{total
  variation of \(\mu\)}, and the \emph{total variation norm of \(\mu\)} is
  \(|\mu|[V]\).

  We say that \(\mu\) is \emph{totally bounded} if its total variation norm is
  finite, i.e., \(|\mu|[V] < \infty\).
\end{defn}

\begin{lem}
\label{lem:tot_var}
  If \(\mu\) is totally bounded, then \(|\mu|\) is a finite positive measure on
  \((V, \C M)\), and \(|\mu[E]| \leq |\mu|[E]\) for all \(E \in \C M\).
\end{lem}

See Theorem~6.2 in \citet{rudin1987real} for proof.

We let \(\B K\) denote the set of totally bounded Borel measures on \(\C K\). We
note that, in the case of path specific fairness, which involves the joint
distributions of counterfactuals, \(X\) is not defined directly. Rather, the
joint distribution of the counterfactuals \(X_{\Pi, A, a'}\) and \(A\) defines
the distribution of \(X\) through consistency, i.e., what would have happened to
someone if their group membership were changed to \(a' \in \C A\) is what
actually happens to them if their group membership \emph{is} \(a'\). More
formally, \(\Pr(X \in E \mid A = a') = \Pr(X_{\Pi, A, a'} \in E \mid A = a')\)
for all Borel sets \(E \subseteq \C X\). (See \S~3.6.3 in
\citet{pearl2009causality}.)

For any \(\mu \in \B K\), we adopt the following notational conventions. If we
say that a property holds \(\mu\)-a.s., then the subset of \(\C K\) on which the
property fails has \(|\mu|\)-measure zero. If \(E \subseteq \C K\) is a
measurable set, then we denote by \(\mu \rest_E\) the restriction of \(\mu\) to
\(E\), i.e., the measure defined by the mapping \(E' \mapsto \mu[E \cap E']\).
We let \(\EE_\mu[f] = \int_{\C K} f \, \dx \mu\), and for measurable sets \(E\),
\(\Pr_\mu(E) = \mu[E]\).\footnote{%
  To state and prove our results in a notationally uniform way, we occasionally
  write \(\Pr_\mu(E)\) even when \(\mu\) ranges over measures that may not be
  probability measures.
}
The fairness criteria we consider involve conditional independence relations. To
make sense of conditional independence relations more generally, for Borel
measurable \(f\) we define \(\EE_\mu[f \mid \C F]\) to be the Radon-Nikodym
derivative of the measure \(E \mapsto \EE_\mu[f \cdot \B 1_E]\) with respect to
the measure \(\mu\) restricted to the sub--\(\sigma\)-algebra of Borel sets \(\C
F\). (See \S~34 in \citet{billingsley2008probability}.) Similarly, we define
\(\EE_\mu[f \mid g]\) to be \(\EE_\mu[f \mid \sigma(g)]\), where \(\sigma(g)\)
denotes the sub--\(\sigma\)-algebra of the Borel sets generated by \(g\). In
cases where the condition can occur with non-zero probability, we can instead
make use of the elementary definition of discrete conditional probability.

\begin{lem}
\label{lem:cond_prob}
  Let \(g\) be a Borel function on \(\C K\), and suppose \(\Pr_\mu(g = c) \neq
  0\) for some constant \(c \in \B R\). Then, we have that \(\mu\)-a.s., for any
  Borel function \(f\),
  \begin{equation*}
    \EE_\mu[f \mid g] \cdot \B 1_{g = c} = \frac{\EE_\mu[f \cdot \B 1_{g = c}]}
    {\Pr_\mu(g = c)} \cdot \B 1_{g = c}.
  \end{equation*}
\end{lem}

See \citet{rao2005conditional} for proof.

With these notational conventions in place, we turn to establishing the
background conditions of Prop.~\ref{prop:k_shy}.

\begin{lem}
\label{lem:banach}
  The set of totally bounded measures on a measure space \((V, \C M)\) form a
  complete topological vector space under the total variation norm, and hence a
  Banach space.
\end{lem}

See, e.g., \citet{steele2019space} for proof. It follows from this that \(\B K\)
is a Banach space.

\begin{rmk}
\label{rmk:Borel}
  Since \(\B K\) is a Banach space, it possesses a topology, and consequently a
  collection of Borel subsets. These Borel sets are to be distinguished from the
  Borel subsets of the underlying state space \(\C K\), which the elements of
  \(\B K\) measure. The requirement that the subset \(E\) of the convex set
  \(C\) be universally measurable in Proposition~\ref{prop:k_shy} is in
  reference to the \emph{Borel subsets of \(\B K\)}; the requirement that \(\mu
  \in \B K\) be a Borel measure is in reference to the \emph{Borel subsets of
  \(\C K\)}.
\end{rmk}

Recall the definition of absolute continuity.

\begin{defn}
  Let \(\mu\) and \(\nu\) be measures on a measure space \((V, \C M)\). We say
  that a measure \(\nu\) is \emph{absolutely continuous with respect to
  \(\mu\)}---also written \(\nu \Lt \mu\)---if, whenever \(\mu[E] = 0\),
  \(\nu[E] = 0\).
\end{defn}

Absolute continuity is a closed property in the topology induced by the total
variation norm.

\begin{lem}
\label{lem:abscont}
  Consider the space of totally bounded measures on a measure space \((V, \C
  M)\) and fix \(\mu\). The set of \(\nu\) such that \(\nu \Lt \mu\) is closed.
\end{lem}

\begin{proof}
  Let \(\{\nu_i\}_{i \in \B N}\) be a convergent sequence of measures absolutely
  continuous with respect to \(\mu\). Let the limit of the \(\nu_i\) be \(\nu\).
  We seek to show that \(\nu \Lt \mu\). Let \(E \in \C M\) be an arbitrary set
  such that \(\mu[E] = 0\). Then, we have that
  \begin{align*}
    \nu[E]
      &= \lim_{n \to \infty} \nu_i[E] \\
      &= \lim_{n \to \infty} 0 \\
      &= 0,
  \end{align*}
  since \(\nu_i \Lt \mu\) for all \(i\). Since \(E\) was arbitrary, the result
  follows.
\end{proof}

Recall the definition of a pushforward measure.

\begin{defn}
  Let \(f : (V, \C M) \to (V', \C M')\) be a measurable function. Let \(\mu\) be
  a measure on \(V\). We define the \emph{pushforward measure} \(\mu \circ
  f^{-1}\) on \(V'\) by the map \(E' \mapsto \mu[f^{-1}(E')]\) for  \(E' \in \C
  M'\).
\end{defn}

Within \(\B K\), in the case of counterfactual equalized odds and conditional
principal fairness, we define the subspace \(\bb K\) to be the set of totally
bounded measures \(\mu\) on \(\C K\) such that the pushforward measure \(\mu
\circ u^{-1}\) is absolutely continuous with respect to the Lebesgue measure
\(\lambda\) on \(\B R\) for all \(u \in \C U\). By the Radon-Nikodym theorem,
these pushforward measures arise from densities, i.e., for any \(\mu \in \bb
K\), there exists a unique \(f_\mu \in L^1(\B R)\) such that for any measurable
subset \(E\) of \(\B R\), we have
\begin{equation*}
  \mu \circ u^{-1}[E] = \int_{E} f_\mu \, \dx \lambda.
\end{equation*}
In the case of path-specific fairness, we require the joint distributions of the
counterfactual utilities to have a joint density. That is, we define the
subspace \(\bb K\) to be the set of totally bounded measures \(\mu\) on \(\C K\)
such that the pushforward measure \(\mu \circ (u^{\C A})^{-1}\) is absolutely
continuous with respect to Lebesgue measure on \(\B R^{\C A}\) for all \(u \in
\C U\). Here, we recall that
\begin{equation*}
  u^{\C A} : (a, (x_{a'})_{a' \in \C A}) \mapsto (u(x_{a'}))_{a' \in \C A}.
\end{equation*}
As before, there exists a corresponding density \(f_\mu \in L^1(\B R^{\C A})\).

We therefore see that \(\bb K\) extends in a natural way the notion of a \(\C
U\)- or \(\C U^{\C A}\)-fine distribution, and so, by a slight abuse of
notation, refer to \(\bb K\) as the set of \emph{\(\C U\)-fine measures on \(\C
K\)}.

Indeed, since \(\Pr_\mu(u(X) \in E, A = a) \leq \Pr_\mu(u(X) \in E)\), it also
follows that, for \(a \in \C A\) such that \(\Pr_\mu(A = a) > 0\), the
conditional distributions of \(u(X) \mid A = a\) are also absolutely continuous
with respect to Lebesgue measure, and so also have densities. For notational
convenience, we set \(f_{\mu, a}\) to be the function satisfying
\begin{equation*}
  \Pr_\mu(u(X) \in E, A = a) = \int_E f_{\mu, a} \, \dx \lambda,
\end{equation*}
so that \(f_{\mu} = \sum_{a \in \C A} f_{\mu, a}\).

Since absolute continuity is a closed condition, it follows that \(\bb K\) is a
closed subspace of \(\B K\). This leads to the following useful corollary of
Lemma~\ref{lem:abscont}.

\begin{cor}
\label{cor:banach}
  The collection of \(\C U\)-fine measures on \(\C K\) is a Banach space.
\end{cor}

\begin{proof}
  It is straightforward to see that \(\bb K\) is a subspace of \(\B K\). Since
  \(\bb K\) is a closed subset of \(\B K\) by Lemma~\ref{lem:abscont}, it is
  complete, and therefore a Banach space.
\end{proof}

We note the following useful fact about elements of \(\bb K\).

\begin{lem}
\label{lem:mapping}
  Consider the mapping \(\mu \mapsto f_\mu\) from \(\bb K\) to \(L^1(\B R)\)
  given by associating a measure \(\mu\) with the Radon-Nikodym derivative of
  the pushforward measure \(\mu \circ u^{-1}\). This mapping is continuous.
  Likewise, the mapping \(\mu \mapsto f_{\mu,a}\) is continuous for all \(a \in
  \C A\), and, in the case of path-specific fairness, the mapping of \(\mu\) to
  the Radon-Nikodym derivative of \(\mu \circ (u^{\C A})^{-1}\) is continuous.
\end{lem}

\begin{proof}
  We show only the first case. The others follow by virtually identical
  arguments.

  Let \(\epsilon > 0\) be arbitrary. Choose \(\mu \in \bb K\), and suppose that
  \(|\mu - \mu'|[\C K] < \epsilon\). Then, let
  \begin{align*}
    E^{\up}
      &= \{x \in \B R : f_\mu(x) > f_{\mu'}(x)\} \\
    E^{\low}
      &= \{x \in \B R : f_\mu(x) < f_{\mu'}(x)\}.
  \end{align*}
  Then \(E^{\up}\) and \(E^{\low}\) are disjoint, so we have that
  \begin{align*}
    \|f_\mu - f_{\mu'}\|_{L^1(\B R)}
      &= \left| \int_{E^{\up}} f_{\mu} - f_{\mu'} \, \dx \lambda \right| \\
      &\hspace{1cm}+ \left| \int_{E^{\low}} f_{\mu} - f_{\mu'} \, \dx \lambda
      \right| \\
      &= |(\mu - \mu')[u^{-1}(E^{\up})]| \\
      &\hspace{1cm} + |(\mu - \mu')[u^{-1}(E^{\low})]| \\
      &< \epsilon,
  \end{align*}
  where the second equality follows by the definition of pushforward measures
  and the inequality follows from Lemma~\ref{lem:tot_var}. Since \(\epsilon\)
  was arbitrary, the claim follows.
\end{proof}

Finally, we define \(\bb Q\). We let \(\bb Q\) be the subset of \(\bb K\)
consisting of all \(\C U\)-fine probability measures, i.e., measures \(\mu \in
\B K\) such that:
\begin{enumerate}
  \item The measure \(\mu\) is \(\C U\)-fine;
  \item For all Borel sets \(E \subseteq \C K\), \(\mu[E] \geq 0\);
  \item The measure of the whole space is unity, i.e., \(\mu[\C K] = 1\).
\end{enumerate}

We conclude the background and notation by observing that threshold policies are
defined wholly by their thresholds for distributions in \(\bb K\) and \(\bb Q\).
Importantly, this observation does not hold when there are atoms on the utility
scale---which measures in \(\bb K\) lack---which can in turn lead to
counterexamples to Theorem~\ref{thm:dist}; see
Appendix~\ref{sec:counterexample}.

\begin{lem}
\label{lem:simplethresh}
  Let \(\tau_0(x)\) and \(\tau_1(x)\) be two multiple threshold policies. If
  \(\tau_0(x)\) and \(\tau_1(x)\) have the same thresholds, then for any \(\mu
  \in \bb K\), \(\tau_0(X) = \tau_1(X)\) \(\mu\)-a.s. Similarly, for \(\mu \in
  \bb Q\), if
  \begin{equation*}
    \EE_\mu[\tau_0(X) \mid A = a] = \EE_\mu[\tau_1(X) \mid A = a]
  \end{equation*}
  for all \(a \in \C A\) such that \(\Pr_\mu(A = a) > 0\), then \(\tau_0(X) =
  \tau_1(X)\) \(\mu\)-a.s.

  Moreover, for \(\mu \in \bb K\) in the case of path-specific fairness, if
  \(\tau_0(x)\) and \(\tau_1(x)\) have the same thresholds, then
  \(\tau_0(X_{\Pi, A, a}) = \tau_1(X_{\Pi, A, a})\) \(\mu\)-a.s.\ for any  \(a
  \in \C A\). Similarly, for \(\mu \in \bb Q\) in the case of path-specific
  fairness, if
  \begin{equation*}
    \EE_\mu[\tau_0(X_{\Pi, A, a})] = \EE_\mu[\tau_1(X_{\Pi, A, a})]
  \end{equation*}
  then \(\tau_0(X_{\Pi, A, a}) = \tau_1(X_{\Pi, A, a})\) \(\mu\)-a.s.\ as well.
\end{lem}

\begin{proof}
  First, we show that threshold policies with the same thresholds are equal,
  then we show that threshold policies that distribute positive decisions across
  groups in the same way are equal.

  Let \(\{t_a\}_{a \in \C A}\) denote the shared set of thresholds. It follows
  that if \(\tau_0(x) \neq \tau_1(x)\), then \(u(x) = t_{\alpha(x)}\). Now,
  \begin{equation*}
    \Pr(u(X) = t_a, A = a) = \int_{t_a}^{t_a} f_{\mu, a} \, \dx \lambda = 0,
  \end{equation*}
  so \(\Pr_\mu(\tau_0(X) \neq \tau_1(X)) = 0\).
  Next, suppose
  \begin{equation*}
    \EE_\mu[\tau_0(X) \mid A = a] = \EE_\mu[\tau_1(X) \mid A = a].
  \end{equation*}
  If the thresholds of the two policies agree for all \(a \in \C A\) such that
  \(\Pr_\mu(A = a) > 0\), then we are done by the previous paragraph. Therefore,
  suppose \(t_a^0 \neq t_a^1\) for some suitable \(a \in \C A\), where \(t_a^i\)
  represents the threshold for group \(a \in \C A\) under the policy
  \(\tau_i(x)\). Without loss of generality, suppose \(t_a^0 < t_a^1\). Then, it
  follows that
  \begin{align*}
    \int_{t_a^{0}}^{t_a^1} f_{\mu,a} \, \dx \lambda
      &= \EE_\mu[\tau_0(X) \mid A = a] - \EE_\mu[\tau_1(X) \mid A = a] \\
      &= 0.
  \end{align*}
  Since \(\mu \in \bb Q\), \(\mu = |\mu|\), whence
  \begin{equation*}
    \Pr_{|\mu|}(t_0^a \leq u(X) \leq t_a^1 \mid A = a) = 0.
  \end{equation*}
  Since this is true for all \(a \in \C A\) such that \(\Pr_\mu(A = a) > 0\),
  \(\tau_0(X) = \tau_1(X)\) \(\mu\)-a.s.

  The proof in the case of path-specific fairness is almost identical.
\end{proof}

\subsubsection{Convexity, complete metrizability, and universal measurability}

The set of regular \(\C U\)-fine probability measures \(\bb Q\) is the set to
which we wish to apply Prop.~\ref{prop:k_shy}. To do so, we must show that \(\bb
Q\) is a convex and completely metrizable subset of \(\bb K\).

\begin{lem}
\label{lem:convex}
  The set of regular probability measures \(\bb Q\) is convex and completely
  metrizable.
\end{lem}

\begin{proof}
  The proof proceeds in two pieces. First, we show that the \(\C U\)-fine
  probability distributions are convex, as can be verified by direct
  calculation. Then, we show that \(\bb Q\) is closed and therefore complete in
  the original metric of \(\bb K\).

  We begin by verifying convexity. Let \(\mu, \mu' \in \bb Q\) and let \(E
  \subseteq \C K\) be an arbitrary Borel subset of \(\C K\). Then, choose
  \(\theta \in [0, 1]\), and note that
  \begin{align*}
    (\theta \cdot \mu + [1 - \theta] \cdot \mu') [E]
      &= \theta \cdot \mu[E] + [1 - \theta] \cdot \mu'[E] \\
      &\geq \theta \cdot 0 + [1 - \theta] \cdot 0 \\
      &= 0,
  \end{align*}
  and, likewise, that
  \begin{align*}
    (\theta \cdot \mu + [1 - \theta] \cdot \mu') [\C K]
      &= \theta \cdot \mu[\C K] + [1 - \theta] \cdot \mu'[\C K] \\
      &= \theta \cdot 1 + [1 - \theta] \cdot 1 \\
      &= 1.
  \end{align*}
  It remains only to show that \(\bb Q\) is completely metrizable. To prove
  this, it suffices to show that it is closed, since closed subsets of complete
  spaces are complete, and \(\bb K\) is a Banach space by Cor.~\ref{cor:banach},
  and therefore complete.

  Suppose \(\{\mu_i\}_{i \in \B N}\) is a convergent sequence of probability
  measures in \(\bb K\) with limit \(\mu\). Then
  \begin{equation*}
    \mu[E] = \lim_{i \to \infty} \mu_i[E] \geq \lim_{i \to \infty} 0 = 0
  \end{equation*}
  and
  \begin{equation*}
    \mu[\C K] = \lim_{i \to \infty} \mu_i[\C K] = \lim_{i \to \infty} 1 = 1.
  \end{equation*}
  Therefore \(\bb Q\) is closed, and therefore complete, and hence is a convex,
  completely metrizable subset of \(\bb K\).
\end{proof}

Next we prove that the set \(\bb E\) of regular \(\C U\)-fine densities over
which there exists a policy satisfying the relevant counterfactual fairness
definition that is not strongly Pareto dominated is universally measurable.

Recall the definition of universal measurability.

\begin{defn}
  Let \(V\) be a complete topological space. Then \(E \subseteq V\) is
  \emph{universally measurable} if \(V\) is measurable by the completion of
  every finite Borel measure on \(V\), i.e., if for every finite Borel measure
  \(\mu\), there exist Borel sets \(E'\) and \(S\) such that \(E\ \triangle\ E'
  \subseteq S\) and \(\mu[S] = 0\).
\end{defn}

We note that if a set is Borel, it is by definition universally measurable.
Moreover, if a set is open or closed, it is by definition Borel.

To show that \(\bb E\) is closed, we show that any convergent sequence in \(\bb
E\) has a limit in \(\bb E\). The technical complication of the argument stems
from the following fact that satisfying the fairness conditions, e.g.,
Eq.~\eqref{eq:counterfactual_fairness}, involves conditional expectations, about
which very little can be said in the absence of a density, and which are
difficult to compare when taken across distinct measures.

To handle these difficulties, we begin with a technical lemma,
Lemma~\ref{lem:ce_est}, which gives a coarse bound on how different the
conditional expectations of the same variable can be with respect to a
sub--\(\sigma\)-algebra \(\C F\) over two different distributions, \(\mu\) and
\(\mu'\), before applying the results to the proof of Lemma~\ref{lem:e_closed}.

\begin{defn}
\label{defn:standard_version}
  Let \(\mu\) be a measure on a measure space \((V, \C M)\), and let \(f\) be
  \(\mu\)-measurable. Consider the equivalence class of \(\C M\)-measurable
  functions \(C = \{g : g = f \text{ \(\mu\)-a.e.}\}\).\footnote{%
    Some authors define \(L^p(\mu)\) spaces to consist of such equivalence
    classes, rather than the definition we use here.
  } We say that any \(g \in C\) is a \emph{version} of \(f\), and that \(g \in
  C\) is a \emph{standard version} if \(g(v) \leq C\) for some constant \(C\)
  and all \(v \in V\).
\end{defn}

\begin{rmk}
  It is straightforward to see that for \(f \in L^\infty(\mu)\), a standard
  version always exists with \(C = \|f\|_\infty\).
\end{rmk}

\begin{rmk}
  Note that in general, the conditional expectation \(\EE_{\mu'}[f \mid \C F]\)
  is defined only \(\mu'\)-a.e. If \(\mu\) is not assumed to be absolutely
  continuous with respect to \(\mu'\), it follows that
  \begin{equation}
  \label{eq:version_rmk}
    \| \EE_\mu[f \mid \C F] - \EE_{\mu'}[f \mid \C F] \|_{L^1(\mu)}
  \end{equation}
  is not entirely well-defined, in that its value depends on what version of
  \(\EE_{\mu'}[f \mid \C F]\) one chooses. For appropriate \(f\), however, one
  can nevertheless bound Eq.~\eqref{eq:version_rmk} for any standard version of
  \(\EE_{\mu'}[f \mid \C F]\).
\end{rmk}

\begin{lem}
\label{lem:ce_est}
  Let \(\mu\), \(\mu'\) be totally bounded measures on a measure space \((V, \C
  M)\). Let \(f \in L^\infty(\mu) \cap L^\infty(\mu')\). Let \(\C F\) be a
  sub--\(\sigma\)-algebra of \(\C M\). Let
  \begin{equation*}
    C = \max(\|f\|_{L^\infty(\mu)}, \|f\|_{L^\infty(\mu')}).
  \end{equation*}
  Then, if \(g\) is a standard version of \(\EE_{\mu'}[f \mid \C F]\), we have
  that
  \begin{equation}
  \label{eq:ce_bound}
    \int_V |\EE_\mu[f \mid \C F] - g| \, \dx \mu \leq 4 C \cdot |\mu - \mu'|[V].
  \end{equation}
\end{lem}

\begin{proof}
  First, we note that both \(\EE_\mu[f \mid \C F]\) and \(g\) are \(\C
  F\)-measurable. Therefore, the sets
  \begin{equation*}
    E^{\up} = \{v \in V : \EE_\mu[f \mid \C F](v) > g(v)\}
  \end{equation*}
  and
  \begin{equation*}
    E^{\low} = \{v \in V : \EE_\mu[f \mid \C F](v) < g(v)\}
  \end{equation*}
  are in \(\C F\). Now, note that
  \begin{multline*}
    \int_V |\EE_\mu[f \mid \C F] - g| \, \dx \mu = \int_{E^{\up}} \EE_\mu[f \mid
    \C F] - g \, \dx \mu \\
      + \int_{E^{\low}} g - \EE_\mu[f \mid \C F] \, \dx \mu.
  \end{multline*}

  First consider \(E^{\up}\). Then, we have that
  \begin{align*}
    \int_{E^{\up}} \EE_\mu[f \mid \C F] &- g \, \dx \mu \\
      &= \int_{E^{\up}} \EE_\mu[f \mid \C F] - g \, \dx \mu \\
      &\hspace{0.5cm}+ \int_{E^{\up}} g - g \, \dx \mu' \\
      &\leq \left| \int_{E^{\up}} \EE_\mu[f \mid \C F] \, \dx \mu -
      \int_{E^{\up}} g \, \dx \mu' \right| \\
      &\hspace{0.5cm}+ \int_{E^{up}} g \, \dx {|\mu - \mu'|} \\
      &\leq \left| \int_{E^{\up}} f \, \dx \mu - \int_{E^{\up}} f \, \dx \mu'
      \right| \\
      &\hspace{0.5cm} + \int_{E^{up}} C \, \dx {|\mu - \mu'|},
  \end{align*}
  where in the final inequality, we have used the fact that, since \(g\) is a
  standard version of \(\EE_{\mu'}[f \mid \C F]\),
  \begin{equation*}
    g(v) \leq \|\EE_{\mu'}[f \mid \C F]\|_{L^\infty(\mu')} \leq C
  \end{equation*}
  for all \(v \in V\), and the fact that, by the definition of conditional
  expectation,
  \begin{equation*}
    \int_E \EE_\nu[h \mid \C F] \, \dx \nu = \int_E h \, \dx \nu
  \end{equation*}
  for any \(E \in \C F\).

  Since \(f\) is everywhere bounded by \(C\), applying Lemma~\ref{lem:tot_var}
  yields that this final expression is less than or equal to \(2 C \cdot |\mu -
  \mu'|[V]\). An identical argument shows that
  \begin{equation*}
    \int_{E^{\low}} g - \EE_\mu[f \mid \C F] \, \dx \mu \leq 2 C \cdot |\mu -
    \mu'|[V],
  \end{equation*}
  whence the result follows.
\end{proof}

\begin{lem}
\label{lem:e_closed}
  Let \(\bb E \subseteq \bb Q\) denote the set of joint densities on \(\C K\)
  such that there exists a policy satisfying the relevant fairness definition
  that is not strongly Pareto dominated. Then, \(\bb E\) is closed, and
  therefore universally measurable.
\end{lem}

\begin{proof}
  For notational simplicity, we consider the case of counterfactual equalized
  odds. The proofs in the other two cases are virtually identical.

  Suppose \(\mu_i \to \mu\) in \(\bb K\), where \(\{\mu_i\}_{i \in \B N}
  \subseteq \bb E\). Then, by Lemma~\ref{lem:mapping}, \(f_{\mu_i, a} \to
  f_{\mu, a}\) in \(L^1(\B R)\). Moreover, by Lemma~\ref{lem:threshold}, there
  exists a sequence of threshold policies \(\{\tau_i(x)\}_{i \in \B N}\) such
  that both
  \begin{equation*}
    \EE_{\mu_i}[\tau(X)] = \min(b, \Pr_{\mu_i}(u(X) > 0))
  \end{equation*}
  and
  \begin{equation*}
    \EE_{\mu_i}[\tau_i(X) \mid A, Y(1)] = \EE_{\mu_i}[\tau_i(X) \mid Y(1)].
  \end{equation*}

  Let \(\{q_{a,i}\}_{a \in \C A}\) be defined by
  \begin{equation*}
    q_{a,i} = \EE_{\mu_i}[\tau_i(X) \mid A = a]
  \end{equation*}
  if \(\Pr_{\mu_i}(A = a) > 0\), and \(q_{a,i} = 0\) otherwise.

  Since \([0,1]^{\C A}\) is compact, there exists a convergent subsequence
  \(\{\{q_{a, n_i}\}_{a \in \C A}\}_{i \in \B N}\). Let it converge to the
  collection of quantiles \(\{q_a\}_{a \in \C A}\) defining, by
  Lemma~\ref{lem:quantile}, a multiple threshold policy \(\tau(x)\) over
  \(\mu\).

  Because \(\mu_i \to \mu\) and \(\{q_{a, n_i}\}_{a \in \C A} \to \{q_a\}_{a \in
  \C A}\), we have that
  \begin{equation*}
    \EE_\mu[\tau_{a, n_i}(X) \mid A = a] \to \EE_\mu[\tau(X) \mid A = a]
  \end{equation*}
  for all \(a \in \C A\) such that \(\Pr_\mu(A = a) > 0\). Therefore, by
  Lemma~\ref{lem:mapping}, \(\tau_{n_i}(X) \to \tau(X)\) in \(L^1(\mu)\).

  Choose \(\epsilon > 0\) arbitrarily. Then, choose \(N\) so large that for
  \(i\) greater than \(N\),
  \begin{align*}
    |\mu - \mu_{n_i}|[\C K] &< \tfrac {\epsilon} {10}, &\|\tau(X) -
    \tau_{n_i}(X)\|_{L^1(\mu)} &\leq \tfrac \epsilon {10}.
  \end{align*}
  Then, observe that \(\tau(x), \tau_i(x) \leq 1\), and recall that
  \begin{equation}
  \label{eq:cea}
    \EE_{\mu_{n_i}}[\tau_{n_i}(X) \mid A, Y(1)] = \EE_{\mu_{n_i}}[\tau_{n_i}(X)
    \mid Y(1)].
  \end{equation}
  Therefore, let \(g_i(x)\) be a standard version of
  \(\EE_{\mu_{n_i}}[\tau_{n_i}(X) \mid Y(1)]\) over \(\mu_{n_i}\). By
  Eq.~\eqref{eq:cea}, \(g_i(x)\) is also a standard version of
  \(\EE_{\mu_{n_i}}[\tau_{n_i}(X) \mid A, Y(1)]\) over \(\mu_{n_i}\). Then, by
  Lemma~\ref{lem:ce_est}, we have that
  \begin{align*}
    \| \EE_\mu&[\tau(X) \mid A, Y(1)] - \EE_{\mu_{n_i}}[\tau_{n_i}(X) \mid Y(1)]
    \|_{L^1(\mu)} \\
      &\leq \| \EE_\mu[\tau(X) \mid A, Y(1)] \\
      &\hspace{1cm} - \EE_\mu[\tau_{n_i}(X) \mid A, Y(1)] \|_{L^1(\mu)} \\
      &\hspace{0.5cm} + \| \EE_\mu[\tau_{n_i}(X) \mid A, Y(1)] - g_i(X)
        \|_{L^1(\mu)} \\
      &\hspace{0.5cm} + \| g_i(X) - \EE_\mu[\tau_{n_i}(X) \mid Y(1)]
        \|_{L^1(\mu)} \|_{L^1(\mu)} \\
      &\hspace{0.5cm} + \| \EE_\mu[\tau_{n_i}(X) \mid Y(1) - \EE_\mu[\tau(X)
        \mid Y(1)] \|_{L^1(\mu)} \\
      &< \frac \epsilon {10} + \frac {4 \epsilon} {10} + \frac {4 \epsilon} {10}
        + \frac \epsilon {10}.
  \end{align*}
  Since \(\epsilon > 0\) was arbitrary, it follows that, \(\mu\)-a.e.,
  \begin{equation*}
    \EE_\mu[\tau(X) \mid A, Y(1)] = \EE_\mu[\tau(X) \mid Y(1)].
  \end{equation*}

  Recall the standard fact that for independent random variables \(X\) and
  \(U\),
  \begin{equation*}
    \EE[f(X,U) \mid X] = \int f(X, u) \diff F_U(u),
  \end{equation*}
  where \(F_U\) is the distribution of \(U\).\footnote{%
    For a proof of this fact see, e.g., \citet{drhab2019conditional}.
  }
  Further recall that \(D = \B 1_{U_D \leq \tau(X)}\), where \(U_D \indep X,
  Y(1)\). It follows that
  \begin{equation*}
    \Pr_\mu(D = 1 \mid X, Y(1)) = \int_0^1 \B 1_{u_d < \tau(X)} \, \dx
    \lambda(u_d) = \tau(X).
  \end{equation*}
  Hence, by the law of iterated expectations,
  \begin{align*}
    \Pr_\mu(D = 1 &\mid A, Y(1)) \\
      &= \EE_\mu[\Pr_\mu(D = 1 \mid X, Y(1)) \mid A, Y(1)] \\
      &= \EE_\mu[\tau(X) \mid A, Y(1)] \\
      &= \EE_\mu[\tau(X) \mid Y(1)] \\
      &= \EE_\mu[\Pr_\mu(D = 1 \mid X, Y(1)) \mid Y(1)] \\
      &= \Pr_\mu(D = 1 \mid Y(1)).
  \end{align*}
  Therefore \(D \indep A \mid Y(1)\) over \(\mu\), i.e., counterfactual
  equalized odds holds for the decision policy \(\tau(x)\) over the distribution
  \(\mu\). Consequently \(\mu \in \bb E\), and so \(\bb E\) is closed and
  therefore universally measurable.
\end{proof}

\subsection{Shy Sets and Probes}
\label{sec:shyness-prelims2}

We require a number of additional technical lemmata for the proof of
Theorem~\ref{thm:dist}. The probe must be constructed carefully, so that, on the
utility scale, an arbitrary element of \(\bb Q\) is absolutely continuous with
respect to a typical perturbation. In addition, it is useful to show that a
number of properties are generic to simplify certain aspects of the proof of
Theorem~\ref{thm:dist}. For instance, Lemma~\ref{lem:condition} is used in
Theorem~\ref{thm:dist} to show that a certain conditional expectation is
generically well-defined, avoiding the need to separately treat certain corner
cases.

Cor.~\ref{cor:maximal} concerns the construction of the probe used in the proof
of Theorem~\ref{thm:dist}. Lemmata~\ref{lem:uncountable_sum}~to~\ref{lem:simple}
use Cor.~\ref{cor:maximal} to provide additional simplifications to the proof of
Theorem~\ref{thm:dist}.

\subsubsection{Maximal support}

First, to construct the probe used in the proof of Theorem~\ref{thm:dist}, we
require elements \(\mu \in \bb Q\) such that the densities \(f_\mu\) have
``maximal'' support. To produce such distributions, we use the following
measure-theoretic construction.

\begin{defn}
  Let \(\{E_\alpha\}_{\gamma \in \Gamma}\) be an arbitrary collection of
  \(\mu\)-measurable sets for some positive measure \(\mu\) on a measure space
  \((M, \C M)\). We say that \(E\) is the \emph{measure-theoretic union} of
  \(\{E_\gamma\}_{\gamma \in \Gamma}\) if \(\mu[E_\gamma \setminus E] = 0\) for
  all \(\gamma \in \Gamma\) and \(E = \bigcup_{i=1}^\infty E_{\gamma_i}\) for
  some countable subcollection \(\{\gamma_i\} \subseteq \B N\).
\end{defn}

While measure-theoretic unions themselves are known
(cf.~\citet{silva2008invitation}, \citet{rudin1991functional}), for
completeness, we include a proof of their existence, which, to the best of our
knowledge, is not found in the literature.

\begin{lem}
\label{lem:mtu}
  Let \(\mu\) be a finite positive measure on a measure space \((V, \C M)\).
  Then an arbitrary collection \(\{E_\gamma\}_{\gamma \in \Gamma}\) of
  \(\mu\)-measurable sets has a measure-theoretic union.
\end{lem}

\begin{proof}
  For each countable subcollection \(\Gamma' \subseteq \Gamma\), consider the
  ``error term''
  \begin{equation*}
    r(\Gamma') = \sup_{\gamma \in \Gamma} \mu \left[ E_{\gamma} \setminus
    \bigcup_{\gamma' \in \Gamma'} E_{\gamma'} \right]
  \end{equation*}
  We claim that the infimum of \(r(\Gamma')\) over all countable subcollections
  \(\Gamma' \subseteq \Gamma\) must be zero.

  For, toward a contradiction, suppose it were greater than or equal to
  \(\epsilon > 0\). Choose any set \(E_{\gamma_1}\) such that
  \(\mu[E_{\gamma_1}] \geq \epsilon\). Such a set must exist, since otherwise
  \(r(\emptyset) < \epsilon\). Choose \(E_{\gamma_2}\) such that
  \(\mu[E_{\gamma_2} \setminus E_{\gamma_1}] > \epsilon\). Again, some such set
  must exist, since otherwise \(r(\{\gamma_1\}) < \epsilon\). Continuing in this
  way, we construct a countable collection \(\{E_{\gamma_i}\}_{i \in \B N}\).

  Therefore, we see that
  \begin{align*}
    \mu[V] \geq \mu \left[ \bigcup_{i=1}^n E_{\gamma_i} \right] = \sum_{i=1}^n
    \mu \left[ E_{\gamma_i} \setminus \bigcup_{j=1}^i E_{\gamma_j} \right].
  \end{align*}
  By construction, every term in the final sum is greater than or equal to
  \(\epsilon\), contradicting the fact that \(\mu[V] < \infty\).

  Therefore, there exist countable collections \(\{\Gamma_n\}_{n \in \B N}\)
  such that \(r(\Gamma_n) < \frac 1 n\). It follows immediately that for all
  \(n\)
  \begin{equation*}
    r \left( \bigcup_{n \in \B N} \Gamma_n \right) \leq r(\Gamma_k)
  \end{equation*}
  for any fixed \(k \in \B N\). Consequently,
  \begin{equation*}
    r \left( \bigcup_{n \in \B N} \Gamma_n \right) = 0,
  \end{equation*}
  and \(\bigcup_{n \in \B N} \Gamma_n\) is countable.
\end{proof}

The construction of the ``maximal'' elements used to construct the probe in the
proof of Theorem~\ref{thm:dist} follows as a corollary of Lemma~\ref{lem:mtu}

\begin{cor}
\label{cor:maximal}
  There are measures \(\mu_{\max, a} \in \bb Q\) such that for every \(a \in \C
  A\) and any \(\mu \in \bb K\),
  \begin{equation*}
    \lambda[\supp(f_{\mu, a}) \setminus \supp(f_{\mu_{\max}, a})] = 0.
  \end{equation*}
\end{cor}

\begin{proof}
  Consider the collection \(\{\supp(f_{\mu, a})\}_{\mu \in \bb K}\). By
  Lemma~\ref{lem:mtu}, there exists a countable collection of measures
  \(\{\mu_i\}_{i \in \B N}\) such that for any \(\mu \in \bb K\),
  \begin{equation*}
    \lambda \left[ \supp(f_{\mu,a}) \setminus \bigcup_{i = 1}^{\infty}
    \supp(f_{\mu_i, a}) \right] = 0,
  \end{equation*}
  where, without loss of generality, we may assume that
  \(\lambda[\supp(f_{\mu_i, a})] > 0\) for all \(i \in \B N\). Such a sequence
  must exist, since, by the first hypothesis of Theorem~\ref{thm:dist}, for
  every \(a \in \C A\), there exists \(\mu \in \bb Q\) such that \(\Pr_\mu(A =
  a) > 0\). Therefore, we can define the probability measure \(\mu_{\max, a}\),
  where
  \begin{equation*}
    \mu_{\max, a} = \sum_{i=1}^n 2^{-i} \cdot \frac {\left| \mu_i \rest_{A = a}
    \right|} {\left |\mu_i \rest_{A = a} \right| [\C K]}.
  \end{equation*}
  It follows immediately by construction that
  \begin{equation*}
    \supp(f_{\mu_{\max}, a}) = \bigcup_{i=1}^\infty \supp(f_{\mu_i,a}),
  \end{equation*}
  and that \(\mu_{\max,a} \in \bb Q\).
\end{proof}

For notational simplicity, we refer to \(\supp(f_{\mu_{\max, a}})\) as \(S_a\)
throughout.

In the case of conditional principal fairness and path-specific fairness, we
need a mild refinement of the previous result that accounts for \(\omega\).

\begin{cor}
\label{cor:maximal_ext}
  There are measures \(\mu_{\max, a, w} \in \bb Q\) defined for every \(w \in \C
  W = \img(\omega)\) and any \(a \in \C A\) such that for some \(\nu \in \bb
  K\), \(\Pr_\nu(W = w, A = a) > 0\). These measures have the property that for
  any \(\mu \in \bb K\),
  \begin{equation*}
    \lambda[\supp(f_{\mu', a, w}) \setminus \supp(f_{\mu_{\max}, a, w})] = 0,
  \end{equation*}
  where \(f_{\mu', a, w}\) is the density of the pushforward measure \((\mu'
  \rest_{W = w, A = a}) \circ u^{-1}\).
\end{cor}

Recalling that \(|\img(\omega)| < \infty\), the proof is the same as
Cor.~\ref{cor:maximal}, and we analogously refer to \(\supp(f_{\mu_{\max, a,
w}})\) as \(S_{a,w}\). Here, we have assumed without loss of generality---as we
continue to assume in the sequel---that for all \(w \in \C W\), there is some
\(\mu \in \bb K\) such that \(\Pr_\mu(W = w) > 0\).

\begin{rmk}
\label{rmk:dist_hyp}
  Because their support is maximal, the hypotheses of Theorem~\ref{thm:dist}, in
  addition to implying that \(\mu_{\max,a}\) is well-defined for all \(a \in \C
  A\), also imply that \(\Pr_{\mu_{\max,a}}(u(X) > 0) > 0\). In the case of
  conditional principal fairness, they further imply that \(\Pr_{\mu_{\max,a}}(W
  = w) > 0\) for all \(w \in \C W\) and \(a \in \C A\). Likewise, in the case of
  path-specific fairness, they further imply that \(\Pr_{\mu_{\max,a}}(W = w_i)
  > 0\) for \(i = 0, 1\) and some \(a \in \C A\).
\end{rmk}

\subsubsection{Shy sets and probes}

In the following lemmata, we demonstrate that a number of useful properties are
generic in \(\bb Q\). We also demonstrate a short technical lemma,
Lemma~\ref{lem:simple}, which allows us to use these generic properties to
simplify the proof of Theorem~\ref{thm:dist}.

We begin with the following lemma, which is useful in verifying that certain
subspaces of \(\bb K\) form probes.

\begin{lem}
\label{lem:probe}
  Let \(\bb W\) be a non-trivial finite dimensional subspace of \(\bb K\) such
  that \(\nu[\C K] = 0\) for all \(\nu \in \bb W\). Then, there exists \(\mu \in
  \bb K\) such that \(\lambda_{\bb W}[\bb Q - \mu] > 0\).
\end{lem}

\begin{proof}
  Set
  \begin{equation*}
    \mu = \sum_{i=1}^n \frac {|\nu_i|}{|\nu_i|[\C K]},
  \end{equation*}
  where \(\nu_1, \ldots, \nu_n\) form a basis of \(\bb W\). Then, if \(|\beta_i|
  \leq \tfrac 1 {|\nu_i|[\C K]}\), it follows that
  \begin{equation*}
    \mu + \sum_{i = 1}^n \beta_i \cdot \nu_i \in \bb Q.
  \end{equation*}
  Since
  \begin{equation*}
    \lambda_n \left[ \prod_{i=1}^n \left[ - \frac 1 {|\nu_i|[\C K]}, \frac 1
    {|\nu_i|[\C K]} \right] \right] > 0,
  \end{equation*}
  it follows that \(\lambda_{\bb W}[\bb Q - \mu] > 0\).
\end{proof}

Next we show that, given a \(\nu \in \bb Q\), a generic element of \(\bb Q\)
``sees'' events to which \(\nu\) assigns non-zero probability. While
Lemma~\ref{lem:support} alone in principle suffices for the proof of
Theorem~\ref{thm:dist}, we include Lemma~\ref{lem:condition} both for conceptual
clarity and to introduce at a high level the style of argument used in the
subsequent lemmata and in the proof of Theorem~\ref{thm:dist} to show that a set
is shy relative to \(\bb Q\).

\begin{lem}
\label{lem:condition}
  For a Borel set \(E \subseteq \C K\), suppose there exists \(\nu \in \bb Q\)
  such that \(\nu[E] > 0\). Then the set of \(\mu \in \bb Q\) such that \(\mu[E]
  > 0\) is prevalent.
\end{lem}

\begin{proof}
  First, we note that the set of \(\mu \in \bb Q\) such that \(\mu[E] = 0\) is
  closed and therefore universally measurable. For, if \(\{\mu_i\}_{i \in \B N}
  \subseteq \bb Q\) is a convergent sequence with limit \(\mu\), then
  \begin{align*}
    \mu[E]
      &= \lim_{n \to \infty} \mu_i[E] \\
      &= \lim_{n \to \infty} 0 \\
      &= 0.
  \end{align*}
  Now, if \(\mu[E] > 0\) for all \(\mu \in \bb Q\), there is nothing to prove.
  Therefore, suppose that there exists \(\nu' \in \bb Q\) such that \(\nu'[E] =
  0\).

  Next, consider the measure \(\tilde \nu = \nu' - \nu\). Then, let \(\bb W =
  \Span(\tilde \nu)\). Since \(\tilde \nu \neq 0\) and
  \begin{equation*}
    \tilde \nu[\C K] = \nu'[\C K] - \nu[\C K] = 0,
  \end{equation*}
  it follows by Lemma~\ref{lem:probe} that \(\lambda_{\bb W}[\bb Q - \mu] > 0\)
  for some \(\mu\).

  Now, for arbitrary \(\mu \in \bb Q\), note that if \((\mu + \beta \cdot \tilde
  \nu)[E] = 0\), then
  \begin{equation*}
    \mu[E] - \beta \cdot \nu[E] = 0
  \end{equation*}
  i.e.,
  \begin{equation*}
    \beta = \frac {\mu[E]} {\nu[E]}.
  \end{equation*}
  A singleton has null Lebesgue measure, and so the set of \(\nu \in \bb W\)
  such that \((\mu + \nu)[E] = 0\) is \(\lambda_{\bb W}\)-null. Therefore, by
  Prop.~\ref{prop:k_shy}, the set of \(\mu \in \bb Q\) such that \(\mu[E] = 0\)
  is shy relative to \(\bb Q\), as desired.
\end{proof}

While Lemma~\ref{lem:condition} shows that a typical element of \(\bb Q\)
``sees'' individual events, in the proof of Theorem~\ref{thm:dist}, we require a
stronger condition, namely, that a typical element of \(\bb Q\) ``sees'' certain
uncountable collections of events. To demonstrate this more complex property, we
require the following technical result, which is closely related to the real
analysis folk theorem that any convergent uncountable ``sum'' can contain only
countably many non-zero terms. (See, e.g., \citet{benji2020sum}.)

\begin{lem}
\label{lem:uncountable_sum}
  Suppose \(\mu\) is a totally bounded measure on \((V, \C M)\), \(f\) and \(g\)
  are \(\mu\)-measurable real-valued functions, and \(g \neq 0\) \(\mu\)-a.e.
  Then the set
  \begin{equation*}
    Z_\beta = \{v \in V : f(v) + \beta \cdot g(v) = 0\}
  \end{equation*}
  has non-zero \(\mu\) measure for at most countably many \(\beta \in \B R\).
\end{lem}

\begin{proof}
  First, we show that for any countable collection \(\{\beta_i\}_{i \in \B N}
  \subseteq \B R\), the sum \(\sum_{i=1}^\infty \mu[Z_{\beta_i}]\) converges.
  Then, we show how this implies that \(\mu[Z_{\beta}] = 0\) for all but
  countably many \(\beta \in \B R\).

  First, we note that for distinct \(\beta, \beta' \in \B R\),
  \begin{equation*}
    Z_{\beta} \cap Z_{\beta'} \subseteq \{v \in V : (\beta - \beta') \cdot g(v)
    = 0\}.
  \end{equation*}
  Now, by hypothesis,
  \begin{equation*}
    \mu[\{v \in V : g(v) = 0\}] = 0,
  \end{equation*}
  and since \(\beta - \beta' \neq 0\), it follows that
  \begin{equation*}
    \mu[\{v \in V : (\beta - \beta') \cdot g(v) = 0\}] = 0
  \end{equation*}
  as well. Consequently, it follows that if \(\{Z_{\beta_i}\}_{i \in \B N}\) is
  a countable collection of distinct elements of \(\B R\), then
  \begin{align*}
    \sum_{i=1}^\infty \mu[Z_{\beta_i}]
      &= \mu \left[\bigcup_{i=1}^\infty Z_{\beta_i} \right] \\
      &\leq \mu[V] \\
      &< \infty.
  \end{align*}

  To see that this implies that \(\mu[Z_\beta] > 0\) for only countably many
  \(\beta \in \B R\), let \(G_\epsilon \subseteq \B R\) consist of those
  \(\beta\) such that \(\mu[Z_\beta] \geq \epsilon\). Then \(G_\epsilon\) must
  be finite for all \(\epsilon > 0\), since otherwise we could form a collection
  \(\{\beta_i\}_{i \in \B N} \subseteq G_\epsilon\), in which case
  \begin{equation*}
    \sum_{i=1}^\infty \mu[Z_{\beta_i}] \geq \sum_{i=1}^\infty \epsilon = \infty,
  \end{equation*}
  contrary to what was just shown. Therefore,
  \begin{equation*}
    \{\beta \in \B R : \mu[Z_\beta] > 0\} = \bigcup_{i=1}^\infty G_{1/i}
  \end{equation*}
  is countable.
\end{proof}

We now apply Lemma~\ref{lem:uncountable_sum} to the proof of the following
lemma, which states, informally, that, under a generic element of \(\bb Q\),
\(u(X)\) is supported everywhere it is supported under some particular fixed
element of \(\bb Q\). For instance, Lemma~\ref{lem:uncountable_sum} can be used
to show that for a generic element of \(\bb Q\), the density of \(u(X) \mid A =
a\) is positive \(\lambda \rest_{S_a}\)-a.e.

\begin{lem}
\label{lem:support}
  Let \(\nu \in \bb Q\) and suppose \(\nu\) is supported on \(E\), i.e.,
  \(\nu[\C K \setminus E] = 0\). Then the set of \(\mu \in \bb Q\) such that
  \(\nu \circ u^{-1} \Lt (\mu \rest_E) \circ u^{-1}\) is prevalent relative to
  \(\bb Q\).
\end{lem}

Lemma~\ref{lem:support} states, informally, that for generic \(\mu \in \bb Q\),
\(f_{\mu \rest_E}\) is supported everywhere \(f_\nu\) is supported.

\begin{proof}
  We begin by showing that the set of \(\mu \in \bb Q\) such that \(\nu \circ
  u^{-1} \Lt (\mu \rest_E) \circ u^{-1}\) is Borel, and therefore universally
  measurable. Then, we construct a probe \(\bb W\) and use it to show that this
  collection is finitely shy.

  To begin, let \(U_q\) denote the set of \(\mu \in \bb Q\) such that
  \begin{equation*}
    \nu \circ u^{-1}[\{|f_{\mu \rest_E}| = 0\}] < q.
  \end{equation*}
  We note that \(U_q\) is open. For, if \(\mu \in U_q\), then there exists some
  \(r > 0\) such that
  \begin{equation*}
    \nu \circ u^{-1}[\{|f_{\mu \rest_E}| < r\}] < q.
  \end{equation*}
  Let
  \begin{equation*}
    \epsilon = q - \nu \circ u^{-1}[\{|f_{\mu \rest_E}| < r\}].
  \end{equation*}
  Now, since \(\nu \circ u^{-1} \Lt \lambda\), there exists a \(\delta\) such
  that if \(\lambda[E'] < \delta\), then \(\nu \circ u^{-1}[E'] < \epsilon\).
  Choose \(\mu'\) arbitrarily so that \(|\mu - \mu'|[\C K] < \delta \cdot r\).
  Then, by Markov's inequality, we have that
  \begin{equation*}
    \lambda[\{|f_{\mu \rest_E} - f_{\mu' \rest_E}| > r\}] < \delta,
  \end{equation*}
  i.e.,
  \begin{equation*}
    \nu \circ u^{-1}[\{f_{\mu \rest_E} - f_{\mu' \rest_E}| > r\}] < \epsilon.
  \end{equation*}
  Now, we note that by the triangle inequality, wherever
  \(|f_{\mu' \rest_E}| = 0\), either \(|f_{\mu \rest_E}| < r\) or \(|f_{\mu
  \rest_E} - f_{\mu' \rest_E}| > r\). Therefore
  \begin{align*}
    \lambda[\{|f_{\mu' \rest_E}| = 0\}]
      &\leq \nu \circ u^{-1}[\{|f_{\mu \rest_E} - f_{\mu' \rest_E}| > r\}] \\
      &\hspace{1cm} + \mu \circ u^{-1} [\{|f_{\mu \rest_E}| < r] \\
      &< \epsilon + \mu \circ u^{-1}[\{|f_{\mu \rest_E}| < r] \\
      &< q.
  \end{align*}
  We conclude that \(\mu' \in U_q\), and so \(U_q\) is open.

  Note that \(\nu \circ u^{-1} \Lt (\mu \rest_E) \circ u^{-1}\) if and only if
  \begin{equation*}
    \lambda[\supp(f_\nu) \setminus \supp(f_{\mu \rest_E})] = 0.
  \end{equation*}
  By the definition of the support of a function, \(\lambda \rest_{\supp(f_\mu)}
  \Lt \mu \circ u^{-1}\). Therefore, it follows that
  \begin{equation*}
    \lambda[\supp(f_\mu) \setminus \supp(f_{\nu \rest_E})] = 0
  \end{equation*}
  if and only if
  \begin{equation*}
    \mu \circ u^{-1}[\supp(f_\mu) \setminus \supp(f_{\nu \rest_E})] = 0.
  \end{equation*}
  Then, it follows immediately that the set of \(\nu \in \bb Q\) such that \(\mu
  \circ u^{-1} \Lt (\nu \rest_E) \circ u^{-1}\) is \(\bigcap_{i=1}^n U_{1/i}\),
  which is, by construction, Borel, and therefore universally measurable.

  Now, since
  \begin{equation*}
    \Pr_\nu(u(X) < t) = \int_{-\infty}^t f_{\nu} \, \dx \lambda
  \end{equation*}
  is a continuous function of \(t\), by the intermediate value theorem, there
  exists \(t\) such that \(\Pr_\nu(u(X) \in S_0) = \Pr_\nu(u(X) \in S_1)\),
  where \(S_0 = \supp(f_\nu) \cap (-\infty, t)\) and \(S_1 = \supp(f_\nu) \cap
  [t, \infty)\). Then, we let
  \begin{equation*}
    \tilde \nu[E'] = \int_{E'} \B 1_{u^{- 1} (S_0)} - \B 1_{u^{-1}(S_1)} \, \dx
    \nu.
  \end{equation*}

  Take \(\bb W = \Span(\tilde \nu)\). Since \(\tilde \nu \neq 0\) and \(\tilde
  \nu[\C K] = 0\), we have by Lemma~\ref{lem:probe} that \(\lambda_{\bb W}[\bb Q
  - \mu] > 0\) for some \(\mu\).

  By the definition of a density, \(f_{\tilde \nu}\) is positive \((\tilde \nu
  \circ u^{-1})\)-a.e. Consequently, by the definition of \(\tilde \nu\),
  \(f_{\tilde \nu}\) is non-zero \((\mu \circ u^{-1})\)-a.e. Therefore, by
  Lemma~\ref{lem:uncountable_sum}, there exist only countably many \(\beta \in
  \B R\) such that the density of \((\mu + \beta \cdot \tilde \nu) \circ
  u^{-1}\) equals zero on a set of positive \(\mu \circ u^{-1}\)-measure. Since
  countable sets have \(\lambda\)-measure zero and \(\nu\) is arbitrary, the set
  of \(\mu \in \bb Q\) such that \(\nu \circ u^{-1} \Lt (\mu \rest_E) \circ
  u^{-1}\) is prevalent relative to \(\bb Q\) by Prop.~\ref{prop:k_shy}.
\end{proof}

The following definition and technical lemma are needed to extend
Theorem~\ref{thm:dist} to the cases of conditional principal fairness and
path-specific fairness, which involve additional conditioning on \(W =
\omega(X)\). In particular, one corner case we wish to avoid in the proof of
Theorem~\ref{thm:dist} is when the decision policy is non-trivial (i.e., some
individuals receive a positive decision and others do not) but from the
perspective of each \(\omega\)-stratum, the policy is trivial (i.e., everyone in
the stratum receives a positive or negative decision).
Definition~\ref{defn:overlap} formalizes this pathology, and
Lemma~\ref{lem:overlap} shows that this issue---under a mild hypothesis---does
not arise for a generic element of \(\bb Q\).

\begin{defn}
\label{defn:overlap}
  We say that \(\mu \in \bb Q\) \emph{overlaps utilities} when, for any
  budget-exhausting multiple threshold policy \(\tau(x)\), if
  \begin{equation*}
    0 < \EE_\mu[\tau(X)] < 1,
  \end{equation*}
  then there exists \(w \in \C W\) such that
  \begin{equation*}
    0 < \EE_\mu[\tau(X) \mid W = w] < 1.
  \end{equation*}

  If there exists a budget-exhausting multiple threshold policy \(\tau(x)\) such
  that
  \begin{equation*}
    0 < \EE_\mu[\tau(X)] < 1,
  \end{equation*}
  but for all \(w \in \C W\),
  \begin{equation*}
    \EE_\mu[\tau(X) \mid W = w] \in \{0, 1\},
  \end{equation*}
  then we say that \(\tau(x)\) \emph{splits utilities} over \(\mu\).
\end{defn}

Informally, having overlapped utilities prevents a budget-exhausting threshold
policy from having thresholds that fall on the utility scale exactly between the
strata induced by \(\omega\)---i.e., a threshold policy that splits utilities.
This is almost a generic condition in \(\bb Q\), as we shown in
Lemma~\ref{lem:overlap}.

\begin{lem}
\label{lem:overlap}
  Let \(0 < b < 1\). Suppose that for all \(w \in \C W\) there exists \(\mu \in
  \bb Q\) such that \(\Pr_\mu(u(X) > 0, W = w) > 0\). Then almost every \(\mu
  \in \bb Q\) overlaps utilities.
\end{lem}

\begin{proof}
  Our goal is to show that the set \(\bb E'\) of measures \(\mu \in \bb Q\) such
  that there exists a splitting policy \(\tau(x)\) is shy. To simplify the
  proof, we divide an conquer, showing that the set \(\bb E_{\Gamma}\) of
  measures \(\mu \in \bb Q\) such that there exists a splitting policy where the
  thresholds fall below \(w \in \Gamma \subseteq \C W\) and above \(w \notin
  \Gamma\) is Borel, before constructing a probe that shows that it is shy.
  Then, we argue that \(\bb E' = \bigcup_{\Gamma \subseteq \C W} \bb
  E_{\Gamma}\), which shows that \(\bb E'\) is shy.

  We begin by considering the linear map \(\Phi : \bb K \to \B R \times \B R^{\C
  W}\) given by
  \begin{equation*}
    \Phi(\mu) = \left( \Pr_\mu(u(X) = 0), \left( \Pr_\mu(W = w) \right)_{w \in
    \C W} \right).
  \end{equation*}
  For any \(\Gamma \subseteq \C W\), the sets
  \begin{align*}
    F^{\up}_{\Gamma}
      &= \{x \in \B R \times \B R^{\C W} : x_0 \geq b, b = \sum_{w \in \Gamma}
      x_w\}, \\
    F^{\low}_{\Gamma}
      &= \{x \in \B R \times \B R^{\C W} : x_0 \leq b, x_0 = \sum_{w \in \Gamma}
      x_w\},
  \end{align*}
  are closed by construction. Therefore, since \(\Phi\) is continuous,
  \begin{equation}
  \label{eq:shy_set}
    \bb E_{\Gamma} = \bb Q \cap \Phi^{-1} \left( \bigcup_{\Gamma \subseteq \C W}
    F^{\up}_{\Gamma} \cup F^{\low}_{\Gamma} \right)
  \end{equation}
  is closed, and therefore universally measurable.

  Note that by our hypothesis and Cor.~\ref{cor:maximal_ext}, for all \(w \in \C
  W\) there exists some \(a_w \in \C A\) such that
  \begin{equation*}
    \Pr_{\mu_{\max, a_w, w}}(u(X) > 0).
  \end{equation*}
  We use this to show that \(\bb E_{\Gamma}\) is shy. Pick \(w^* \in \C W\)
  arbitrarily, and consider the measures \(\nu_w\) for \(w \neq w^*\) defined by
  \begin{align*}
    \nu_w &= \frac {\mu_{\max, a_w, w} \rest_{u(X) > 0}} {\Pr_{\mu_{\max, a_w,
    w}}(u(X) > 0)} \\
      &\hspace{1cm}- \frac {\mu_{\max, a_{w^*}, w^*} \rest_{u(X) > 0}}
      {\Pr_{\mu_{\max, a_{w^*}, w^*}}(u(X) > 0)}.
  \end{align*}
  We note that \(\nu_w[\C K] = 0\) by construction. Therefore, if \(\bb W_w =
  \Span(\nu_w)\), then \(\lambda_{\bb W_w}[\bb Q - \mu_w] > 0\) for some
  \(\mu_w\) by Lemma~\ref{lem:probe}.

  Moreover, we have that \(\Pr_\nu(u(X) > 0) = 0\) for all \(\nu \in \bb W_w\),
  i.e.,
  \begin{equation*}
  \label{eq:static}
    \Pr_\mu(u(X) > 0) = \Pr_{\mu + \nu}(u(X) > 0).
  \end{equation*}
  Now, since \(0 < b < 1\) and \(\omega\) partitions \(\C X\), it follows that
  \begin{equation*}
    \bb E_{\C W} = \bb E_{\emptyset} = \emptyset.
  \end{equation*}
  Since \(\lambda_{\bb W}[\emptyset] = 0\) for any subspace \(\bb W\), we can
  assume without loss of generality that \(\Gamma \neq \C W, \emptyset\).

  In that case, there exists \(w_{\Gamma} \in \C W\) such that if \(w^* \in
  \Gamma\), then \(w_{\Gamma} \notin \Gamma\), and vice versa. Without loss of
  generality, assume \(w_{\Gamma} \in \Gamma\) and \(w^* \notin \Gamma\). It
  then follows that for arbitrary \(\mu \in \bb Q\),
  \begin{equation*}
    \Phi(\mu + \beta \cdot \nu_{w_{\Gamma}}) = \Phi(\mu) + \beta \cdot \bb
    e_{w_{\Gamma}} - \beta \cdot \bb e_{w^*},
  \end{equation*}
  where \(\bb e_w\) is the basis vector corresponding to \(w \in \C W\). From
  this, it follows immediately by Eq.~\eqref{eq:static} that
  \begin{equation*}
    \mu + \beta \cdot \nu_{w_{\Gamma}} \in \bb E_{\Gamma}
  \end{equation*}
  only if
  \begin{equation*}
    \beta = \min(b, \Pr_\mu(u(X) > 0)) - \sum_{w \in \Gamma} \Pr_\mu(W = w).
  \end{equation*}
  This is a measure zero subset of \(\B R\), and so it follows that
  \begin{equation*}
    \lambda_{\bb W_{w_{\Gamma}}}[\bb E_{\Gamma} - \mu] = 0
  \end{equation*}
  for all \(\mu \in \bb K\). Therefore, by Prop.~\ref{prop:k_shy}, \(\bb
  E_{\Gamma}\) is shy in \(\bb Q\). Taking the union over \(\Gamma \subseteq \C
  W\), it follows by Prop.~\ref{prop:shy_axioms_rel} that \(\bigcup_{\Gamma
  \subseteq \C W} \bb E_{\Gamma}\) is shy.

  Now, we must show that \(\bb E' = \bigcup_{\Gamma \subseteq \C W} \bb
  E_{\Gamma}\). By construction, \(\bb E_{\Gamma} \subseteq \bb E'\), since the
  policy \(\tau(x) = \B 1_{\omega(x) \in \Gamma}\) is budget-exhausting and
  separates utilities. To see the reverse inclusion, suppose \(\mu \in \bb E'\),
  i.e., that there exists a budget-exhausting multiple threshold policy
  \(\tau(x)\) that splits utilities over \(\mu\). Then, let
  \begin{equation*}
    \Gamma_\mu = \{w \in \C W : \EE_\mu[\tau(X) \mid W = w] = 1\}.
  \end{equation*}
  Since \(\tau(x)\) is budget-exhausting, it follows immediately that \(\mu \in
  \bb E_{\Gamma_{\mu}}\). Therefore, \(\bb E' = \bigcup_{\Gamma \subseteq \C W}
  \bb E_{\Gamma}\), and so \(\bb E'\) is shy, as desired.
\end{proof}

We conclude our discussion of shyness and shy sets with the following general
lemma, which simplifies relative prevalence proofs by showing that one can,
without loss of generality, restrict one's attention to the elements of the shy
set itself in applying Prop.~\ref{prop:k_shy}.

\begin{lem}
\label{lem:simple}
  Suppose \(E\) is a universally measurable subset of a convex, completely
  metrizable set \(C\) in a topological vector space \(V\). Suppose that for
  some finite-dimensional subpace \(V'\), \(\lambda_{V'}[C + v_0] > 0\) for some
  \(v_0 \in V\). If, in addition, for all \(v \in E\),
  \begin{equation}
  \label{eq:hyp}
    \lambda_{V'}[\{v' \in V' : v + v' \in E\}] = 0,
  \end{equation}
  then it follows that \(E\) is shy relative to \(C\).
\end{lem}

\begin{proof}
  Let \(v\) be arbitrary. Then, either \((v + V') \cap E\) is empty or not.

  First, suppose it is empty. Since \(\lambda_{V'}[\emptyset] = 0\) by
  definition, it follows immediately that in this case \(\lambda_{V'}[E - v] =
  0\).

  Next, suppose the intersection is not empty, and let \(v + v^* \in E\) for
  some fixed \(v^* \in V'\). It follows that
  \begin{align*}
    \lambda_{V'}[E - v]
      &= \lambda_{V'}[\{v' \in V' : v + v' \in E \}] \\
      &= \lambda_{V'}[\{v' \in V' : (v + v^*) + v' \in E \}] \\
      &= 0,
  \end{align*}
  where the first equality follows by definition; the second equality follows by
  the translation invariance of \(\lambda_{V'}\), and the fact that \(v^* + V' =
  V'\); and the final inequality follows from Eq.~\eqref{eq:hyp}.

  Therefore \(\lambda_{V'}[E - v] = 0\) for arbitrary \(v\), and so \(E\) is
  shy.
\end{proof}

\subsection{Proof of Theorem~\ref{thm:dist}}
\label{sec:shyness-proof}

Using the lemmata above, we can prove Theorem~\ref{thm:dist}. We briefly
summarize what has been established so far:
\begin{itemize}
  \item \textbf{Lemma~\ref{lem:banach}:} The set \(\bb K\) of \(\C U\)-fine
    distributions on \(\C K\) is a Banach space;
  \item \textbf{Lemma~\ref{lem:convex}:} The subset \(\bb Q\) of \(\C U\)-fine
    probability measures on \(\C K\) is a convex, completely metrizable subset
    of \(\bb K\);
  \item \textbf{Lemma~\ref{lem:e_closed}:} The subset \(\bb E\) of \(\bb Q\) is
    a universally measurable subset of \(\bb K\), where \(\bb E\) is the set
    consisting of \(\C U\)-fine probability measures over which there exists a
    policy satisfying counterfactual equalized odds (resp., conditional
    principal fairness, or path-specific fairness) that is not strongly Pareto
    dominated.
\end{itemize}

Therefore, to apply Prop.~\ref{prop:k_shy}, it follows that what remains is to
construct a probe \(\bb W\) and show that \(\lambda_{\bb W}[\bb Q + \mu_0] > 0\)
for some \(\mu_0 \in \bb K\) but \(\lambda_{\bb W}[\bb E + \mu] = 0\) for all
\(\mu \in \bb K\).

\begin{proof}
  We divide the proof into three pieces. First, we illustrate how to construct
  the probe \(\bb W\) from a particular collection of distributions
  \(\{\nu_a^{\up}, \nu_a^{\low}\}_{a \in \C A}\). Second, we show that
  \(\lambda_{\bb W}[\bb E + \mu] = 0\) for all \(\mu \in \bb K\). For notational
  and expository simplicity, we focus in these first two sections on the case of
  counterfactual equalized odds. Therefore, in the third section, we show how to
  generalize the argument to conditional principal fairness and path-specific
  fairness.

  \paragraph{Construction of the probe}

  We will construct our probe to address two different cases. We recall that, by
  Cor.~\ref{cor:exhaust}, any policy that is not strongly Pareto dominated must
  be a budget-exhausting multiple threshold policy with non-negative thresholds.
  In the first case, we consider when the candidate budget-exhausting multiple
  threshold policy is \(\B 1_{u(x) > 0}\). By perturbing the underlying
  distribution by \(\nu \in \bb W^{\low}\), we will be able to break the balance
  requirements implied by Eq.~\eqref{eq:counterfactual_equalized_odds}. In the
  second case, we treat the possibility that the candidate budget-exhausting
  multiple threshold policy has a non-trivial positive threshold for at least
  one group. By perturbing the underlying distribution by \(\nu \in \bb
  W^{\up}\) for an alternative set of perturbations \(\bb W^{\up}\), we will
  again be able to break the balance requirements.

  More specifically, to construct our probe \(\bb W = \bb W^{\up} + \bb
  W^{\low}\), we want \(\bb W^{\up}\) and \(\bb W^{\low}\) to have a number of
  properties. In particular, for all \(\nu \in \bb W\), perturbation by \(\nu\)
  should not affect whether the underlying distribution is a probability
  distribution, and should not affect how much of the budget is available to
  budget-exhausting policies. Specifically, for all \(\nu \in \bb W\),
  \begin{equation}
  \label{eq:preserve_prob}
    \int_{\C K} 1 \, \dx \nu = 0,
  \end{equation}
  and
  \begin{equation}
  \label{eq:preserve_budget}
    \int_{\C K} \B 1_{u(X) > 0} \, \dx \nu = 0.
  \end{equation}
  In fact, the amount of budget available to budget-exhausting policies will not
  change within group, i.e., for all \(a \in \C A\) and \(\nu \in \bb W\),
  \begin{equation}
  \label{eq:preserve_budget_ext}
    \int_{\C K} \B 1_{u(X) > 0, A = a} \, \dx \nu = 0.
  \end{equation}
  Additionally, for some distinguished \(y_0, y_1 \in \C Y\), non-zero
  perturbations in \(\nu^{\low} \in \bb W^{\low}\) should move mass between
  \(y_0\) and \(y_1\). That is, they should have the property that if
  \(\Pr_{|\nu^{\low}|}(A = a) > 0\), then
  \begin{equation}
  \label{eq:see_zero}
    \int_{\C K} \B 1_{u(X) < 0, Y = y_i, A = a} \, \dx \nu^{\low} \neq 0.
  \end{equation}
  Finally, perturbations in \(\bb W^{\up}\) should have the property that for
  any non-trivial \(t > 0\), some mass is moved either above or below \(t > 0\).
  More precisely, for any \(\mu \in \bb Q\) and any \(t\) such that
  \begin{equation*}
    0 < \Pr_\mu(u(X) > t \mid A = a) < 1,
  \end{equation*}
  if \(\nu^{\up} \in \bb W^{\up}\) is such that \(\Pr_{|\nu^{\up}|}(A = a) >
  0\), then
  \begin{equation}
  \label{eq:see_threshold}
    \int_{\C K} \B 1_{u(X) > t, A = a} \, \dx \nu^{\up} \neq 0.
  \end{equation}

  To carry out the construction, choose distinct \(y_0, y_1 \in \C Y\). Then,
  since
  \begin{equation*}
    \mu_{\max,a} \circ u^{-1}[S_a \cap [0, r_a)] - \mu_{\max,a} \circ u^{-1}[S_a
    \cap [r_a, \infty)]
  \end{equation*}
  is a continuous function of \(r_a\), it follows by the intermediate value
  theorem that we can partition \(S_a\) into three pieces,
  \begin{align*}
    S_a^{\low}
      &= S_a \cap (-\infty, 0), \\
    S^{\up}_{a,0}
      &= S_a \cap [0, r_a), \\
    S^{\up}_{a,1}
      &= S_a \cap [r_a, \infty),
  \end{align*}
  so that
  \begin{equation*}
    \Pr_{\mu_{\max,a}} \left( u(X) \in S^{\up}_{a,0} \right) =
    \Pr_{\mu_{\max,a}} \left( u(X) \in S^{\up}_{a,1} \right).
  \end{equation*}

  Recall that \(\C K = \C X \times \C Y\). Let \(\pi_{\C X} : \C K \to \C X\)
  denote projection onto \(\C X\), and \(\gamma_y : \C X \to \C K\) be the
  injection \(x \mapsto (x, y)\). We define
  \begin{align*}
    \nu_a^{\up}[E]
      &= \mu_{\max,a} \circ (\gamma_{y_1} \circ \pi_{\C X})^{-1} \left[ E \cap
        u^{-1} \left( S^{\up}_{a,1} \right) \right], \\
      &\hspace{0.5cm}- \mu_{\max,a} \circ (\gamma_{y_1} \circ \pi_{\C X})^{-1}
        \left[ E \cap u^{-1} \left( S^{\up}_{a,0} \right) \right], \\
    \nu_a^{\low}[E]
      &= \mu_{\max,a} \circ (\gamma_{y_1} \circ \pi_{\C X})^{-1} \left[ E \cap
        u^{-1} \left( S^{\low}_a \right) \right] \\
      &\hspace{0.5cm}- \mu_{\max,a} \circ (\gamma_{y_0} \circ \pi_{\C X})^{-1}
        \left[ E \cap u^{-1} \left( S^{\low}_a \right) \right].
  \end{align*}
  By construction, \(\nu_a^{\up}\) concentrates on
  \begin{equation*}
    \{y_1\} \times u^{-1}(S_a \cap [0, \infty)),
  \end{equation*}
  while \(\nu_a^{\low}\) concentrates on
  \begin{equation*}
    \{y_0, y_1\} \times u^{-1}(S_a \cap (-\infty, 0)).
  \end{equation*}
  Moreover, if we set
  \begin{align*}
    \bb W^{\up}
      &= \Span(\nu_a^{\up})_{a \in \C A}, \\
    \bb W^{\low}
      &= \Span(\nu_a^{\low})_{a \in \C A},
  \end{align*}
  then it is easy to see that
  Eqs.~\eqref{eq:preserve_prob}~to~\eqref{eq:see_zero} will hold. The only
  non-trivial case is Eq.~\eqref{eq:see_threshold}. However, by
  Cor.~\ref{cor:maximal}, the support of \(f_{\mu_{\max,a}}\) is maximal. That
  is, for \(\mu \in \bb Q\), if
  \begin{equation*}
    0 < \Pr_{\mu}(u(X) > t \mid A = a, u(X) > 0) < 1,
  \end{equation*}
  then it follows that \(0 < t < \sup S_a\). Either \(t \leq r_a\) or \(t >
  r_a\). First, assume \(t \leq r_a\); then, it follows by the construction of
  \(\nu_a^{\up}\) that
  \begin{align*}
    \nu_a^{\up} \circ u^{-1}[(t, \infty)]
      &= \int_{r_a}^\infty f_{\max,a} \, \dx \lambda \\
      &\hspace{1cm} - \int_t^{r_a} f_{\max,a} \, \dx \lambda \\
      &> \int_{r_a}^\infty f_{\max,a} \, \dx \lambda \\
      &\hspace{1cm} - \int_0^{r_a} f_{\max,a} \, \dx \lambda \\
      &= 0.
  \end{align*}
  Similarly, if \(t > r_a\),
  \begin{align*}
    \nu_a^{\up} \circ u^{-1}[(t, \infty)]
      &= \int_t^\infty f_{\max,a} \, \dx \lambda \\
      &> \int_{\sup S_a}^\infty f_{\max,a} \, \dx \lambda \\
      &= 0.
  \end{align*}
  Therefore Eq.~\eqref{eq:see_threshold} holds.

  Since \(\bb W\) is non-trivial\footnote{%
    In general, some or all of the \(\nu^{\low}\) may be zero depending on the
    \(\lambda\)-measure of \(S_a^{\low}\). However, as noted in
    Remark~\ref{rmk:dist_hyp}, the \(\nu_{a,i}^{\up}\) cannot be zero, since
    \(\Pr_{\mu_{\max,a}}(u(X) > 0) > 0\) for all \(a \in \C A\). Therefore \(\bb
    W \neq \{0\}\).
  }
  and \(\nu[\C K] = 0\) for all \(\nu \in \bb W\), it follows by
  Lemma~\ref{lem:probe} that  \(\lambda_{\bb W}[\bb Q - \mu] > 0\) for some
  \(\mu \in \bb K\).

  \paragraph{Shyness}

  Recall that, by Prop.~\ref{prop:shy_axioms_rel}, a set \(E\) is shy if and
  only if, for an arbitrary shy set \(E'\), \(E \setminus E'\) is shy. By
  Lemma~\ref{lem:condition}, a generic element of \(\mu \in \bb Q\) satisfies
  \(\Pr_{\mu}(u(X) > 0, Y(1) = y_i, A = a) > 0\) for \(i = 0, 1\), and \(a \in
  \C A\). Likewise, by Lemma~\ref{lem:support}, a generic \(\mu \in \bb Q\)
  satisfies \(\nu_a^{\up} \circ u^{-1} \Lt (\mu \rest_{\C X \times \{y_1\}})
  \circ u^{-1}\). Therefore, to simplify our task and recalling
  Remark~\ref{rmk:dist_hyp}, we may instead demonstrate the shyness of the set
  of \(\mu \in \bb Q\) such that:
  \begin{itemize}
    \item There exists a budget-exhausting multiple threshold policy \(\tau(x)\)
      with non-negative thresholds satisfying counterfactual equalized odds over
      \(\mu\);
    \item For \(i = 0, 1\),
      \begin{equation}
      \label{eq:imperturbable}
        \Pr_{\mu}(u(X) > 0, A = a, Y(1) = y_i) > 0;
      \end{equation}
    \item For all \(a \in \C A\),
      \begin{equation}
      \label{eq:nice_abscont}
        \nu_{a}^{\up} \circ u^{-1} \Lt (\mu \rest_{\alpha^{-1}(a) \times
        \{y_1\}}) \circ u^{-1}.
      \end{equation}
  \end{itemize}
  By a slight abuse of notation, we continue to refer to this set as \(\bb E\).
  We note that, by the construction of \(\bb W\), Eq.~\eqref{eq:imperturbable}
  is not affected by perturbation by \(\nu \in \bb W\), and
  Eq.~\eqref{eq:nice_abscont} is not affected by perturbation by \(\nu^{\low}
  \in \bb W\).

  In particular, by Lemma~\ref{lem:simple}, it suffices to show that
  \(\lambda_{\bb W}[\bb E - \mu] = 0\) for \(\mu \in \bb E\).

  Therefore, let \(\mu \in \bb E\) be arbitrary. Let the budget-exhausting
  multiple threshold policy satisfying counterfactual equalized odds over it be
  \(\tau(x)\), so that
  \begin{equation*}
    \EE_\mu[\tau(X)] = \min(b, \Pr_\mu(u(X) > 0)),
  \end{equation*}
  with thresholds \(\{t_a\}_{a \in \C A}\). We split into two cases based on
  whether \(\tau(X) = \B 1_{u(X) > 0}\) \(\mu\)-a.s.\ or not.

  In both cases, we make use of the following two useful observations.

  First, note that as \(\bb E \subseteq \bb Q\), if \(\mu + \nu\) is not a
  probability measure, then \(\mu + \nu \notin \bb E\). Therefore, without loss
  of generality, we assume throughout that \(\mu + \nu\) is a probability
  measure.

  Second, suppose \(\tau'(x)\) is a policy satisfying counterfactual equalized
  odds over some \(\nu \in \bb Q\). Then, if \(0 < \EE_\mu[\tau'(X)] < 1\), it
  follows that for all \(a \in \C A\),
  \begin{equation}
  \label{eq:all_nontrivial}
    0 < \EE_\mu[\tau'(X) \mid A = a] < 1.
  \end{equation}
  For, suppose not. Then, without loss of generality, there must be \(a_0, a_1
  \in \C A\) such that
  \begin{equation*}
    \EE_\mu[\tau'(X) \mid A = a_0] = 0
  \end{equation*}
  and
  \begin{equation*}
    \EE_\mu[\tau'(X) \mid A = a_1] > 0.
  \end{equation*}
  But then, by the law of iterated expectation, there must be some \(\C Y'
  \subseteq \C Y\) such that \(\mu[\C X \times \C Y'] > 0\) and so,
  \begin{align*}
    \B 1_{\C X \times \C Y'} \cdot \EE_\mu&[\tau'(X) \mid A = a_1, Y(1)] \\
      &> 0 \\
      &= \B 1_{\C X \times \C Y'} \cdot \EE_\mu[\tau'(X) \mid A = a_0, Y(1)],
  \end{align*}
  contradicting the fact that \(\tau'(x)\) satisfies counterfactual equalized
  odds over \(\mu\). Therefore, in what follows, we can assume that
  Eq.~\eqref{eq:all_nontrivial} holds.

  Our goal is to show that \(\lambda_{\bb W}[\bb E - \mu] = 0\).

  \begin{case}[\(\tau(X) = \B 1_{u(X) > 0}\)]
    We argue as follows. First, we show that \(\B 1_{u(X) > 0}\) is the unique
    budget-exhausting multiple threshold policy with non-negative thresholds
    over \(\mu + \nu\) for all \(\nu \in \bb W\). Then, we show that the set of
    \(\nu \in \bb W\) such that \(\B 1_{u(x) > 0}\) satisfies counterfactual
    equalized odds over \(\mu + \nu\) is a \(\lambda_{\bb W}\)-null set.

    We begin by observing that \(\bb W^{\low} \neq \{0\}\). For, if that were
    the case, then Eq.~\eqref{eq:all_nontrivial} would not hold for \(\tau(x)\).

    Next, we note that by Eq.~\eqref{eq:preserve_budget}, for any \(\nu \in \bb
    W\),
    \begin{equation*}
      \Pr_{\mu + \nu}(u(X) > 0) = \Pr_\mu(u(X) > 0)
    \end{equation*}
    and so
    \begin{equation*}
      \EE_{\mu + \nu}[\B 1_{u(X) > 0}] = \min(b, \Pr_{\mu + \nu}(u(X) > 0)).
    \end{equation*}
     If \(\tau'(x)\) is a feasible multiple threshold policy with non-negative
     thresholds and \(\tau'(X) \neq \B 1_{u(X) > 0}\) \((\mu + \nu)\)-a.s.,
     then, as a consequence,
    \begin{equation*}
      \EE_{\mu + \nu}[\tau'(X)] < \Pr_{\mu + \nu}(u(X) > 0) \leq b.
    \end{equation*}
    Therefore, it follows that \(\B 1_{u(X) > 0}\) is the unique
    budget-exhausting multiple threshold policy over \(\mu + \nu\) with
    non-negative thresholds.

    Now, note that if counterfactual equalized odds holds with decision policy
    \(\tau(x) = \B 1_{u(x) > 0}\), then, by
    Eq.~\eqref{eq:counterfactual_fairness} and Lemma~\ref{lem:cond_prob}, we
    must have that
    \begin{multline*}
      \Pr_{\mu + \nu}(u(X) > 0 \mid A = a, Y(1) = y_1) \\
        = \Pr_{\mu + \nu}(u(X) > 0 \mid A = a', Y(1) = y_1)
    \end{multline*}
    for \(a, a' \in \C A\).\footnote{%
      To ensure that both quantities are well-defined, here and throughout the
      remainder of the proof we use the fact that by
      Eqs.~\eqref{eq:preserve_budget_ext}~and~\eqref{eq:imperturbable},
      \(\Pr_{\mu + \nu}(u(X) > 0, A = a, Y(1) = y_1) > 0\).
    }

    Now, we will show that a typical element of \(\bb W\) breaks this balance
    requirement. Choose \(a^*\) such that \(\nu^{\low}_{a*} \neq 0\). Recall
    that \(\nu\) is fixed, and let \(\nu' = \nu - \beta_{a^*}^{\low} \cdot
    \nu^{\low}_{a^*}\). Let
    \begin{equation*}
      p_a = \Pr_{\mu + \nu'}(u(X) > 0 \mid A = a', Y(1) = y_1).
    \end{equation*}
    Note that it cannot be the case that \(p_a = 0\) for all \(a \in \C A\),
    since, by Eq.~\eqref{eq:imperturbable},
    \begin{equation*}
      \Pr_{\mu + \nu'}(u(X) > 0 \mid Y(1) = y_1) > 0.
    \end{equation*}
    Therefore, by the foregoing discussion, either \(p_{a^*} > 0\) or \(p_{a^*}
    = 0\) and we can choose \(a' \in \C A\) such that \(p_{a'} > 0\). Since the
    \(\nu_a^{\low}\), \(\nu_{a,i}^{\up}\) are all mutually singular, it follows
    that counterfactual equalized odds can only hold over \(\mu + \nu\) if
    \begin{equation*}
    \label{eq:pa_prime}
      p_{a'} = \Pr_{\mu + \nu}(u(X) > 0 \mid A = a^*, Y(1) = y_1).
    \end{equation*}
    Now, we observe that by Lemma~\ref{lem:cond_prob}, that
    \begin{equation*}
      \Pr_{\mu + \nu}(u(X) > 0 \mid A = a^*, Y(1) = y_1) = \frac {\eta}
      {\pi + \beta_{a^*}^{\low} \cdot \rho}
    \end{equation*}
    where
    \begin{align*}
      \eta
        &= \Pr_\mu(u(X) > 0, A = a^*, Y(1) = y_1) \\
      \pi
        &= \Pr_\mu(A = a^*, Y(1) = y_1), \\
      \rho
        &= \int_{\C K} \B 1_{A = a^*, Y(1) = y_1} \, \dx \nu_{a^*}^{\low}.
    \end{align*}
    since
    \begin{align*}
      0
        &= \int_{\C K} \B 1_{u(X) > 0, A = a^*, Y(1) = y_1} \, \dx
          \nu_{a^*}^{\low}, \\
      0
        &\neq \int_{\C K} \B 1_{A = a^*, Y(1) = y_1} \, \dx \nu_{a^*}^{\low}.
    \end{align*}
    Here, the equality follows by the fact that \(\nu^{\low}\) is supported on
    \(S^{\low}_a \times \{y_0, y_1\}\) and the inequality from
    Eq.~\eqref{eq:see_zero}.

    Therefore, if, in the first case, \(p_{a'} > 0\), then counterfactual
    equalized odds only holds if
    \begin{equation*}
      \beta_{a^*}^{\low} = \frac {e - p_{a'} \cdot \pi} {p_{a'}
      \cdot \rho},
    \end{equation*}
    since, as noted above, \(\rho \neq 0\) by Eq.~\eqref{eq:see_zero}. In the
    second case, if \(p_{a'} = 0\), then counterfactual equalized odds can only
    hold if
    \begin{equation*}
      e = p_{a^*} \cdot \pi = 0.
    \end{equation*}
    Since we chose \(a'\) so that \(p_{a^*} > 0\) if \(p_{a'} = 0\) and \(\pi >
    0\) by Eq.~\eqref{eq:imperturbable}, this is impossible.

    In either case, we see that the set of \(\beta_{a^*}^{\low} \in \B R\) such
    that there a budget-exhausting threshold policy with positive thresholds
    satisfying counterfactual equalized odds over \(\mu + \nu' +
    \beta_{a^*}^{\low} \cdot \nu_{a^*}^{\low}\) has \(\lambda\)-measure zero.
    That is
    \begin{equation*}
      \lambda_{\Span(\nu_{a^*}^{\low})}[\bb E - \mu - \nu'] = 0.
    \end{equation*}
    Since \(\nu'\) was arbitrary, it follows by Fubini's theorem that
    \(\lambda_{\bb W}[\bb E - \mu] = 0\).
  \end{case}

  \begin{case}[\(\tau(X) \neq \B 1_{u(X) > 0}\)]
    Our proof strategy is similar to the previous case. First, we show that, for
    a given fixed \(\nu^{\low} \in \bb W^{\low}\), there is a unique candidate
    policy \(\tilde \tau(x)\) for being a budget-exhausting multiple threshold
    policy with non-negative thresholds and satisfying counterfactual equalized
    odds over \(\mu + \nu^{\low} + \nu^{\up}\) for any \(\nu^{\up} \in \bb
    W^{\up}\). Then, we show that the set of \(\nu^{\up}\) such that \(\tilde
    \tau(X)\) satisfies counterfactual equalized odds has \(\lambda_{\bb
    W^{\up}}\) measure zero. Finally, we argue that this in turn implies that
    the set of \(\nu \in \bb W\) such that there exists a Pareto efficient
    policy satisfying counterfactual equalized odds over \(\mu + \nu\) has
    \(\lambda_{\bb W}\)-measure zero.

    We seek to show that \(\lambda_{\bb W^{\up}}[\bb E - (\mu + \nu^{\low})] =
    0\). To begin, we note that since \(\nu_{a,i}^{\up}\) concentrates on
    \(\{y_1\} \times \C X\) for all \(a \in \C A\), it follows that
    \begin{multline*}
      \EE_{\mu + \nu^{\low}}[d(X) \mid A = a, Y(1) = y_0] \\
        = \EE_{\mu + \nu^{\low} + \nu^{\up}}[d(X) \mid A = a, Y(1) = y_0]
    \end{multline*}
    for any \(\nu^{\up} \in \bb W^{\up}\).

    Now, suppose there exists some \(\nu^{\up} \in \bb W^{\up}\) such that there
    exists a budget-exhausting multiple threshold policy \(\tilde \tau(x)\) with
    non-negative thresholds such that counterfactual equalized odds is satisfied
    over \(\mu + \nu^{\low} + \nu^{\up}\). (If not, then we are done and
    \(\lambda_{\bb W^{\up}}[\bb E - (\mu + \nu^{\low})] = 0\), as the measure of
    the empty set is zero.) Let
    \begin{equation*}
      p = \EE_{\mu + \nu^{\low}}[\tilde \tau(X) \mid A = a, Y(1) = y_0].
    \end{equation*}
    Suppose that \(\tilde \tau'(x)\) is an alternative budget-exhausting
    multiple threshold policy with non-negative thresholds such that
    counterfactual equalized odds is satisfied. We seek to show that \(\tau'(X)
    = \tau(X)\) \((\mu + \nu^{\low} + \nu^{\up})\)-a.e.\ for any \(\nu^{\up} \in
    \bb W^{\up}\). Toward a contradiction, suppose that for some \(a_0 \in \C
    A\),
    \begin{equation*}
      \EE_{\mu + \nu^{\low}}[\tilde \tau'(X) \mid A = a_0, Y(1) = y_0] < p.
    \end{equation*}
    Since, by Eq.~\eqref{eq:imperturbable}, \(\Pr_{\mu + \nu^{\low}}(A = a_0,
    Y(1) = y_0) > 0\), it follows that
    \begin{equation*}
      \EE_{\mu + \nu^{\low}}[\tilde \tau'(X) \mid A = a_0] < \EE_{\mu +
      \nu^{\low}}[\tilde \tau(X) \mid A = a_0].
    \end{equation*}
    Therefore, since \(\tilde \tau(x)'\) is budget exhausting, there must be
    some \(a_1\) such that
    \begin{equation*}
      \EE_{\mu + \nu^{\low}}[\tilde \tau'(X) \mid A = a_1] > \EE_{\mu +
      \nu^{\low}}[\tilde \tau(X) \mid A = a_1].
    \end{equation*}
    From this, it follows \(\tilde \tau'(x)\) can be represented by a threshold
    greater than or equal to that of \(\tilde \tau(x)\) on \(\alpha^{-1}(a_1)\),
    and hence
    \begin{align*}
      \EE_{\mu + \nu^{\low}}&[\tilde \tau'(X) \mid A = a_1, Y(1) = y_0] \\
        &\geq \EE_{\mu + \nu^{\low}}[\tilde \tau(X) \mid A = a_0, Y(1) = y_0] \\
        &= p \\
        &> \EE_{\mu + \nu^{\low}}[\tilde \tau'(X) \mid A = a_0, Y(1) = y_0],
    \end{align*}
    contradicting the fact that \(\tilde \tau'(x)\) satisfies counterfactual
    equalized odds.

    By the preceding discussion, Lemma~\ref{lem:simplethresh}, and the fact that
    \(\nu^{\low}\) is supported on \(u^{-1}((-\infty, 0])\),
    \begin{equation*}
      \tilde \tau(X) = \tilde \tau'(X) \quad (\mu \rest_{\C X \times
      \{y_0\}})\text{-a.e.}
    \end{equation*}
    By Eq.~\eqref{eq:nice_abscont}, it follows that \(\tilde \tau(X) = \tilde
    \tau'(X)\) \(\nu^{\up}_{a,i}\)-a.e.\ for \(i = 0, 1\). As a consequence,
    \begin{equation*}
      \tilde \tau(X) = \tilde \tau'(X) \quad (\mu + \nu^{\low} +
      \nu^{up})\text{-a.e.}
    \end{equation*}
    for all \(\nu^{\up} \in \bb W^{\up}\). Therefore \(\tilde \tau(X)\) is,
    indeed, unique, as desired.

    Now, we note that since \(\tau(X) \neq \B 1_{u(X) > 0}\), it follows that
    \(\EE[\tau(X)] < \Pr_\mu(u(X) > 0)\). It follows that \(\EE_\mu[\tau(X)] =
    b\), since \(\tau(x)\) is budget exhausting. Therefore, by
    Eq.~\eqref{eq:preserve_budget}, it follows that for any budget-exhausting
    policy \(\tilde \tau(X)\), \(\EE[\tilde \tau(X)] = b\), and so \(\tilde
    \tau(X) \neq \B 1_{u(X) > 0}\) over \(\mu + \nu\).

    Therefore, fix \(\nu^{\low}\) and \(\tilde \tau(X)\). By
    Eq.~\eqref{eq:all_nontrivial}, there is some \(a^*\) such that
    \begin{equation*}
      0 < \Pr_{\mu + \nu^{\low}}(u(X) > \tilde t_{a^*} \mid A = a^*) < 1.
    \end{equation*}
    Then, it follows by Eq.~\eqref{eq:see_threshold} that
    \begin{equation*}
      \int_{\C K} \B 1_{u(X) > \tilde t_{a^*}} \, \dx \nu^{\up}_{a^*} \neq
      0.
    \end{equation*}
    Fix \(\nu' = \nu - \beta_{a^*}^{\up} \cdot \nu_{a^*}^{\up}\). Then, for some
    \(a \neq a^*\), set
    \begin{equation*}
      p^* = \EE_{\mu + \nu'}[\tilde \tau(X) \mid A = a, Y(1) = y_1].
    \end{equation*}
    Since the \(\nu_a^{\low}\), \(\nu_a^{\up}\) are all mutually singular, it
    follows that counterfactual equalized odds can only hold over \(\mu + \nu\)
    if
    \begin{equation*}
      p^* = \Pr_{\mu + \nu}(u(X) > \tilde t_{a^*} \mid A = a^*, Y(1) = y_1).
    \end{equation*}
    Now, we observe that by Lemma~\ref{lem:cond_prob}, that
    \begin{align}
    \label{eq:case_2_exp}
    \begin{split}
      \Pr_{\mu + \nu}(u(X) > \tilde t_{a^*} &\mid A = a^*, Y(1) = y_1) \\
        &= \frac {\eta + \beta_a^{\up} \cdot \gamma} \pi
    \end{split}
    \end{align}
    where
    \begin{align*}
      \eta
        &= \Pr_{\mu + \nu^{\low}}(u(X) > \tilde t_{a^*} \mid A = a^*, Y(1) =
          y_1), \\
      \pi
        &= \Pr_{\mu + \nu^{\low}}(A = a^*, Y(1) = y_1), \\
      \gamma
        &= \int_{\C K} \B 1_{u(X) > \tilde t_{a^*}, A = a^*, Y(1) = y_1} \, \dx
          \nu_a^{\up},
    \end{align*}
    and we note that
    \begin{equation*}
      0 = \int_{\C K} \B 1_{A = a^*, Y(1) = y_1} \, \dx \nu_a^{\low}.
    \end{equation*}
    Eq.~\eqref{eq:case_2_exp} can be rearranged to
    \begin{equation*}
      (p^* \cdot \pi - \eta) - \beta \cdot \gamma = 0.
    \end{equation*}
    This can only hold if
    \begin{equation*}
      \beta = \frac {p^* \cdot \pi - \eta} \gamma,
    \end{equation*}
    since by Eq.~\eqref{eq:see_threshold}, \(\gamma \neq 0\). Since any
    countable subset of \(\B R\) is a \(\lambda\)-null set,
    \begin{equation*}
      \lambda_{\Span(\nu_{a^*}^{\up})}[\bb E - \mu - \nu'] = 0.
    \end{equation*}
    Since \(\nu'\) was arbitrary, it follows by Fubini's theorem that
    \(\lambda_{\bb W^{\up}}[\bb E - \mu - \nu^{\low}] = 0\) in this case as
    well. Lastly, since \(\nu^{\low}\) was also arbitrary, applying Fubini's
    theorem a final time gives that \(\lambda_{\bb W}[\bb E - \mu] = 0\).
  \end{case}

  \paragraph{Conditional principal fairness and path-specific fairness}

  The extension of these results to conditional principal fairness and
  path-specific fairness is straightforward. All that is required is a minor
  modification of the probe.

  In the case of conditional principal fairness, we set
  \begin{align*}
    \nu_{a,w}^{\up}[E]
      &= \mu_{\max,a,w} \circ (\gamma_{(y_1,y_1)} \circ \pi_{\C X})^{-1}[E \cap
        u^{-1}(S^{\up}_{a,1})], \\
      &\hspace{0.4cm}- \mu_{\max,a} \circ (\gamma_{(y_1,y_1)} \circ \pi_{\C
        X})^{-1}[E \cap u^{-1}(S^{\up}_{a,w})], \\
    \nu_{a,w}^{\low}[E]
      &= \mu_{\max,a,w} \circ (\gamma_{(y_1,y_1)} \circ \pi_{\C X})^{-1}[E \cap
        u^{-1}(S^{\low}_a)] \\
      &\hspace{0.4cm}- \mu_{\max,a} \circ (\gamma_{(y_0,y_0)} \circ \pi_{\C
        X})^{-1}[E \cap u^{-1}(S^{\low}_{a,w})],
  \end{align*}
  where \(\gamma_{(y,y')} : \C X \to \C K\) is the injection \(x \mapsto (x, y,
  y')\).
  Our probe is then given by
  \begin{align*}
    \bb W^{\up}
    &= \Span(\nu_{a,w}^{\up}), \\
    \bb W^{\low}
    &= \Span(\nu_{a,w}^{\low}),
  \end{align*}
  almost as before.

  The proof otherwise proceeds virtually identically, except for two points.
  First, recalling Remark~\ref{rmk:dist_hyp}, we use the fact that a generic
  element of \(\bb Q\) satisfies \(\Pr_\mu(A = a, W = w) > 0\) in place of
  \(\Pr_\mu(A = a) > 0\) throughout. Second, we use the fact that \(\omega\)
  overlaps utility in place of Eq.~\eqref{eq:all_nontrivial}. In particular, If
  \(\omega\) does not overlap utilities for a generic \(\mu \in \bb Q\), then,
  by Lemma~\ref{lem:overlap}, there exists \(w \in \C W\) such that
  \(\Pr_\mu(u(X) > 0, W = w) = 0\) for all \(\mu \in \bb Q\). If this occurs, we
  can show that no budget-exhausting multiple threshold policy with positive
  thresholds satisfies conditional principal fairness, exactly as we did to show
  Eq.~\eqref{eq:all_nontrivial}.

  In the case of path-specific fairness, we instead define
  \begin{align*}
    S_{a,w}^{\low}
    &= S_{a,w} \cap (-\infty, r_{a,w}), \\
    S_{a,w}^{\up}
    &= S_{a,w} \cap [r_{a,w}, \infty),
  \end{align*}
  where \(r_{a,w}\) is chosen so that
  \begin{equation*}
    \Pr_{\mu_{\max,a,w}}(u(X) \in S_{a,w}^{\low}) = \Pr_{\mu_{\max,a,w}}(u(X)
    \in S_{a,w}^{\up}).
  \end{equation*}
  Let \(\pi_X\) denote the projection from \(\C K = \C A \times \C X^{\C A}\)
  given by
  \begin{equation*}
    \left( a, (x_{a'})_{a' \in \C A} \right) \mapsto x_a.
  \end{equation*}
  Let \(\pi_{a'}\) denote the projection from the \(a'\)-th component. (That is,
  given \(\mu \in \bb K\), the distribution of \(X_{\Pi, A, a'}\) over \(\mu\)
  is given by \(\mu \circ \pi^{-1}_{a'}\) and the distribution of \(X\) is given
  by \(\mu \circ \pi^{-1}_X\).)
  Then, we let \(\tilde \mu_{\max,a,w}\) be the measure on \(\C X\)
  given by
  \begin{align*}
    \tilde \mu_{\max,a,w}[E]
    &= \mu_{\max,a,w}[E \cap (u \circ \pi_a)^{-1}(s_{a,w}^{\up})] \\
    &\hspace{1cm} - \mu_{\max,a,w}[E \cap (u \circ \pi_a)^{-1}(S_{a,w}^{\low})].
  \end{align*}
  Finally, let \(\phi : \C A \to \C A\) be a
  permutation of the groups with no fixed points, i.e., so that \(a' \neq
  \phi(a')\) for all \(a' \in \C A\). Then, we define
  \begin{equation*}
    \nu_{a'} = \delta_{a'} \times \tilde \mu_{\max,\phi(a'),w_1} \times \prod_{a
    \neq \phi(a')}
    \mu_{\max,a,w_1} \circ \pi_a^{-1},
  \end{equation*}
  where \(\delta_a\) is the measure on \(\C A\) given by \(\delta_a[\{a'\}] =
  \B 1_{a = a'}\). Then, simply let
  \begin{equation*}
    \bb W = \Span(\nu_a')_{a'\in \C A}.
  \end{equation*}

  Since \(\tilde \mu_{\max,a,w}[\C X] = 0\) for all \(a \in \C A\), it follows
  that \(\nu_{a,w} \circ \pi_X^{-1} = 0\), i.e.,
  \begin{equation*}
    \Pr_\mu(X \in E) = \Pr_{\mu + \nu}(X \in E)
  \end{equation*}
  for any \(\nu \in \bb W\) and \(\mu \in \bb Q\). Therefore
  Eqs.~\eqref{eq:preserve_prob}~and~\eqref{eq:preserve_budget} hold. Moreover,
  the \(\nu_a\) satisfy the following strengthening of
  Eq.~\eqref{eq:see_threshold}. Perturbations in \(\bb W\) have the property
  that for any non-trivial \(t\)---not necessarily positive---some of the mass
  of \(u(X_{\Pi, A, a})\) is moved either above or below \(t\). More precisely,
  for any \(\mu \in \bb Q\) and any \(t\) such that
  \begin{equation*}
    0 < \Pr_\mu(u(X) > t \mid A = a) < 1,
  \end{equation*}
  if \(\nu \in \bb W\) is such that \(\Pr_{|\nu|}(A = \phi^{-1}(a)) > 0\), then
  \begin{equation}
    \int_{\C K} \B 1_{u(X_{\Pi, A, a}) > t} \, \dx \nu_a \neq 0.
  \end{equation}
  This stronger property means that we need not separately treat the case where
  \(\tau(X) = \B 1_{u(X) > 0}\) \(\mu\)-a.e.

  Other than this difference the proof proceeds in the same way, except for two
  points. First, we again make use of the fact that \(\omega\) can be assumed to
  overlap utilities in place of Eq.~\eqref{eq:all_nontrivial}, as in the case of
  conditional principal fairness. Second, \(w_0\) and \(w_1\) take the place of
  \(y_0\) and \(y_1\). In particular, to establish the uniqueness of \(\tilde
  \tau(x)\) given \(\mu\) and \(\nu^{\low}\) in the second case, instead of
  conditioning on \(y_0\), we instead condition on \(w_0\), where, following the
  discussion in Remark~\ref{rmk:dist_hyp} and Lemma~\ref{lem:condition}, this
  conditioning is well-defined for a generic element of \(\bb Q\).
\end{proof}

We have focused on causal definitions of fairness, but the thrust of our
analysis applies to non-causal conceptions of fairness as well. Below we show
that policies constrained to satisfy (non-counterfactual) equalized
odds~\cite{hardt2016equality} are generically strongly Pareto dominated, a
result that follows immediately from our proof above.

\begin{defn}
\label{defn:eo}
  \emph{Equalized odds} holds for a decision policy \(d(x)\) when
  \begin{equation}
  \label{eq:eo}
    d(X) \indep A \mid Y.
  \end{equation}
\end{defn}

We note that \(Y\) in Eq.~\eqref{eq:eo} does \emph{not} depend on our choice of
\(d(x)\), but rather represents the realized value of \(Y\), e.g., under some
\emph{status quo} decision making rule.

\begin{cor}
\label{cor:eo}
  Suppose \(\C U\) is a set of utilities consistent modulo \(\alpha\). Further
  suppose that  for all \(a \in \C A\) there exist a \(\C U\)-fine distribution
  of \(X\) and a utility \(u \in \C U\) such that \(\Pr(u(X) > 0, A = a) > 0\),
  where \(A = \alpha(X)\). Then, for almost every \(\C U\)-fine distribution of
  \(X\) and \(Y\) on \(\C X \times \C Y\), any decision policy \(d(x)\)
  satisfying equalized odds is strongly Pareto dominated.
\end{cor}

\begin{proof}
  Consider the following maps. Distributions of \(X\) and \(Y(1)\), i.e.,
  probability measures on \(\C X \times \C Y\), can be embedded in the space of
  joint distributions on \(X\), \(Y(0)\), and \(Y(1)\) via pushing forward by
  the map \(\iota\), where \(\iota : (x, y) \mapsto (x, y, y)\). Likewise, given
  a fixed decision policy \(D = d(X)\), joint distributions of \(X\), \(Y(0)\),
  and \(Y(1)\) can be projected onto the space of joint distributions of \(X\)
  and \(Y\) by pushing forward by the map \(\pi_d : (x, y_0, y_1) \mapsto (x,
  y_{d(x)})\). Lastly, we see that the composition of \(\iota\) and
  \(\pi_d\)---regardless of our choice of \(d(x)\)---is the identity, as shown
  in the diagram below.
  \begin{equation*}
    \begin{tikzpicture}[xscale = 3, yscale = 2]
      \node at (0,1) (A) {\(\C X \times \C Y\)};
      \node at (1,1) (B) {\(\C X \times \C Y \times \C Y\)};
      \node at (1,0) (C) {\(\C X \times \C Y\)};
      \draw[->] (A) to node[above] {\(\iota\)} (B);
      \draw[->] (B) to node[left] {\(\pi_d\)} (C);
      \draw[->] (A) to node[below left] {\(\id\)} (C);
    \end{tikzpicture}
  \end{equation*}
  We note also that counterfactual equalized odds holds for \(\mu\) exactly when
  equalized odds holds for \(\mu \circ (\pi_d \circ \iota)^{-1}\). The result
  follows immediately from this and Theorem~\ref{thm:dist}.
\end{proof}

\subsection{General Measures on \(\C K\)}
\label{sec:counterexample}

Theorem~\ref{thm:dist} is restricted to \(\C U\)-fine and \(\C U^{\C A}\)-fine
distributions on the state space. The reason for this restriction is that when
the distribution of \(X\) induces atoms on the utility scale, threshold policies
can possess additional---or even infinite---degrees of freedom
when the threshold falls exactly on an atom. In particular circumstances, these
degrees of freedom can be used to ensure causal fairness notions, such as
counterfactual equalized odds, hold in a locally robust way. In particular, the
generalization of Theorem~\ref{thm:dist} beyond \(\C U\)-fine measures to all
totally bounded measures on the state space is false, as illustrated by the
following proposition.

\begin{prop}
\label{prop:counterexample}
  Consider the set \(\bb E' \subseteq \B K\) of distributions---not necessarily
  \(\C U\)-fine---on \(\C K = \C X \times \C Y\) over which there exists a
  Pareto efficient policy satisfying counterfactual equalized odds. There exist
  \(b\), \(\C X\), \(\C Y\), and \(\C U\) such that \(\bb E'\) is \emph{not}
  relatively shy.
\end{prop}

\begin{proof}
  We adopt the notational conventions of Section~\ref{sec:shyness-prelims}. We
  note that by Prop.~\ref{prop:shy_axioms_rel}, a set can only be shy if it has
  empty interior. Therefore, we will construct an example in which an open ball
  of distributions on \(\C K\) in the total variation norm all allow for a
  Pareto efficient policy satisfying counterfactual equalized odds, i.e., are
  contained in \(\bb E'\).

  Let \(b = \tfrac 3 4\), \(\C Y = \{0, 1\}\), \(\C A = \{a_0, a_1\}\), and \(\C
  X = \{0, 1\} \times \{a_0, a_1\} \times \B R\). Let \(\alpha : \C X \to \C A\)
  be given by \(\alpha : (y, a, v) \mapsto a\) for arbitrary \((y, a, v) \in \C
  X\). Likewise, let \(u : \C X \to \B R\) be given by \(u : (y, a, v) \mapsto
  v\). Then, if \(\C U = \{u\}\), \(\C U\) is vacuously consistent modulo
  \(\alpha\). Consider the joint distribution \(\mu\) on \(\C K = \C X \times \C
  Y\) where for all \(y, y' \in \C Y\), \(a \in \C A\), and \(u \in \B R\),
  \begin{multline*}
    \Pr_\mu(X = (a, y, u), Y(1) = y') \\
      = \frac 1 4 \cdot \B 1_{y = y'} \cdot \Pr_\mu(u(X) = u),
  \end{multline*}
  where, over \(\mu\), \(u(X)\) is distributed as a \(\tfrac 1 2\)-mixture of
  \(\unif(1, 2)\) and \(\delta(1)\); that is, \(\Pr(u(X) = 1) = \tfrac 1 2\) and
  \(\Pr(a < u(X) < b) = \tfrac {b - a}  2\) for \(0 \leq a \leq b < 1\).

  We first observe that there exists a Pareto efficient threshold policy
  \(\tau(x)\) such that counterfactual equalized odds is satisfied with respect
  to the decision policy \(\tau(X)\). Namely, let
  \begin{equation*}
    \tau(a, y, u) =
      \begin{cases}
        1       &   u > 1, \\
        \frac 1 2   &   u = 1, \\
        0       &   u < 1.
      \end{cases}
  \end{equation*}
  Then, it immediately follows that \(\EE[\tau(X)] = \tfrac 3 4 = b\). Since
  \(\tau(x)\) is a threshold policy and exhausts the budget, it is utility
  maximizing by Lemma~\ref{lem:maximize}. Moreover, if \(D = \B 1_{U_D \leq
  \tau(X)}\) for some \(U_D \sim \unif(0, 1)\) independent of \(X\) and
  \(Y(1)\), then \(D \indep A \mid Y(1)\). Since \(u(X) \indep A, Y(1)\), it
  follows that
  \begin{align*}
    \Pr_\mu(D = 1 &\mid A = a, Y(1) = y) \\
      &\hspace{0.5cm}= \Pr(U_D \leq \tau(X) \mid A = a, Y(1) = y) \\
      &\hspace{0.5cm}= \Pr(U_D \leq \tau(X)) \\
      &\hspace{0.5cm}= \EE_\mu[\tau(X)],
  \end{align*}
  Therefore Eq.~\eqref{eq:counterfactual_equalized_odds} is satisfied, i.e.,
  counterfactual equalized odds holds. Now, using \(\mu\), we construct an open
  ball of distributions over which we can construct similar threshold policies.
  In particular, suppose \(\mu'\) is any distribution such that \(|\mu -
  \mu'|[\C K] < \tfrac 1 {64}\). Then, we claim that there exists a
  budget-exhausting threshold policy satisfying counterfactual equalized odds
  over \(\mu'\). For, we note that
  \begin{align*}
    \Pr_{\mu'}(U > 1) < \Pr_\mu(U > 1) + \frac 1 {64} = \frac {33} {64}, \\
    \Pr_{\mu'}(U \geq 1) > \Pr_\mu(U \geq 1) - \frac 1 {64} = \frac {63} {64},
  \end{align*}
  and so any threshold policy \(\tau'(x)\) satisfying \(\EE[\tau'(X)] = b =
  \tfrac 3 4\) must have \(t = 1\) as its threshold.

  We will now construct a threshold policy \(\tau'(x)\) satisfying
  counterfactual equalized odds over \(\mu'\). Consider a threshold policy of
  the form
  \begin{equation*}
    \tau'(a, y, u) =
      \begin{cases}
        1     & u > 1, \\
        p_{a,y} & u = 1, \\
        0     & u < 1.
      \end{cases}
  \end{equation*}

  For notational simplicity, let
  \begin{align*}
    q_{a,y}
      &= \Pr_{\mu'}(A = a, Y = y, U > 1), \\
    r_{a,y}
      &= \Pr_{\mu'}(A = a, Y = y, U = 1), \\
    \pi_{a,y}
      &= \Pr_{\mu'}(A = a, Y = y).
  \end{align*}
  Then, we have that
  \begin{align*}
    \EE_{\mu'}[\tau'(X)]
      &= \sum_{a,y} q_{a,y} + p_{a,y} \cdot r_{a,y}, \\
    \EE_{\mu'}[\tau'(X) \mid A = a, Y = y]
      &= \frac {q_{a,y} + p_{a,y} \cdot r_{a,y}} {\pi_{a,y}}.
  \end{align*}
  Therefore, the policy will be budget exhausting if
  \begin{equation*}
    \sum_{a,y} q_{a,y} + p_{a,y} \cdot r_{a,y} = \tfrac 3 4,
  \end{equation*}
  and it will satisfy counterfactual equalized odds if
  \begin{align}
  \label{eq:system}
  \begin{split}
    \pi_{a_1, 0} \cdot (q_{a_0, 0} &+ p_{a_0, 0} \cdot r_{a_0, 0}) \\
      &= \pi_{a_0, 0} \cdot (q_{a_1, 0} + p_{a_1, 0} \cdot r_{a_1, 0}), \\
    \pi_{a_1, 1} \cdot (q_{a_0, 1} &+ p_{a_0, 1} \cdot r_{a_0, 1}) \\
      &= \pi_{a_0, 1} \cdot (q_{a_1, 1} + p_{a_1, 1} \cdot r_{a_1, 1}),
  \end{split}
  \end{align}
  since, as above,
  \begin{multline*}
    \Pr(D = 1 \mid A = a, Y(1) = y) \\
      = \EE[\tau'(X) \mid A = a, Y(1) = y].
  \end{multline*}

  Again, for notational simplicity, let
  \begin{equation*}
    S = \frac {\tfrac 3 4 - \Pr_{\mu'}(U > 1)} {\Pr_{\mu'}(U = 1)}.
  \end{equation*}
  Then, a straightforward algebraic manipulation shows that
  Eq.~\eqref{eq:system} is solved by setting \(p_{a_0, y}\) to be
  \begin{equation*}
    \frac {S \cdot \pi_{a_0, y} \cdot (r_{a_0, y} + r_{a_1, y}) + \pi_{a_0, y}
    \cdot q_{a_1, y} - \pi_{a_1, y} \cdot q_{a_0, y}} {r_{a_0, y} \cdot
    (\pi_{a_0, y} + \pi_{a_1, y})},
  \end{equation*}
  and \(p_{a_1, y}\) to be
  \begin{equation*}
    \frac {S \cdot \pi_{a_1, y} \cdot (r_{a_0, y} + r_{a_1, y}) + \pi_{a_1, y}
    \cdot q_{a_0, y} - \pi_{a_0, y} \cdot q_{a_1, y}} {r_{a_1, y} \cdot
    (\pi_{a_0, y} + \pi_{a_1, y})}.
  \end{equation*}
  In order for \(\tau'(x)\) to be a well-defined policy, we need to show that
  \(p_{a,y} \in [0,1]\) for all \(a \in \C A\) and \(y \in \C Y\). To that end,
  note that
  \begin{align*}
    q_{a,y}
      &= \Pr_{\mu'}(A = a, Y = y, U > 1), \\
    r_{a,y}
      &= \Pr_{\mu'}(A = a, Y = y, U = 1), \\
    \pi_{a,y}
      &= \Pr_{\mu'}(A = a, Y = y), \\
    r_{a_0, y} + r_{a_1, y}
      &= \Pr_{\mu'}(Y = y, U = 1), \\
    \pi_{a_0,y} + \pi_{a_1,y}
      &= \Pr_{\mu'}(Y = y), \\
    S
      &= \frac {\frac 3 4 - \Pr_{\mu'}(U > 1)} {\Pr_{\mu'}(U = 1)}. \\
  \end{align*}
  Now, we recall that \(|\Pr_{\mu'}(E) - \Pr_\mu(E)| < \frac 1 {64}\) for any
  event \(E\) by hypothesis. Therefore,
  \begin{alignat*}{3}
    \frac {7} {64} \leq\ 
      & \makebox[\myl]{\(q_{a,y}\)}
      & \ \leq \frac {9} {64}, \\
    \frac 7 {64} \leq\ 
      & \makebox[\myl]{\(r_{a,y}\)}
      & \ \leq \frac 9 {64}, \\
    \frac 7 {64} \leq\ 
      & \makebox[\myl]{\(\pi_{a,y}\)}
      & \ \leq \frac {17} {64}, \\
    \frac {15} {64} \leq\ 
      & \makebox[\myl]{\(r_{a_0, y} + r_{a_1, y}\)}
      & \ \leq \frac {17} {64}, \\
    \frac {31} {64} \leq\ 
      & \pi_{a_0,y} + \pi_{a_1,y}
      & \ \leq \frac {33} {64}, \\
    \frac {15} {31} \leq\ 
      & \makebox[\myl]{\(S\)}
      & \ \leq \frac {17} {33}.
   \end{alignat*}
  Using these bounds and the expressions for \(p_{a,y}\) derived above, we see
  that
  \begin{align*}
    \frac {629} {3069} < p_{a,y} < \frac {6497} {7161},
  \end{align*}
  and hence \(p_{a,y} \in [0, 1]\) for all \(a \in \C A\) and \(y \in \C Y\).

  Therefore, the policy \(\tau'(x)\) is well-defined, and, by construction, is
  budget-exhausting and therefore utility-maximizing by
  Lemma~\ref{lem:maximize}. It also satisfies counterfactual equalized odds by
  construction. Since \(\mu'\) was arbitrary, it follows that the set of
  distributions on \(\C K\) such that there exists a Pareto efficient policy
  satisfying counterfactual equalized odds contains an open ball, and hence is
  not shy.
\end{proof}

\section{Theorem~\ref{thm:path_specific} and Related Results}

We first prove a variant of Theorem~\ref{thm:path_specific} for general,
continuous covariates \(\C X\). Then, we extend and generalize
Theorem~\ref{thm:path_specific} using the theory of finite Markov chains,
offering a proof of the theorem different from the sketch included in the main
text.

\subsection{Extension to Continuous Covariates}

Here we follow the proof sketch in the main text for
Theorem~\ref{thm:path_specific}, which assumes a finite covariate-space \(\C
X\). In that case, we start with a point \(x^*\) with maximum decision
probability \(d(x^*)\), and then assume, toward a contradiction, that there
exists a point with strictly lower decision probability. The general case is
more involved since it is not immediately clear that the maximum value of
\(d(x)\) is achieved with positive probability  in \(\C X\). We start with the
lemma below before proving the main result.

\begin{lem}
\label{lem:d_pi}
  A decision policy \(d(x)\) satisfies path-specific fairness with \(W = X\) if
  and only if any \(a' \in \C A\),
  \begin{equation*}
    \EE[d(X_{\Pi, A, a'}) \mid X] = d(X).
  \end{equation*}
\end{lem}

\begin{proof}
  First, suppose that \(d(x)\) satisfies path-specific fairness. To show the
  result, we use the standard fact that for independent random variables \(X\)
  and \(U\),
  \begin{equation}
  \label{eq:efxu}
    \EE[f(X,U) \mid X] = \int f(X, u) \diff F_U(u),
  \end{equation}
  where \(F_U\) is the distribution of \(U\). \citep[For a proof of this fact
  see, for example,][]{drhab2019conditional}

  Now, we have that
  \begin{align*}
    \EE[D_{\Pi, A, a'} \mid X_{\Pi, A, a'}]
    &= \EE[\mathbb{1}_{U_D \leq d(X_{\Pi, A, a'})} \mid X_{\Pi, A, a'}] \\
    &= \int_0^1 \mathbb{1}_{u \leq d(X_{\Pi, A, a'})} \diff u \\
    &= d(X_{\Pi, A, a'}),
  \end{align*}
  where the first equality follows from the definition of \(D_{\Pi, A, a'}\),
  and the second from Eq.~\eqref{eq:efxu}, since the exogenous variable \(U_D
  \sim \unif(0,1)\) is independent of the counterfactual covariates \(X_{\Pi, A,
  a'}\). An analogous argument shows that \(\EE[D \mid X] = d(X)\).

  Finally, conditioning on \(X\), we have
  \begin{align*}
    \EE[d(X_{\Pi, A, a'}) \mid X]
      &= \EE[\EE[D_{\Pi, A, a'} \mid X_{\Pi, A, a'}] \mid X] \\
      &= \EE[\EE[D_{\Pi, A, a'} \mid X_{\Pi, A, a'}, X] \mid X] \\
      &= \EE[D_{\Pi, A, a'} \mid X] \\
      &= \EE[D \mid X] \\
      &= d(X),
  \end{align*}
  where the second equality follows from the fact that \(D_{\Pi, A, a'} \indep X
  \mid X_{\Pi, A, a'}\), the third from the law of iterated expectations, and
  the fourth from the definition of path-specific fairness.

  Next, suppose that
  \begin{equation*}
    \EE[d(X_{\Pi, A, a'} \mid X] = d(X)
  \end{equation*}
  for all \(a' \in \C A\). Then, since \(W = X\) and \(X \indep U_D\), using
  Eq.~\eqref{eq:efxu}, we have that for all \(a' \in \C A\),
  \begin{align*}
    \EE[D_{\Pi, A, a'} \mid X]
      &= \EE[\EE[\B 1_{U_D \leq d(X_{\Pi, A, a'})} \mid X_{\Pi, A, a'}, X] \mid
        X] \\
      &= \EE[\EE[d(X_{\Pi, A, a'}) \mid X_{\Pi, A, a'}, X] \mid X] \\
      &= \EE[d(X_{\Pi, A, a'}) \mid X] \\
      &= d(X) \\
      &= \EE[d(X) \mid X] \\
      &= \EE[D \mid X].
  \end{align*}
  This is exactly Eq.~\eqref{eq:path_specific_fairness}, and so the result
  follows.
\end{proof}

We are now ready to prove a continuous variant of
Theorem~\ref{thm:path_specific}. The technical hypotheses of the theorem ensure
that the conditional probability measures \(\Pr(E \mid X)\) are ``sufficiently''
mutually non-singular distributions on \(\C X\) with respect to the distribution
of \(X\)---for example, the conditions ensure that the conditional distribution
of \(X_{\Pi, A, a} \mid X\) does not have atoms that \(X\) itself does not have,
and \emph{vice versa}. For notational and conceptual simplicity, we only
consider the case of trivial \(\zeta\), i.e., where \(\zeta(x) = \zeta(x')\) for
all \(x, x' \in \C X\).

\begin{prop}
\label{prop:path_specific}
  Suppose that
  \begin{enumerate}
    \item For all \(a \in \C A\) and any event \(S\) satisfying \(\Pr(X \in S
      \mid A = a) > 0\), we have, a.s.,
      \begin{equation*}
        \Pr(X_{\Pi, A, a} \in S \lor A = a \mid X) > 0.
      \end{equation*}
    \item For all \(a \in \C A\) and \(\epsilon > 0\), there exists \(\delta >
      0\) such that for any event \(S\) satisfying \(\Pr(X \in S \mid A = a) <
      \delta\), we have, a.s.,
      \begin{equation*}
        \Pr(X_{\Pi, A, a} \in S, A \neq a \mid X) < \epsilon.
      \end{equation*}
  \end{enumerate}
  Then, for \(W = X\), any \(\Pi\)-fair policy \(d(x)\) is constant a.s. (i.e.,
  \(d(X) = p\) a.s.\ for some \(0 \leq p \leq 1\)).
\end{prop}

\begin{proof}
  Let \(d_{\max} = \|d(x)\|_\infty\), the essential supremum of \(d\). To
  establish the theorem statement, we show that \(\Pr(d(X) = d_{\max} \mid A
  = a) = 1\) for all \(a \in \C A\). To do that, we begin by showing that there
  exists some \(a \in \C A\) such that \(\Pr(d(X) = d_{\max} \mid A = a) > 0\).

  Assume, toward a contradiction, that for all \(a \in \C A\),
  \begin{equation}
  \label{eq:contradiction}
    \Pr(d(X) = d_{\max} \mid A = a) = 0.
  \end{equation}
  Because \(\C A\) is finite, there must be some \(a_0 \in \C A\) such that
  \begin{equation}
  \label{eq:dmax}
    \Pr(d_{\max} - d(X) < \epsilon \mid A = a_0) > 0
  \end{equation}
  for all \(\epsilon > 0\).

  Choose \(a_1 \neq a_0\). We show that for values of \(x\) such that \(d(x)\)
  is close to \(d_{\max}\), the distribution of \(d(X_{\Pi, A, a_1}) \mid X =
  x\) must be concentrated near \(d_{\max}\) with high probability to satisfy
  the definition of path-specific fairness, in
  Eq.~\eqref{eq:path_specific_fairness}. But, under the assumption in
  Eq.~\eqref{eq:contradiction}, we also show that the concentration occurs with
  low probability, by the continuity hypothesis in the statement of the theorem,
  establishing the contradiction.

  Specifically, by Markov's inequality, for any \(\rho > 0\), a.s.,
  \begin{align*}
    \Pr(d_{\max} - &\ d(X_{\Pi, A, a_1}) \geq \rho \mid X) \\
      &\leq \frac {\EE[d_{\max} - d(X_{\Pi, A, a_1}) \mid X]} {\rho} \\
      &= \frac {d_{\max} - d(X)} {\rho},
  \end{align*}
  where the final equality follows from Lemma~\ref{lem:d_pi}. Rearranging, it
  follows that for any \(\rho > 0\), a.s.,
  \begin{align}
  \label{eq:ineq1}
  \begin{split}
    \Pr(d_{\max} &- d(X_{\Pi, A, a_1}) < \rho \mid X) \\
      &\hspace{1.2cm}\geq 1 - \frac {d_{\max} - d(X)} {\rho}.
  \end{split}
  \end{align}

  Now let \(S = \{x \in \C X : d_{\max} - d(x) < \rho\}\). By the second
  hypothesis of the theorem, we can choose \(\delta\) sufficiently small that if
  \begin{equation*}
    \Pr(X \in S \mid A = a_1) < \delta
  \end{equation*}
  then, a.s.,
  \begin{equation*}
    \Pr(X_{\Pi, A, a_1} \in S, A \neq a_1 \mid X) < \tfrac 1 2.
  \end{equation*}
  In other words, we can chose \(\delta\) such that if
  \begin{equation*}
    \Pr(d_{\max} - d(X) < \rho \mid A = a_1) < \delta
  \end{equation*}
  then, a.s.,
  \begin{equation*}
    \Pr(d_{\max} - d(X_{\Pi, A, a_1}) < \rho, A \neq a_1 \mid X) < \tfrac 1 2
  \end{equation*}

  By Eq.~\eqref{eq:contradiction}, we can choose \(\epsilon > 0\) so small that
  \begin{equation*}
    \Pr(d_{\max} - d(X) < \epsilon \mid A = a_1) < \delta.
  \end{equation*}
  Then, we have that
  \begin{equation*}
    \Pr(d_{\max} - d(X_{\Pi, A, a_1}) < \epsilon, A \neq a_1 \mid X) < \tfrac 1
    2
  \end{equation*}
  a.s. Further, by the definition of the essential supremum and \(a_0\), and the
  fact that \(a_0 \neq a_1\), we have that
  \begin{equation*}
    \Pr(d_{\max} - d(X) < \tfrac \epsilon 2, A \neq a_1) > 0.
  \end{equation*}
  Therefore, with positive probability, we have that
  \begin{align*}
  \label{eq:ineq2}
    1 &- \frac {d_{\max} - d(X)} {\epsilon} \\
      &\hspace{1cm}> 1 - \frac { \frac \epsilon 2} \epsilon \\
      &\hspace{1cm}= \frac 1 2 \\
      &\hspace{1cm}> \Pr(d_{\max} - d(X_{\Pi, A, a_1}) < \epsilon, A \neq a_1
        \mid X).
  \end{align*}
  This contradicts Eq.~\eqref{eq:ineq1}, and so it cannot be the case that
  \(\Pr(d(X) = d_{\max} \mid A = a_0) = 0\), meaning \(\Pr(d(X) = d_{\max} \mid
  A = a_0) > 0\).

  Now, we show that \(\Pr(d(X) = d_{\max} \mid A = a_1) = 1\). Suppose, toward a
  contradiction, that
  \begin{equation*}
    \Pr(d(X) < d_{\max} \mid A = a_1) > 0.
  \end{equation*}
  Then, by the first hypothesis, a.s.,
  \begin{equation*}
    \Pr(d(X_{\Pi, A, a_1}) < d_{\max} \lor A = a_1 \mid X) > 0
  \end{equation*}
  As a consequence,
  \begin{align*}
    d_{\max}
      &= \EE[d(X) \mid d(X) = d_{\max}, A = a_0] \\
      &= \EE[ \EE[d(X_{\Pi, A, a_1}) \mid X] \mid d(X) = d_{\max}, A = a_0] \\
      &< \EE[ \EE[d_{\max} \mid X] \mid d(X) = d_{\max}, A = a_0] \\
      &= \EE[d_{\max} \mid d(X) = d_{\max}, A = a_0] \\
      &= d_{\max},
  \end{align*}
  where we can condition on the set \(\{d(X) = d_{\max}, A = a_0\}\) since
  \(\Pr(d(X) = d_{\max} \mid A = a_0) > 0\); and the second equality above
  follows from Lemma~\ref{lem:d_pi}. This establishes the contradiction, and so
  \(\Pr(d(X) = d_{\max} \mid A = a_1) = 1\).

  Finally, we extend this equality to all \(a \in \C A\). Since, \(\Pr(d(X) \neq
  d_{\max} \mid A = a_1) = 0\), we have, by the second hypothesis of the
  theorem, that, a.s.,
  \begin{equation*}
    \Pr(d(X_{\Pi, A, a_1}) \neq d_{\max}, A \neq a_1 \mid X) = 0.
  \end{equation*}
  Since, by definition, \(\Pr(X_{\Pi, A, a_1} = X \mid A = a_1) = 1\), and
  \(\Pr(d(X) = d_{\max} \mid A = a_1) = 1\), we can strengthen this to
  \begin{equation*}
    \Pr(d(X_{\Pi, A, a_1}) \neq d_{\max} \mid X) = 0.
  \end{equation*}

  Consequently, a.s.,
  \begin{align*}
    d(X) &= \EE[d(X_{\Pi, A, a}) \mid X]\\
    & = \EE[d_{\max} \mid X]\\
    & = d_{\max},
  \end{align*}
  where the first equality follows from Lemma~\ref{lem:d_pi}, establishing the
  result.
\end{proof}

\subsection{A Markov Chain Perspective}
\label{sec:markov}

The theory of Markov chains illuminates---and allows us to extend---the proof of
Theorem~\ref{thm:path_specific}. Suppose \(\C X = \{x_1, \ldots,
x_n\}\).\footnote{%
  Because of the technical difficulties associated with characterizing the
  long-run behavior of arbitrary infinite Markov chains, we restrict our
  attention in this section to Markov chains with finite state spaces.
} For any \(a' \in \C A\), let \(P_{a'} = [p^{a'}_{i,j}]\) where \(p^{a'}_{i,j}
= \Pr(X_{\Pi, A, a'} = x_j \mid X = x_i)\). Then \(P_{a'}\) is a stochastic
matrix.

To motivate the subsequent discussion, we first note that this perspective
conceptually simplifies some of our earlier results. Lemma~\ref{lem:d_pi} can be
recast as stating that when \(W = X\), a policy \(d\) is \(\Pi\)-fair if and
only if \(P_{a'} d = d\)---i.e., if and only if \(d\) is a 1-eigenvector of
\(P_{a'}\)---for all \(a' \in \C A\).

The 1-eigenvectors of Markov chains have a particularly simple structure, which
we derive here for completeness.

\begin{lem}
\label{lem:markov}
  Let \(S_1, \ldots, S_m\) denote the recurrent classes of a finite Markov chain
  with transition matrix \(P\). If \(d\) is a 1-eigenvector of \(P\), then \(d\)
  takes a constant value \(p_k\) on each \(S_k\), \(k = 1, \ldots, m\), and
  \begin{equation}
  \label{eq:pi_fair_character}
    d_i = \sum_{k=1}^m \left[ \lim_{n \to \infty} \sum_{j \in S_k} P_{ij}^n
    \right] \cdot p_k.
  \end{equation}
\end{lem}

\begin{rmk}
  We note that \(\lim_{n \to \infty} \sum_{j \in S_k} P^n_{i,j}\) always exists
  and is the probability that the Markov chain, beginning at state \(i\), is
  eventually absorbed by the recurrent class \(S_k\).
\end{rmk}

\begin{proof}
  Note that, possibly by reordering the states, we can arrange that the
  stochastic matrix \(P\) is in canonical form, i.e., that
  \begin{equation*}
    P =
      \begin{bmatrix}
        B   & \\
        R'  & Q
      \end{bmatrix},
  \end{equation*}
  where \(Q\) is a sub-stochastic matrix, \(R\) is non-negative, and
  \begin{equation*}
    B =
      \begin{bmatrix}
        P_1 &     &       & \\
          & P_2   &       & \\
          &     & \ddots  & \\
          &     &       & P_m
      \end{bmatrix}
  \end{equation*}
  is a block-diagonal matrix with the stochastic matrix \(P_i\) corresponding to
  the transition probabilities on the recurrent set \(S_i\) in the \(i\)-th
  position along the diagonal.

  Now, consider a \(1\)-eigenvector \(v = [v_1\ v_2]^\top\) of \(P\). We must
  have that \(P v = v\), i.e., \(B v_1 = v_1\) and \(R' v_1 + Q v_2 = v_2\).
  Therefore \(v_1\) is a 1-eigenvector of \(B\). Since \(B\) is block diagonal,
  and each diagonal element is a positive stochastic matrix, it follows by the
  Perron-Frobenius theorem that the 1-eigenvectors of \(B\) are given by
  \(\Span(\bb 1_{S_i})_{i=1, \ldots, m}\), where \(\bb 1_{S_i}\) is the vector
  which is 1 at index \(j\) if \(j \in S_i\) and is 0 otherwise.

  Now, for \(v_1 \in \Span(\bb 1_{S_i})_{i=1, \ldots, m}\), we must find \(v_2\)
  such that \(R' v_1 + Q v_2 = v_2\).

  Note that every finite Markov chain \(M\) can be canonically associated with
  an absorbing Markov chain \(M^{\abs}\) where the set of states of \(M^{\abs}\)
  is exactly the union of the transitive states of \(M\) and the recurrent sets
  of \(M\). (In essence, one tracks which state of \(M\) the Markov chain is in
  until it is absorbed by one of the recurrent sets, at which point the entire
  recurrent set is treated as a single absorbent state.) The transition matrix
  \(P^{\abs}\) associated with \(M^{\abs}\) is given by
  \begin{equation*}
    P^{\abs} =
      \begin{bmatrix}
        I & \\
        R & Q
      \end{bmatrix}
  \end{equation*}
  where \(R = R' [\bb 1_{S_1}\ \ldots\ \bb 1_{S_m}]\). In particular, it follows
  that \(v = [v_1\ v_2]^\top\) is a 1-eigenvector of \(P\) if and only if \([T
  v_1\ v_2]^\top\) is a 1-eigenvector of \(P^{\abs}\), where \(T : \bb 1_{S_i}
  \mapsto \bb e_i\).

  Now, if \(v\) is a 1-eigenvector of \(P^{\abs}\), then it is a 1-eigenvector
  of \((P^{\abs})^k\) for all \(k\). Since \(Q\) is sub-stochastic, the series
  \(\sum_{k=0}^\infty Q^k\) converges to \((I - Q)^{-1}\). Since
  \begin{equation*}
    (P^{\abs})^k =
      \begin{bmatrix}
        I                 & \\
        (I + Q + \cdots + Q^{k-1}) R  & Q^k
      \end{bmatrix},
  \end{equation*}
  it follows that
  \begin{equation*}
    \lim_{k \to \infty} (P^{\abs})^k =
      \begin{bmatrix}
        I         & \\
        (I - Q)^{-1} R   & 0
      \end{bmatrix}.
  \end{equation*}
  Therefore, if \(v = [v_1\ v_2]^\top\) is a 1-eigenvector of \(P^{\abs}\), we
  must have that \((I - Q)^{-1} R v_1 = v_2\). By Theorem~3.3.7 in
  \citet{kemeny1976finite}, the \((i,k)\) entry of \((I - Q)^{-1} R\) is exactly
  the probability that, conditional on \(X_0 = x_i\), the Markov chain is
  eventually absorbed by the recurrent set \(S_k\). This is, in turn, by the
  Chapman-Kolmogorov equations and the definition of \(S_k\), equal to \(\lim_{n
  \to \infty} \sum_{j \in S_k} P^n_{i, j}\), and therefore the result follows.
\end{proof}

We arrive at the following simple necessary condition on \(\Pi\)-fair policies.

\begin{cor}
\label{cor:markov}
  Suppose \(\C X\) is finite, and define the stochastic matrix \(P = \tfrac 1
  {|\C A|} \sum_{a' \in \C A} P_{a'}\). If \(d(x)\) is a \(\Pi\)-fair policy
  then it is constant on the recurrent classes of \(P\).
\end{cor}

\begin{proof}
  By Lemma~\ref{lem:d_pi}, \(d\) is \(\Pi\)-fair if and only if \(P_{a'} d = d\)
  for all \(a' \in \C A\). Therefore,
  \begin{equation}
    \frac 1 {|\C A|} \sum_{a' \in \C A} P_{a'} d = \frac 1 {|\C A|} \sum_{a' \in
    \C A} d = d,
  \end{equation}
  and so \(d\) is a 1-eigenvector of \(P\). Therefore it is constant on the
  recurrent classes of \(P\) by Lemma~\ref{lem:markov}.
\end{proof}

We note that Theorem~\ref{thm:path_specific} follows immediately from this.

\begin{proof}[Proof of Theorem~\ref{thm:path_specific}]
  Note that \(\tfrac 1 {|\C A|} \sum_{a \in \C A} P_a\) decomposes into a block
  diagonal stochastic matrix, where each block corresponds to a single stratum
  of \(\zeta\) and is irreducible. Consequently, each stratum forms a recurrent
  class, and the result follows.
\end{proof}

\section{Proof of Proposition~\ref{prop:pred_parity}}

To prove the proposition, we must characterize the conditional tail risks of the
beta distribution. Note that in the main text, we parameterize beta
distributions in terms of their mean \(\mu\) and sample size \(v\); here, for
mathematical simplicity, we parameterize them in terms of successes, \(\alpha\),
and failures, \(\beta\), where \(\mu = \tfrac \alpha {\alpha + \beta}\) and \(v
= \alpha + \beta\).

\begin{lem}
\label{lem:beta}
  Suppose \(Z_i \sim \bbeta(\alpha_i, \beta_i)\) for \(i = 0, 1\), and that
  \(\alpha_0 > \alpha_1 > 1\), \(1 < \beta_0 < \beta_1\). Then, for all \(t \in
  (0, 1]\), \(\EE[Z_1 \mid Z_1 < t] < \EE[Z_0 \mid Z_0 < t]\).
\end{lem}

\begin{proof}
  Let \(Z(\alpha, \beta) \sim \bbeta(\alpha, \beta)\). Then,
  \begin{align*}
    \EE[Z(\alpha, \beta) \mid Z(\alpha, \beta) < t]
      &= \frac {\int_0^t x \cdot \frac {x^{\alpha - 1} (1 - x)^{\beta - 1}}
        {B(\alpha, \beta)} \, \dx x} {\int_0^t \frac {x^{\alpha - 1} (1 -
        x)^{\beta - 1}} {B(\alpha, \beta)} \, \dx x} \\
      &= \frac {\int_0^t x^{\alpha} (1 - x)^{\beta - 1} \, \dx x} {\int_0^t
        x^{\alpha - 1} (1 - x)^{\beta - 1} \, \dx x}.
  \end{align*}
  Since \(\alpha > 1\), we may take the partial derivative with respect to
  \(\alpha\) by differentiating under the integral sign, which yields that \(
  \tfrac {\partial} {\partial \alpha}  \EE[Z(\alpha, \beta) \mid Z(\alpha,
  \beta) < t]\) equals
  \begin{equation*}
    \frac {\alpha \cdot I(t, \alpha, \beta)^2 - [\alpha - 1] \cdot I(t, \alpha +
    1, \beta) \cdot I(t, \alpha - 1, \beta)} {I(t, \alpha, \beta)^2},
  \end{equation*}
  where \(I(x, \alpha, \beta) = \int_0^t x^{\alpha - 1} (1 - x)^{\beta - 1} \,
  \dx x\). Rearranging gives that this is greater than zero when
  \begin{align*}
    0
      &< \alpha \cdot I(t, \alpha + 1, \beta) \cdot \int_0^t (x^{\alpha-2} -
        x^{\alpha-1}) (1 - x)^\beta \, \dx x  \\
      &\hspace{1cm} + \alpha \cdot I(t, \alpha, \beta) \cdot \int_0^t
        (x^{\alpha-1} - x^{\alpha}) (1 - x)^\beta \, \dx x \\
      &\hspace{1cm} + I(t, \alpha + 1, \beta) \cdot I(t, \alpha - 1, \beta).
  \end{align*}
  Since all of the integrands are positive, \(\tfrac {\partial} {\partial
  \alpha}  \EE[Z(\alpha, \beta) \mid Z(\alpha, \beta) < t] > 0\).

  A virtually identical argument shows that \(\tfrac {\partial} {\partial \beta}
  \EE[Z(\alpha, \beta) \mid Z(\alpha, \beta) < t] < 0\). Therefore, the result
  follows.
\end{proof}

We use this lemma to prove a modest generalization of
Prop.~\ref{prop:pred_parity}.

\begin{lem}
\label{lem:pred_parity_general}
  Suppose \(\C A = \{a_0, a_1\}\), and consider the family \(\C U\) of utility
  functions of the form
  \begin{equation*}
    u(x) =  r(x) + \lambda \cdot \B 1_{\alpha(x) = a_1},
  \end{equation*}
  indexed by \(\lambda \geq 0\), where \(r(x) = \EE[Y(1) \mid X = x]\). Suppose
  the conditional distributions of \(r(X)\) given \(A\) are beta distributed,
  i.e.,
  \begin{equation*}
    \C D( r(X) \mid A = a ) = \bbeta(\alpha_a, \beta_a),
  \end{equation*}
  with \(1 < \alpha_{a_1} < \alpha_{a_0}\) and \(1 < \beta_{a_0} <
  \beta_{a_1}\). Then any policy satisfying counterfactual predictive parity is
  strongly Pareto dominated.
\end{lem}

\begin{proof}
  Suppose there were a Pareto efficient policy satisfying counterfactual
  predictive parity. Let \(\lambda = 0\). Then, by Prop.~\ref{prop:threshold},
  we may without loss of generality assume that there exist thresholds
  \(t_{a_0}\), \(t_{a_1}\) and a constant \(p\) such that a threshold policy
  \(\tau(x)\) witnessing Pareto efficiency is given by
  \begin{equation*}
    \tau(x) = \begin{cases}
      1 & r(x) > t_{\alpha(x)}, \\
      0 & r(x) < t_{\alpha(x)}.
    \end{cases}
  \end{equation*}
  (Note that by our distributional assumption, \(\Pr(u(x) = t) = 0\) for all \(t
  \in [0, 1]\).) Since \(\lambda \geq 0\), we must have that \(t_{a_0} \geq
  t_{a_1}\). Since \(b < 1\), \(0 < t_{a_0}\). Therefore,
  \begin{align*}
    \EE[Y(1) \mid A = a_0, &\, D = 0] \\
      &= \EE[r(X) \mid A = a_0, u(X) < t_{a_0}] \\
      &\geq \EE[r(X) \mid A = a_0, u(X) < t_{a_1}] \\
      &> \EE[r(X) \mid A = a_1, u(X) < t_{a_1}] \\
      &= \EE[Y(1) \mid A = a_1, D = 0],
  \end{align*}
  where the first equality follows by the law of iterated expectation, the
  second from the fact that \(t_{a_1} \leq t_{a_0}\), the third from our
  distributional assumption and Lemma~\ref{lem:beta}, and the final again from
  the law of iterated expectation. However, since counterfactual predictive
  parity is satisfied, \(\EE[Y(1) \mid A = a_0, D = 0] = \EE[Y(1) \mid A = a_1,
  D = 0]\), which is a contradiction. Therefore, no such threshold policy
  exists.
\end{proof}

After accounting for the difference in parameterization,
Prop.~\ref{prop:pred_parity} follows as a corollary.

\begin{proof}[Proof of Prop.~\ref{prop:pred_parity}]
  Note that \(\alpha_a = \mu_a \cdot v_a > v \cdot \tfrac 1 v = 1\) and
  \(\beta_a = v - \alpha_a = v \cdot (1 - \mu_a) > v \cdot (1 - \tfrac 1 v) > 2
  - 1 = 1\). Moreover, since \(\mu_{a_0} > \mu_{a_1}\), \(\alpha_{a_0} = v \cdot
  \mu_{a_0} > v \cdot \mu_{a_1} = \alpha_{a_1}\) and \(\beta_{a_0} = v \cdot (1
  - \mu_{a_0}) < v \cdot (1 - \mu_{a_1}) = \beta_{a_1}\). Therefore \(1 <
  \alpha_{a_0} < \alpha_{a_1}\) and \(1 < \alpha_{a_1} < \alpha_{a_0}\), and so,
  by Lemma~\ref{lem:pred_parity_general}, the proposition follows.
\end{proof}

\end{document}